\documentclass{article}

\usepackage[main, final]{neurips_2026}

\usepackage{xcolor}
\definecolor{darkmutedteal}{RGB}{25, 95, 105}
\usepackage{hyperref}
\hypersetup{
    colorlinks=true,
    citecolor = darkmutedteal,
    linkcolor=darkmutedteal,
    filecolor=darkmutedteal,      
    urlcolor=black,
    pdftitle={A Mean-Field Framework for Inference-Time Distributional Control of Diffusion Models},
}

\usepackage{latexsym,amssymb}
\usepackage[ruled]{algorithm2e}
\usepackage{amsmath, amsthm, bbm, bm}
\usepackage{cleveref}
\usepackage{geometry}

\usepackage{algorithmic}

\usepackage{graphicx}
\usepackage{subcaption}
\usepackage{float}
\usepackage{mathtools}
\usepackage[dvipsnames]{xcolor}
\usepackage[nodayofweek]{datetime}
\usepackage{makecell}
\usepackage{booktabs}
\usepackage{wrapfig}
\usepackage{enumitem}
\usepackage{amsmath}

\newtheoremstyle{custom}  
  {10pt}   
  {10pt}   
  {\itshape}  
  {}        
  {\bfseries} 
  {.}       
  {5pt}     
  {}        
\theoremstyle{custom}

\newtheorem{theorem}{Theorem}[section]
\newtheorem{corollary}{Corollary}[theorem]
\newtheorem{proposition}[theorem]{Proposition}
\newtheorem{lemma}[theorem]{Lemma}
\newtheorem{assumption}[theorem]{Assumption}

\newtheorem{remark}{Remark}[section]

\usepackage{thmtools, thm-restate}

\usepackage{cleveref}

\usepackage{booktabs}
\usepackage{tabularx}
\usepackage{array}
\usepackage{gensymb}

\newcolumntype{Y}{>{\raggedright\arraybackslash}X}

\DeclareMathOperator*{\argmax}{arg\,max}
\DeclareMathOperator*{\argmin}{arg\,min}

\usepackage{soul}

\def\to{\rightarrow}

\def\eps{\varepsilon}

\def\cF{\mathcal{F}}

\def\cP{\mathcal{P}}

\def\cR{\mathcal{R}}

\def\cX{\mathcal{X}}

\def\sE{{\mathbb{E}}}

\def\sR{{\mathbb R}}

\def\rmd{{\mathrm{d}}}

\def\KL{{\mathrm{KL}}}

\def\CV{{\mathrm{CV}}}
\def\MMD{{\mathrm{MMD}}}
\def\SD{{\mathrm{SD}}}

\def\tmu{{\Tilde{\mu}}}
\def\tPhi{{\Tilde{\Phi}}}

\setcitestyle{authoryear,round,citesep={;},aysep={,},yysep={;}}
\renewcommand{\cite}[1]{\citep{#1}}
\usepackage[utf8]{inputenc} 
\usepackage[T1]{fontenc}    
\usepackage{hyperref}       
\usepackage{url}            
\usepackage{booktabs}       
\usepackage{amsfonts}       
\usepackage{nicefrac}       
\usepackage{microtype}      
\usepackage{xcolor}         

\title{A Mean-Field Framework for Inference-Time Distributional Control of Diffusion Models}

\author{%
  Samuel Howard \\
  Department of Statistics\\
  University of Oxford\\
  \And
  Nikolas N\"usken \\
  Department of Mathematics\\
  King's College London \\
}

\begin{document}

\maketitle
{\renewcommand{\thefootnote}{}\footnotetext{Corresponding author: \texttt{howard@stats.ox.ac.uk}}}

\begin{abstract}
    Diffusion models are increasingly used as controllable samplers, whose generations can be steered at inference time according to a chosen reward function. While such rewards are typically defined on individual samples, for many applications it is desirable to steer according to distribution-level rewards, for example to calibrate with population-level information or to encourage diversity. In both cases, simply incorporating the reward gradient into the dynamics, while often effective, comes with few theoretical guarantees on the sampled distribution. For pointwise rewards, recent work has therefore sought to develop a principled framework for targeting a prescribed tilted distribution using particle reweighting. However, an analogous theoretically-grounded approach for distributional rewards is currently lacking. In this work, we formulate inference-time distributional control as targeting a tilted measure under a mean-field framework, and derive a weighted interacting particle scheme to target it in a principled manner. Our framework recovers pointwise-reward steering as a special case, while providing a theoretical foundation for existing batch-level steering methods. Empirically, we verify that the procedure correctly targets the prescribed distribution in tractable low-dimensional settings, and investigate its behaviour in higher-dimensional protein conformation tasks.
\end{abstract}

\section{Introduction}

\begin{figure*}[t]
    \centering
    \includegraphics[
        width=1.0\linewidth,
        trim={30mm 60mm 60mm 50mm},
        clip
    ]{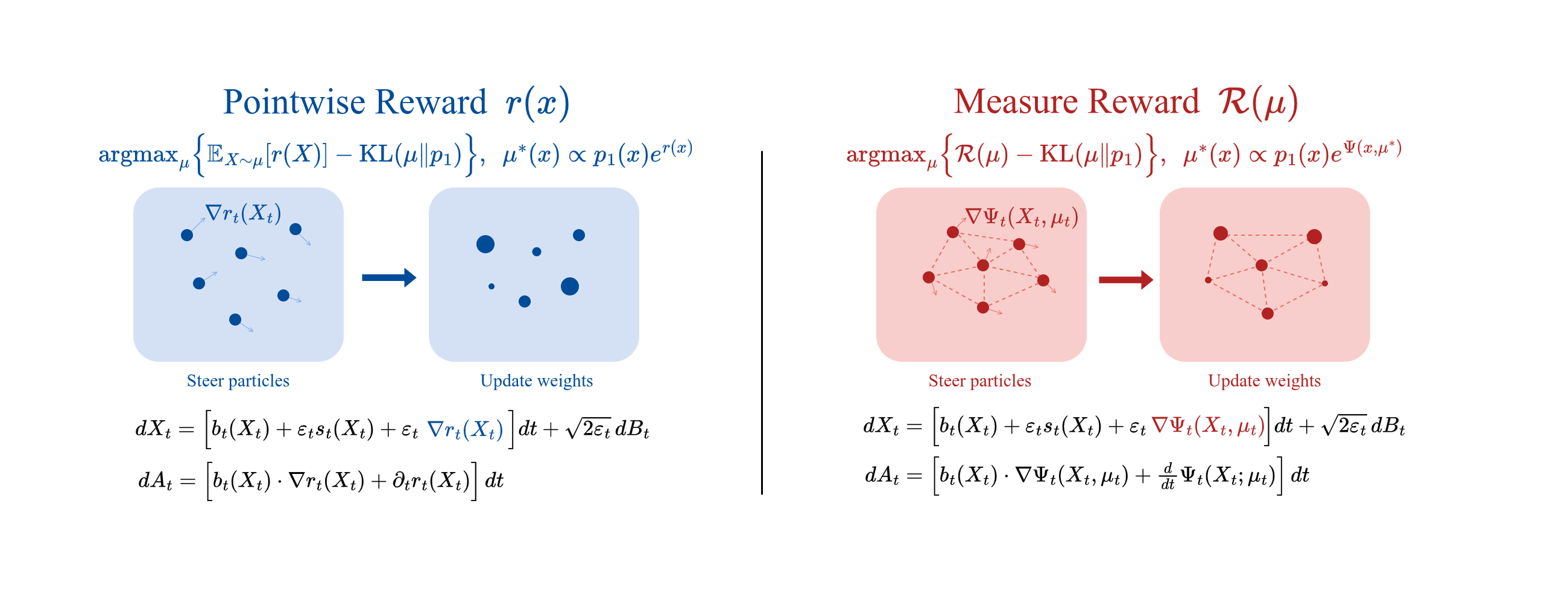}
    \caption{\textit{(Left)} \textbf{Standard Inference-Time Steering:} For pointwise rewards, the reward-gradients $\nabla r_t$ can be incorporated into the inference dynamics, to steer the particles independently towards areas of higher reward. Feynman-Kac corrections are applied to the log-weights $A_t$ to ensure the weighted particles trace the correct marginals $\mu_t^*(\cdot) \propto  p_t(\cdot) e^{r_t(\cdot)}$. \textit{(Right)} \textbf{Mean-Field Inference-Time Steering:} For a \textit{measure-defined} reward function $\cR(\mu)$, we incorporate the gradient of the first-variation, $\nabla \Psi_t(\cdot, \mu_t)$ as the steering term, leading to interaction between the particles. Through a mean-field framework, we derive the corrective reweighting procedure to ensure the reweighted interacting particles trace out the target tilted measures $\mu_t^*(\cdot) \propto p_t(\cdot) e^{\Psi_t(\cdot, \mu^*)} $. }
\end{figure*}

\paragraph{Diffusion and Flow Models}
Diffusion and flow-based generative models \citep{song2021scorebased, lipman2023flow, liu2022, albergo2023building} generate novel samples from a data distribution by learning time-dependent dynamics that transport an initial Gaussian distribution $p_0$ to the target $p_1$. Such models have achieved state-of-the-art performance across a wide range of domains, including images \citep{rombach22_SD}, molecules \citep{Watson2023_RFDiffusion}, and videos \citep{Ho22_video}.

\paragraph{Inference-Time Steering}
In many applications, the goal is not only to sample from the base distribution $p_1$, but to instead obtain samples that exhibit certain desired properties. To this end, one can `steer' the generative process by including additional drift terms in the inference dynamics.
Such steering mechanisms can be used to direct a pretrained generative model towards higher-quality samples, or samples that fulfil certain task-specific properties, without requiring further training.
Most existing inference-time steering methods aim to improve samples with respect to a \textit{pointwise} reward function $r: \sR^d \rightarrow \sR$ defined on individual generated samples.
Early steering approaches generally incorporated the reward gradient $\nabla r (X_t)$ into the inference dynamics, to `push' the generations towards areas of higher reward \citep{Dhariwal21_diffusion_beats_GANs, chung2023DPS, bansal2024universal}. While effective at improving sample quality, these heuristic modifications to the dynamics come with a poor theoretical understanding of the sampled distribution.

This limitation has motivated a large number of recent works aiming to place inference-time steering on a firmer theoretical footing. To do so, the dominant paradigm aims to sample from the $\KL$-regularised objective,
\begin{equation}
\label{eq:pointwise_KL_reg_objective}
    \argmax_\mu \big\{ \sE_{X \sim \mu} [r(X)] - \KL(\mu \| p_1) \big\}
\end{equation}
which corresponds to the \textit{exponentially-tilted} distribution
$\mu^*_1(x_1) \propto p_1(x_1)e^{ r(x_1)}$.
Since gradient-steering alone does not sample from this target tilted measure, recent works have sought to correct for the discrepancy using Feynman-Kac particle reweighting schemes, enabling correct tracking of the desired distributional path via SMC resampling procedures \citep{wu2023twisted, thornton2025controlled, singhal2025fks, skreta2025fkc, he2026_RNE, sabour2025FMTT}.

\paragraph{Distributional-Reward Steering}
Following these recent advances, steering according to pointwise-defined reward functions is now well-understood and based upon a principled theoretical foundation.
However, many important control objectives are inherently \textit{distributional} rather than pointwise. In many applications, rather than aiming for individual samples to score highly, one may wish to shape the overall generated distribution according to global properties, such as calibration to experimental observations, balancing across modes, or encouraging diversity. To formalise such distributional-steering, we will instead consider optimising according to a\textit{ measure-defined} reward function $\cR: \cP(\sR^d) \rightarrow \sR$,
\begin{equation}
\label{eq:distributional_KL_reg_objective}
    \argmax_\mu \big\{ \cR(\mu) - \KL(\mu \| p_1) \big\}.
\end{equation}
A number of existing works have incorporated steering terms that operate at the distributional level.
For example, \citet{corso2024particleguidance, kirchhof2025shielded, lam2026metadiffusion} steer towards generating diverse samples, \citet{maddipatla2025inverse} steer according to experimental observations, \citet{saravanan2026diffbed} steer to increase expected information gain, and  \citet{sani2026_mmdsteering} guide with MMD gradients. However, as in the pointwise-reward case, such approaches are largely heuristic in nature and it is not clear precisely what distribution these procedures target. Unlike in the pointwise-reward setting, a theoretically-grounded framework and approach for \textit{distributional} inference-time steering of diffusion models is currently lacking.

\paragraph{Contributions}
In this work, we address this gap by reconciling these two currently distinct literatures through a \textit{mean-field framework}, allowing us to place distributional steering on a similarly principled footing to the pointwise case. In particular, we make the following contributions.
\begin{itemize}[leftmargin=*, align=left]
    \item We formalise distribution-level steering \eqref{eq:distributional_KL_reg_objective} as aiming to sample from a tilted distribution for a measure-defined reward function $\cR$, in which the tilt now depends on the measure itself. This gives a distributional analogue of the reward-tilted targets used in the pointwise case.
    \item Using this mean-field framework, we derive weighted McKean-Vlasov dynamics whose marginals correctly track the desired tilted path. From this, we develop a novel weighted interacting-particle sampling procedure that enables principled distribution-level steering.
    \item We show that our framework recovers standard pointwise reward steering as a special case, and provides a principled counterpart to existing gradient-steering methods for distributional control. Empirically, we validate that our procedure correctly targets the tilted distribution in verifiable low-dimensional settings, and examine its behaviour on higher-dimensional protein conformation tasks relative to standard gradient-steering approaches.
\end{itemize}

\section{Background: Inference-Time Steering for Pointwise Rewards}
\label{sec:pointwise background}
In this work, we will largely follow the presentation of \citet{sabour2025FMTT}. To begin, let us consider a flow-based model given by the dynamics
\begin{equation}
\label{eq:base model}
    \rmd X_t = \big[ b_t(X_t) + \eps_t s_t(X_t) \big] \rmd t + \sqrt{2 \eps_t} \, \rmd B_t.
\end{equation}
Here, $\eps_t$ is an arbitrary noise schedule, and $s_t$ is the score function which can be obtained from a transformation of the vector field $b_t$ (see \citet{albergo2023_unifying}, for example). These dynamics trace out a path of marginals $p_t$ which interpolates from $p_0$ to the data distribution $p_1$.

Given a reward $r$ defined on the generated samples $x_1$, inference-time steering methods aim to modify the generation procedure to instead sample from the \textit{reward-tilted} distribution, which is the maximiser of the KL-regularised objective \eqref{eq:pointwise_KL_reg_objective},
\begin{equation}
\label{eq:pointwise reward tilt}
\mu^*(x) \propto p_1(x) e^{r(x)}.
\end{equation}
We remark that inference-time steering differs from the fine-tuning paradigm \citep{clark2024draft, xu2023imagereward, domingo-enrich2025adjoint}, as there is no additional training. Here, the dynamics are modified for the currently-chosen reward function $r$, and different choices of $r$ can be used without having to perform extra training.

\paragraph{Gradient-Steered Dynamics}
The aim of inference-time steering methods \citep{singhal2025fks, skreta2025fkc, he2026_RNE, sabour2025FMTT} is to `push' the particles towards areas of higher reward during the generation, while performing SMC-style reweighting of the particles to ensure they trace out a prescribed path of marginal distributions from $p_0$ to $\mu^*$.
To do so, one chooses a time-dependent reward function $r_t$ which satisfies $r_0=0$ and $r_1 = r$; we aim to trace out the \textit{tilted} path of marginals $\mu_t^*(x) \propto p_t(x)e^{r_t(x)}$, which will interpolate from $p_0$ to the target tilted distribution $\mu^* \propto p_1(x)e^{r(x)}$.
A common approach is then to modify the dynamics by incorporating a reward-gradient steering term,
\begin{equation}\label{eq:modified dynamics}
    \rmd X_t = \big[ b_t(X_t) + \eps_t s_t(X_t) + \eps_t \nabla r_t(X_t) \big] \rmd t + \sqrt{2 \eps_t} \rmd B_t.
\end{equation}
As the original reward function $r$ is defined on data, a natural choice is 
\begin{equation}
    \label{eq:pointwise_r_t}
    r_t(x_t) = t \, r( \hat{x}_1(t, x_t)),
\end{equation}
where $\hat{x}_1(t, x_t) $ is an estimate of the terminal point $x_1$ given the current point $(t, x_t)$ in the trajectory. In this work, we will consider using a reward construction that utilises the denoiser estimate, though we note that this framework is suitable for any such choice of $r_t$ (see \citet{sabour2025FMTT}, for example).

\paragraph{Corrective Feynman-Kac Reweighting}
While the additional steering term pushes the particles towards areas of higher reward, the dynamics \eqref{eq:modified dynamics} do not in general follow the desired tilted marginal path $\mu^*_t$. To correct for the discrepancy, it is necessary to adjust the measure using particle weights, which can then be used in an SMC-style resampling scheme to target $\mu^*$ in a theoretically-principled manner. The log-weights $A_t^i$ evolve according to the dynamics
\begin{equation}
    \label{eq:FMTT_weights}
    \rmd A_t = \Big[ b_t(X_t) \cdot \nabla r_t(X_t) + \partial_t r_t(X_t) \Big] \, \rmd t.
\end{equation}
Note that one can incorporate additional `pushing' terms rather than only weighting by $\eps_t$ in \eqref{eq:modified dynamics}, though this would introduce additional divergence terms in the weight dynamics (see \citet{sabour2025FMTT}, for example); for simplicity, we therefore consider the steering outlined above in our presentation.
\section{A Mean-Field Framework for Inference-Time Distributional Control}

\label{sec:mean_field_framework}

\subsection{Distributional-Reward Measure Tilting}
\label{sec:dist tilt}

Consider now that we want to tilt towards a desirable \textit{distributional} property of the generated measure, quantified by a measure-dependent reward function $\cR: \cP(\sR^d) \rightarrow \sR$ defined on the terminal distribution $\mu_1$. We can formalise this as aiming to sample from
\begin{equation}
    \mu^* = \argmax_\mu \big\{ \cR(\mu) - \KL( \mu \| p_1)  \big\},
\end{equation}
assuming for the moment that the maximiser is unique. 
To understand the properties of $\mu^*$, let us introduce the functional derivative $\Psi(x, \mu) := \frac{\delta \cR}{\delta \mu}(x, \mu)$,
which quantifies the change of the reward $\mathcal{R}$ when the measure $\mu$ is perturbed (for a rigorous definition and examples, see Appendix \ref{app:fun derivatives}). Using this object, we can characterise $\mu^*$ as an exponential tilt, mirroring \eqref{eq:pointwise reward tilt} for pointwise rewards.
\begin{proposition}[First-order optimality condition]
\label{prop:mean field tilt}
Assuming mild regularity properties and concavity of $\cR$, the maximiser in \eqref{eq:distributional_KL_reg_objective} is unique and satisfies
\begin{equation}
\label{eq:distributional_tilt_target}
    \mu^*(\rmd x) = \frac{1}{Z} e^{\Psi(x, \mu^*)} \, p_1(x) \, \rmd x,
\end{equation}
for normalising constant $Z = \int_{\mathbb{R}^d} e^{\Psi(x, \mu^*)} \, p_1(x) \, \rmd x$.
\end{proposition}
The proof, essentially setting the functional derivative of the objective $\cR(\mu) - \KL( \mu \| p_1)$ to zero, can be found in Appendix \ref{app:proof mean field tilt}. 
In contrast to \eqref{eq:pointwise reward tilt}, the tilt \eqref{eq:distributional_tilt_target} is \textit{implicit}, as the target measure $\mu^*$ appears on both sides of the equation and the tilting potential $\Psi(x,\mu^*)$ is itself measure-dependent. \Cref{prop:mean field tilt} therefore places the problem in a \textit{mean-field regime}---the optimality condition is a self-consistency relation, in which the target measure is characterised by a potential that depends on the measure itself.

As in the standard case from Section \ref{sec:pointwise background}, let us fix a family of measure-defined reward functions $\cR_t: \cP(\sR^d) \rightarrow \sR$ interpolating from $\cR_0 = 0$ to $\cR_1 = \cR$. Extending \eqref{eq:distributional_tilt_target}, we can define the curve of distributions  
\begin{equation}
\label{eq:targets}
\mu_t^*(\rmd x) = \frac{1}{Z_t} e^{\Psi_t(x, \mu_t^*)} \, p_t(x) \, \rmd x
, \qquad t \in [0,1],
\end{equation}
tracing from $p_0$ to $\mu_1^*$.
As in \eqref{eq:distributional_KL_reg_objective}, the measures $\mu_t^*$ are defined implicitly and self-consistently through the corresponding measure-dependent potentials $\Psi_t(x,\mu_t)$, similarly defined as functional derivatives of the time-dependent rewards $\mathcal{R}_t$. 
We note that these intermediate measures $\mu_t^*$ satisfy analogous variational principles, namely
\(
    \mu_t^* = \argmax_\mu \big\{ \cR_t(\mu) - \KL( \mu  \| p_t)  \big\}.
\)
Following \eqref{eq:pointwise_r_t}, in this work we will take $\cR_t = t \, \cR(\hat x_1(\cdot, t) \# \mu)$, using the denoiser estimate $\hat x_1$.

\subsection{Idealised Dynamics}

\paragraph{Distributionally-Steered Dynamics}
We will use this mean-field framework to develop a probabilistically-principled inference-time steering procedure that traces out this distributional path $\mu_t^*$ during the generation process.
As in the standard setting described in \Cref{sec:pointwise background}, it is again intuitive  to include a gradient-steering term to `push' the generated distribution in the direction of higher reward during the inference procedure. To do so, one can consider the modified dynamics,
\begin{align}
    \rmd X_t &= \big[ b_t(X_t) + \eps_t s_t(X_t) + \eps_t \nabla_x \Psi_t( X_t; \mu_t) \big] \rmd t + \sqrt{2 \eps_t} \, \rmd B_t.
\end{align}
In fact, such gradient-steering resembles many existing works for batch-level reward steering that incorporate gradient terms into the inference dynamics, for example for repulsive potentials \citep{corso2024particleguidance, lam2026metadiffusion}, ensemble likelihoods given experimental data \citep{maddipatla2025inverse}, expected information gain \citep{saravanan2026diffbed}, and maximum mean discrepancy to a reference dataset \citep{sani2026_mmdsteering}.

\paragraph{Corrective Mean-Field Reweighting}
Just as in the standard case, merely incorporating a gradient steering term into the dynamics does not in general follow the target marginals $\mu_t^*$, nor come with any strong theoretical understanding of the generated distribution. It is therefore again necessary to incorporate a Feynman-Kac style reweighting scheme to correct for this discrepancy. This setting differs from the standard pointwise case, as the particles now \textit{interact} during the inference dynamics due to the dependence of the steering term on the evolving distribution. Fortunately, by appealing to our mean-field framework we can derive the corrective potentials under this interacting particle behaviour.
This yields our main result, in which we construct mean-field type, or McKean-Vlasov, dynamics whose coefficients are chosen so that the corresponding law follows the desired path $\mu_t^*$.

\begin{restatable}[\textbf{Weighted McKean-Vlasov Dynamics}]
{theorem}{mfdynamics}
\label{th:mf_weighted_dynamics}
    Consider the dynamics
\begin{subequations}
\label{eq:McKean}
\begin{align}
\label{eq:X McKean}
\rmd X_t &= \big[ b_t(X_t) + \eps_t s_t(X_t) + \eps_t \nabla_x \Psi_t( X_t; \mu_t) \big] \rmd t + \sqrt{2 \eps_t} \, \rmd B_t, ~X_0 \sim p_0 \\
\rmd A_t &= \big[ b_t(X_t) \cdot \nabla_x \Psi_t( X_t; \mu_t) + \dot{\Psi}_t(X_t) \big] \rmd t, ~ A_0 = 0 
\end{align}
\end{subequations}
with the weighted measure $\mu_t \propto \mathrm{Law}(X_t) e^{A_t}$,
and $\dot{\Psi}_t(x) :=  \tfrac{\rmd}{\rmd t} \Psi_t( x; \mu_t)$. Then, under mild conditions, $\mu_t = \mu_t^*$, for all $t \in [0,1]$.
\end{restatable}
Note that these dynamics strongly resemble the pointwise setting \eqref{eq:FMTT_weights}. However, a key difference is that now the time-derivative term $\partial_t r(\cdot)$ has been replaced with the term $\dot \Psi_t(\cdot) = \tfrac{\rmd}{\rmd t} \Psi_t(\cdot, \mu_t)$; here, the position $x$ is kept fixed but the time-dependence now arises through both $\Psi_t$ and the measure evolution $\mu_t$.
Again in contrast to the standard case, this induces \textit{implicit} behaviour in the log-weight updates—their dynamics involve the change in the measure $\mu_t$, but the measure evolution is itself dependent on the change in the weights.
We remark that the term $\dot \Psi_t$ can be equivalently characterised as solving the following implicit expression, which we will later use in our inference procedure.

\begin{restatable}{proposition}{timederivative}
\label{prop:time derivative}
    Define the second variation of the reward functional, $\Phi_t(x,y,\mu) = \frac{\delta^2 \cR_t}{\delta \mu^2}(x,y,\mu)$,
    and a mean-centred version $\tPhi_t(x,y,\mu) = \Phi_t(x,y,\mu) - \int \Phi_t(x,z,\mu) \, \mu(\rmd z)$. Then, under the target flow \eqref{eq:targets}, the time derivatives $\dot \Psi_t$ satisfy the following implicit equality:
\begin{align}
\label{eq:implicit_eq_for_dotPsi}
\dot \Psi_t(x) &= \partial_t \Psi_t(x, \mu^*_t) + \int_{\mathbb{R}^d} \Big(\tPhi_t(x,y,\mu^*_t) \dot \Psi_t(y) \notag \\
&\qquad \qquad + b_t(y) \cdot \nabla_y \tPhi_t(x,y,\mu^*_t) + \tPhi_t(x,y,\mu^*_t) \, b_t(y) \cdot \nabla_y \Psi_t(y,\mu^*_t) \Big) \; \rmd \mu^*_t(y).
    \end{align}
\end{restatable}


\subsection{Method}
\label{sec:method}

\begin{algorithm*}[t]
\captionsetup{labelfont=bf}
\caption{Mean-Field Inference-Time Steering for Diffusion Models (Single Batch)}
\label{alg:interacting_particle_steering}

\begin{minipage}{\linewidth}
\setlist[itemize]{nosep, topsep=2pt, partopsep=0pt, parsep=0pt, itemsep=1pt}
\setlength{\abovedisplayskip}{4pt}
\setlength{\belowdisplayskip}{4pt}

\noindent\textbf{Input:}
Number of particles $N$, distributional reward function $\cR$ and corresponding first variation $\Psi_t$, time increment $\delta t$ or alternative time schedule.

\noindent\textbf{Initialisation:} Sample $X_0^{1:N} \sim p_0^{\otimes N}$, initialise log-weights $A_0^i=0$.

\noindent\textbf{For} $t \in \{0, \delta t, 2\delta t, \dots, 1-\delta t\}$:
\begin{itemize}[label=\textbullet, leftmargin=2.5em, labelsep=0.7em]
    \item \textbf{Positional updates:}
    \begin{itemize}[label=$\circ$, leftmargin=2.5em, labelsep=0.7em]
        \item Perform discretisation step for positional updates,
        \[
            X_{t+\delta t}^i = X_t^i + \big[ b_t(X_t^i) + \eps_t s_t(X_t^i) + \eps_t \nabla_x \Psi_t(X_t^i; \hat \mu_t^N) \big]\delta t + \sqrt{2\eps_t}\,\delta B_t^i,
        \]
        where $\hat \mu_t^N = \sum_{j=1}^N w_t^j \delta_{X_t^j}$ for $w_t^j \propto e^{A_t^j}$.
    \end{itemize}

    \item \textbf{Log-weight updates:}
    \begin{itemize}[label=$\circ$, leftmargin=2.5em, labelsep=0.7em]
        \item Compute $\tfrac{\rmd}{\rmd t}\Psi_t(\cdot;\hat \mu_t^N)$ at points $X_t^i$ using \Cref{alg:Psi_solver_implicit} or \Cref{alg:Psi_solver_fixedpoint}.
        \item Compute mean-field FK potentials $g_t^i = b_t(X_t^i) \cdot \nabla_x \Psi_t(X_t^i;\hat \mu_t^N) + \tfrac{\rmd}{\rmd t}\Psi_t(X_t^i;\hat \mu_t^N)$.
        \item Perform discretisation step for log-weight updates,
        \(
            A_{t+\delta t}^i = A_t^i + g_t^i\delta t.
        \)
    \end{itemize}
    \item \textbf{Resample:} Optionally resample according to current weights $w_t^i \propto e^{A_t^i}$.
\end{itemize}
\textbf{Return:} Weighted particles $(X_1^i, w_1^i)$, where $w_1^i \propto e^{A_1^i}$.
\end{minipage}
\end{algorithm*}

To turn the idealised dynamics of \Cref{th:mf_weighted_dynamics} into a finite-particle algorithm, we approximate the McKean-Vlasov system \eqref{eq:McKean} with an ensemble $(X_t^i, w_t^i)_{i=1}^N$, where the weighted empirical measure $\hat\mu_t = \sum_{j=1}^N w_t^j \delta_{X_t^j}$ approximates $\mu_t$. To do so, we need the following terms. 

\paragraph{Gradient-Steering Term $\nabla_x \Psi_t( x; \mu_t)$}
The gradient-steering term $\nabla_x \Psi_t( x; \mu_t)$ in the positional updates can be easily computed with autodifferentiation, evaluated on the current empirical measure $\hat{\mu}_t = \sum_{j=1}^N w_t^j \delta_{X_t^j}$ with weights $w_t^j \propto e^{A_t^j}$ . When one can analytically compute the first variation (as is often the case; see \Cref{app:fun derivatives}), then one can use a finite-sample approximation $\Psi_t(X_t^i, \hat{\mu}_t)$ and then differentiate with respect to the first argument. Alternatively, one could approximate the steering term using $(w_t^i)^{-1} \nabla_{X_t^i} \cR_t(\hat \mu_t)$, by directly evaluating the reward on the empirical measure and differentiating with respect to the positions $X_t^i$ (see Appendix \ref{app:implementation_details}).

\paragraph{Time-Derivative Term $\dot \Psi_t(x)$}
The more challenging term in the dynamics is the time-derivative $\dot \Psi_t(x) = \tfrac{\rmd}{\rmd t} \big( \Psi_t(x, \mu_t) \big)$ due to its implicit behaviour.
We provide two different approaches for computing this term. The first involves solving a linear system based on the implicit expression in \eqref{eq:implicit_eq_for_dotPsi}, which involves only tractable terms, and the second solves for $\dot \Psi_t(x)$ through a lightweight Picard iteration scheme; we describe these in \Cref{alg:Psi_solver_implicit,alg:Psi_solver_fixedpoint} respectively. We remark that the cost of such procedures is low, and the computational cost is often dominated by the network evaluations rather than these lightweight solvers (see Appendix \ref{app:comp_cost}). We found that the choice of solver for $\dot \Psi_t(x)$ had little effect on the empirical performance of the method and we therefore primarily used the fixed-point iterative solver for simplicity, but we remark that the first approach may be useful in certain settings.

\paragraph{SMC Resampling} 
As in the pointwise-reward Feynman-Kac steering approaches \citep{singhal2025fks, skreta2025fkc}, we use an SMC-style resampling scheme to avoid weight degeneracy. Any such scheme could be used; in our experiments we use residual resampling at fixed intervals.

\paragraph{Algorithm}

We outline the inference procedure for a single batch of $N$ particles in \Cref{alg:interacting_particle_steering}, which will approximate the target tilted measure $\mu^*$ as the batch size $N$ goes to infinity.
In practice, the batch size is limited by computational constraints; we therefore also provide a multi-batch procedure in \Cref{alg:interacting_particle_steering_multibatch} in Appendix \ref{app:implementation_details}.

\paragraph{Convergence in $N$}

Under mild assumptions on the dynamics and the $\dot \Psi_t$ solver (see Assumption \ref{ass:particle_convergence}), we can show convergence of the finite-particle system to the idealised dynamics $\eqref{eq:McKean}$. In particular, for $\delta_N$ (defined in \eqref{eq:solver_error}) quantifying the approximation error when solving for $\dot \Psi_t$, we have the following convergence result. 

\begin{restatable}[\textbf{Convergence of the Weighted Particle System}]{theorem}{convergenceresult}
\label{thm:weighted_particle_convergence_main}
Under mild assumptions,  
we have that 
\begin{equation}
\label{eq:test_function_qualitative_convergence}
    \mathbb E\left[
        \sup_{t\in[0,1]}
        \left|
            \hat\mu_t^N(\phi)-\mu_t^*(\phi)
        \right|
    \right]
    \longrightarrow0,
\end{equation}
 for every $\phi\in C_b^2(\mathbb R^d)$, as $N \rightarrow \infty$.
\end{restatable}

\paragraph{Generalising the Pointwise Setting}
We remark that the pointwise-reward setting is recovered from our framework by choosing
$\cR(\mu) = \int r(x) \, \rmd \mu(x)$. In this case, the first variation is $\Psi(x,\mu)=r(x)$, so our mean-field steering procedure recovers the standard procedure from \Cref{sec:pointwise background}.

\section{Experiments}
\label{sec:experiments}

\begin{figure*}[t!]
    \centering

    \begin{minipage}{\linewidth}
        \centering
        \includegraphics[width=0.8\linewidth]{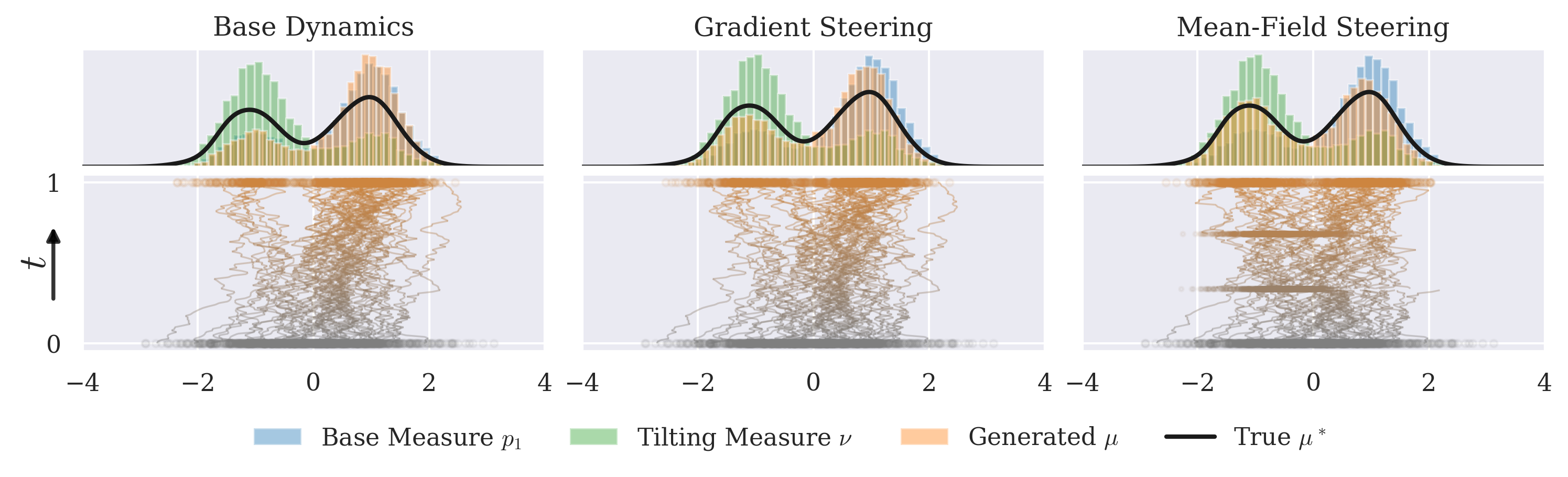}
        \caption{Trajectories and final measures $\mu_1$ for the base, gradient-steered, and mean-field dynamics, alongside an approximation of the true reward-tilted target $\mu^*$ for comparison.}
        \label{fig:trajectories_3_panel}
    \end{minipage}


    \begin{minipage}{\linewidth}
        \centering
        \includegraphics[width=0.8\linewidth]{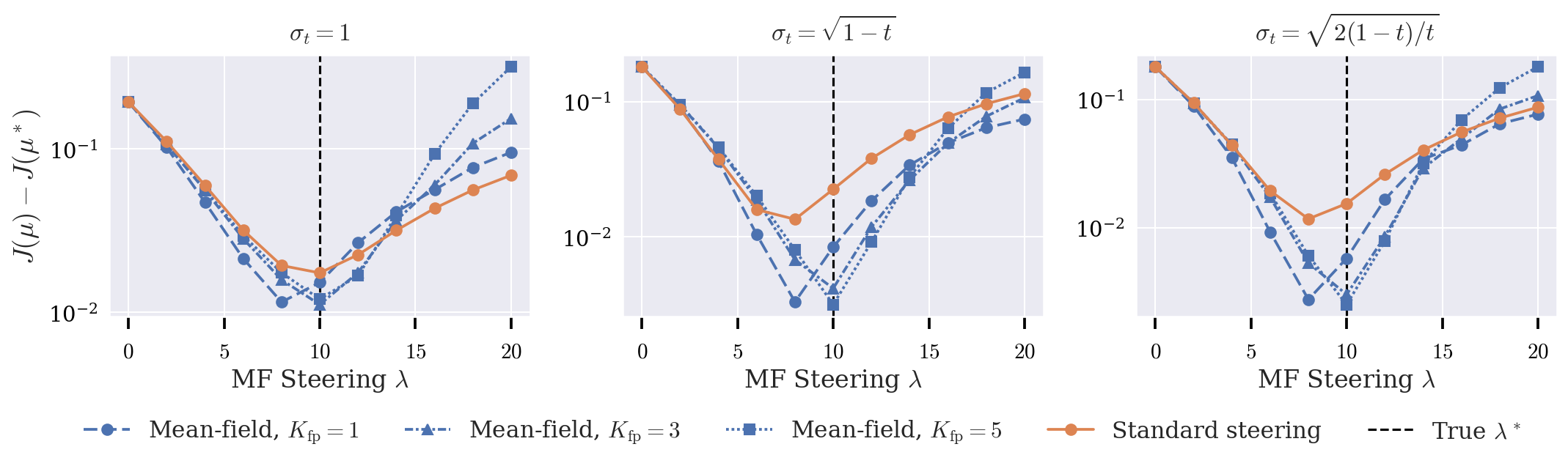}
        \caption{Objective gaps $J(\mu)-J(\mu^*)$ for mean-field sampling and gradient steering in the $\MMD$ example with $\lambda^*=10$. We compare fixed-point iterations $K_{FP}=1,3,5$ for our method against gradient steering with the steering strength varied linearly over $[0,4]$.}
        \label{fig:objective_curves_main}
    \end{minipage}

\end{figure*}

In this section, we empirically examine the behaviour of our proposed framework, aiming to verify that the resulting samples are drawn from the prescribed reward-tilted distribution.
We first consider low-dimensional settings where the reward-tilted target can be tractably computed, allowing us to reliably quantify the sampling accuracy. Following this, we consider larger-scale illustrative case studies on proteins, where we observe similar behaviour in higher-dimensional, realistic settings.

\subsection{Low-dimensional experiments}
\label{sec:low_dim}
We first consider a flow model that samples from a 1-dimensional bimodal Gaussian mixture $p_1$ with means $(-1.0, 1.0)$, both modes with unit variance, and mode weights $(1,3)$ respectively. To construct a distributional reward, we consider tilting the output distribution towards a similar Gaussian mixture $\nu$ but now with respective weights $(3,1)$ using a squared Maximum Mean Discrepancy (MMD) reward (similar to \citet{sani2026_mmdsteering}). We aim to sample from
\begin{equation}
    \label{eq:MMD objective}
    \mu^* = \argmin_\mu \big\{ \KL( \mu \| p_1) + \lambda^* \, \MMD^2( \mu , \nu  )  \big\}.
\end{equation}

\paragraph{Trajectory Visualisation}
In \Cref{fig:trajectories_3_panel}, we qualitatively compare trajectories for the base dynamics, the gradient-steered dynamics (representative of existing approaches such as \citet{corso2024particleguidance, maddipatla2025inverse, lam2026metadiffusion, saravanan2026diffbed, sani2026_mmdsteering}), and our mean-field steering approach. We use a constant noise schedule $\sigma_t=1.0$ and $\lambda^*=10$, and include an approximation to the true tilted solution $\mu^*$ for reference. 
As expected, the base dynamics sample the original measure $p_1.$
We see that standard gradient steering does push the particles towards the tilting measure $\nu$, but does not target $\mu^*$. In contrast, incorporating the mean-field reweighting procedure accurately samples from the tilted target $\mu^*$.

\paragraph{Quantitative Analysis of Optimality}
We quantitatively assess how accurately each sampler targets the tilted distribution $\mu^*$. For the objective $J(\mu)=\KL(\mu\|p_1)+\lambda^* \MMD^2(\mu,\nu)$, we report the optimality gap $J(\mu)-J(\mu^*)$. In \Cref{fig:objective_curves_main}, we compare gradient steering with our mean-field sampler using fixed-point iterations to compute $\dot{\Psi}_t$. Results are shown for three noise schedules: fixed $\sigma_t=1.0$, decaying $\sigma_t=\sqrt{1-t}$, and the memoryless schedule $\sigma_t=\sqrt{2(1-t)/t}$ used in \citet{domingo-enrich2025adjoint}.

In all cases, the objective gap is minimised when using mean-field steering with the correct $\lambda$ value, confirming that our procedure correctly targets the prescribed distribution $\mu^*$. The behaviour is consistent across noise schedules, indicating that the targeted distribution is invariant to these hyperparameter choices. Gradient-steering, by contrast, samples from a measure whose form depends on the steering strength and which does not in general coincide with $\mu^*$. The mean-field procedure thus provides a sampler whose target distribution is explicitly characterised by $\lambda$, rather than being indirectly induced by the choice of steering parameters.

We also note that running $K_{FP}=1$ in the fixed-point $\dot \Psi_t$ solver appears to give a slight bias as expected, but running $K_{FP}=3, 5$ are highly accurate and perform very similarly. We provide full experimental details and additional results, including for different tilting strengths $\lambda$, different measure-defined rewards $\cR$, and a more detailed analysis of the fixed-point solver, in \Cref{app:low_dim_details}.

\subsection{Tilting Protein Conformational Ensembles Towards Experimental Observations}

\begin{figure*}[!t]
    \centering

    \begin{minipage}{0.5\textwidth}
        \centering
        \includegraphics[width=0.7\linewidth]{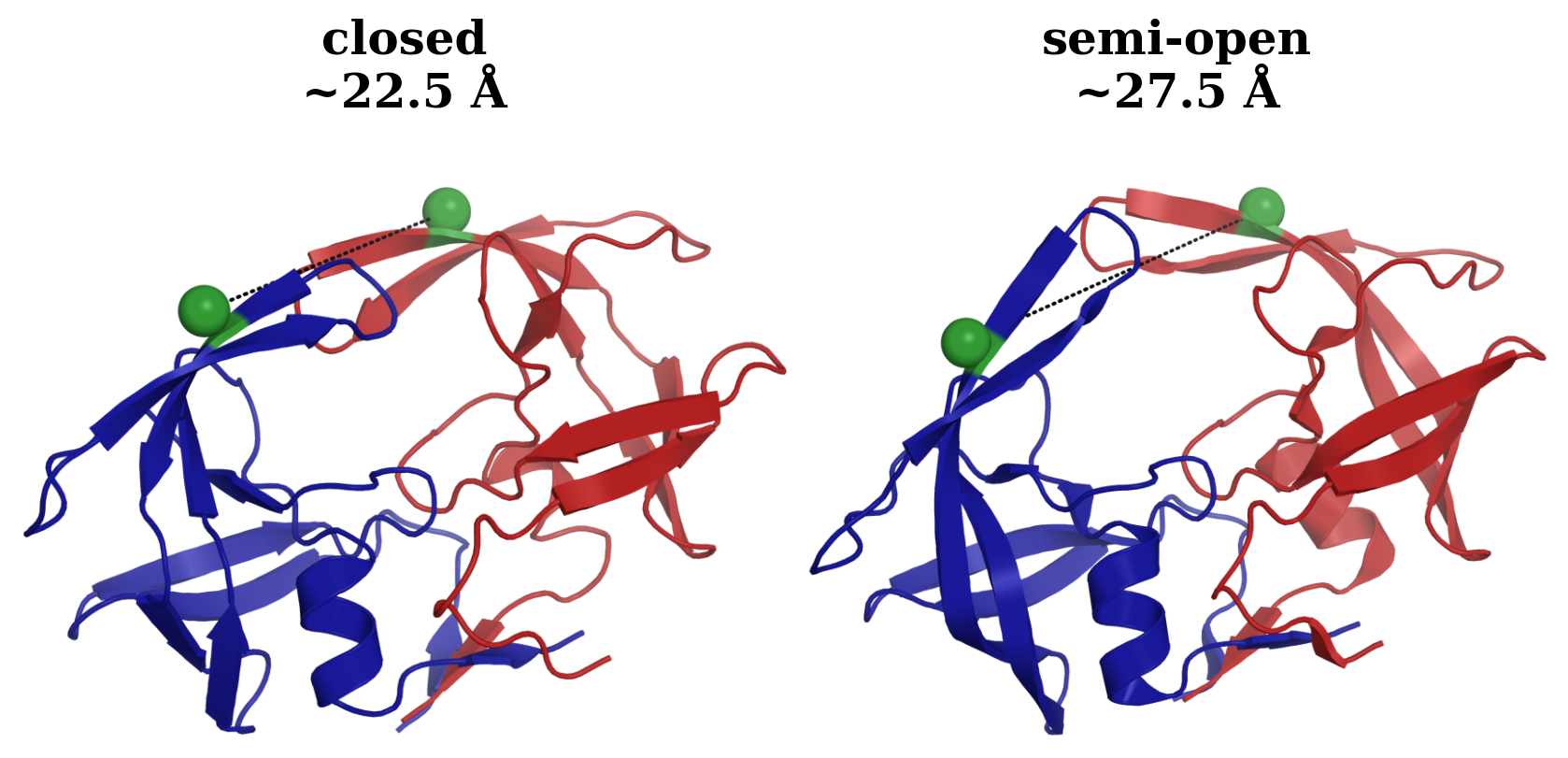}
        \caption{HIV-1 Protease exhibits two main conformations: closed \textit{(left)} and semi-open \textit{(right)}.}
        \label{fig:v6apo_visualisation}
    \end{minipage}
    \hfill
    \begin{minipage}{0.47\textwidth}
        \centering
        \includegraphics[width=0.5\linewidth]{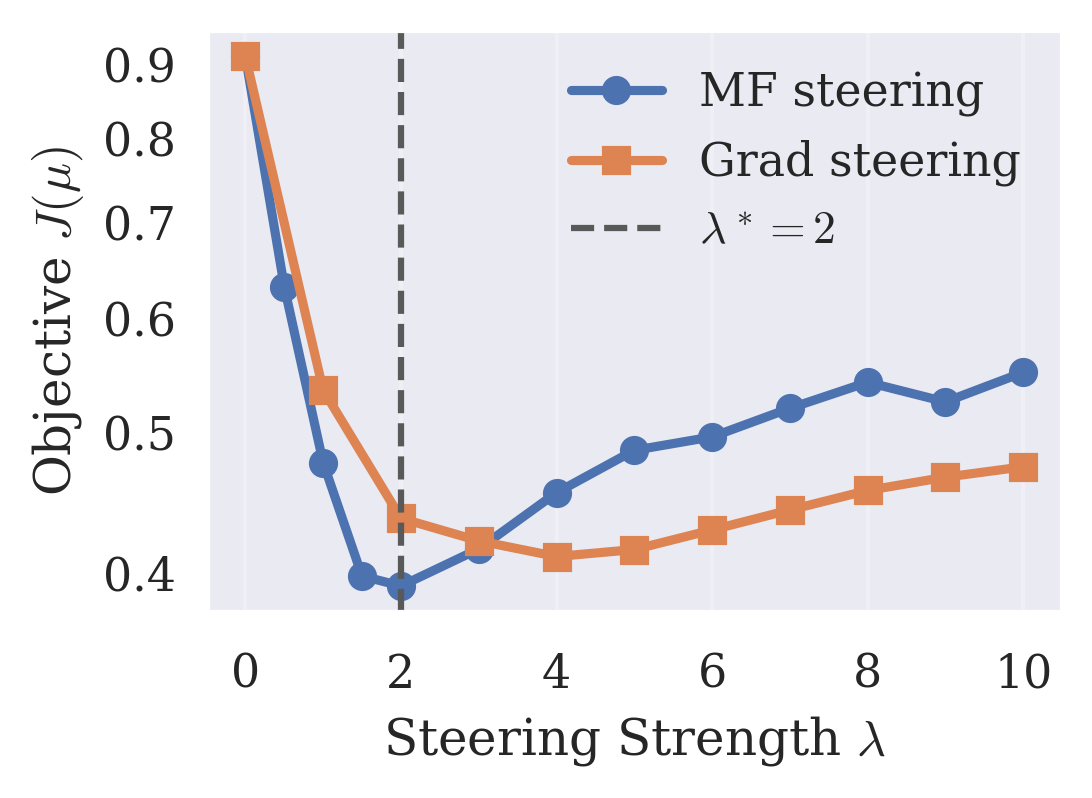}
        \caption{The objective proxy is minimised at the correct value for the steering coefficient $\lambda$.}
        \label{fig:v6apo_objective}
    \end{minipage}


    \begin{minipage}{\textwidth}
        \centering
        \includegraphics[width=0.85\linewidth]{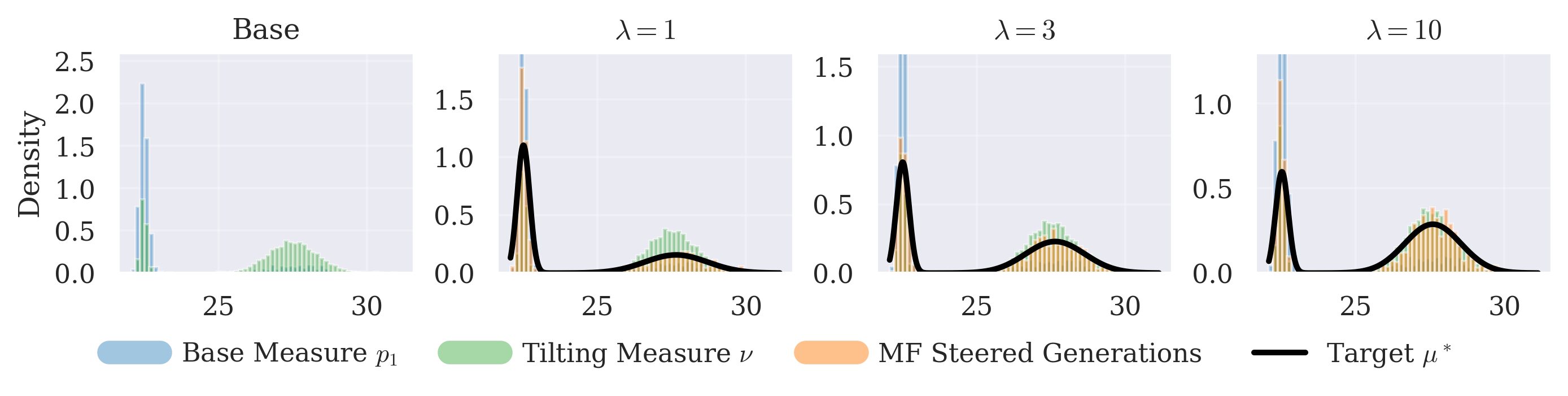}
        \caption{Distributions of the residue-55 distances in the generated conformations in the HIV-1 protease example. Increasing the strength $\lambda$ increases the tilt towards the experimentally-observed distribution $\nu$.}
        \label{fig:v6apo_histograms_mf}
    \end{minipage}
    \vspace{-1em}

\end{figure*}

In the previous section, we verified that our approach provides a principled correction mechanism for distributional steering in controlled low-dimensional settings. We now investigate to what extent this holds in higher-dimensional generative modelling tasks. Diffusion models are widely used for complex modalities such as images, videos, and proteins; among these, protein generation is a particularly natural application for the distributional steering setting, since the generated ensemble is often itself the object of interest and experimental observations frequently provide indirect, distribution-level information. This has motivated recent work on steering generated protein ensembles using distributional reward functions derived from experimental observations \citep{maddipatla2025inverse, Maddipatla2026_nature}.

\subsubsection{Distributions Over Residue Distances}
\label{sec:HIV_protease_main}

We first consider two settings in which we steer conformational ensembles generated by Boltz-2 \citep{passaro2025boltz2, wohlwend2024boltz1} towards experimentally observed distributions over lower-dimensional features, again comparing relative to gradient-steering. 
We remark that our aim is not necessarily to improve over gradient-steering; indeed, this is highly effective at steering towards desired distributional behaviour. Rather, we aim to understand to what extent the improved consistency and predictability afforded by our theoretically grounded correction mechanism extends to these higher-dimensional, challenging settings.

\paragraph{Protein Conformation Generation}
Proteins typically occupy ensembles of conformational states rather than a single fixed structure, and characterising these ensembles is crucial for understanding disease. Recent diffusion models can sample plausible conformations conditioned on an amino-acid sequence, but their induced conformational distributions may be miscalibrated, for example collapsing onto certain modes or assigning incorrect mode proportions.
Experimental methods can probe conformational ensembles, but they often provide indirect, distribution-level information rather than individual measurements. Such data can nevertheless reveal approximate conformational state proportions; this motivates combining experimental data with generative modelling, by `tilting' the base model toward empirically observed distributional information.

\paragraph{HIV-1 Protease}
We first consider HIV-1 protease; this protein exhibits a `flap region' that opens and closes between different conformational states (\Cref{fig:v6apo_visualisation}), which can be inferred from the residue-55 $C_\alpha-C_\alpha$ distances. 
Under the base Boltz-2 model, the states are sampled with proportions 79\% closed and 21\% semi-open. 
The double electron-electron resonance (DEER) method can be used to infer state proportions; such analysis in \citet{Liu16_HIV1conformation} suggests that these relative proportions are 25\% and 75\%. 
Based on these findings, we will aim to tilt generations towards a Gaussian mixture $\nu$ over residue-55 distances with these proportions, using the MMD objective \eqref{eq:MMD objective} but here instead evaluated on the 1-dimensional residue distances (that is, $\cR(\mu) = - \lambda \, \MMD^2( T \# \mu , \nu  )$, where $T$ is the projection onto the 1 dimensional space). We emphasise that the diffusion process itself takes place in the full 4608-dimensional space.

\paragraph{Results}
In \Cref{fig:v6apo_histograms_mf} we display histograms of the residue-55 distance for different tilt-strengths $\lambda$ using our method. 
For reference, we also plot a proxy for the target $\mu^*$ similarly to before; as the ambient-space $\KL(\mu \| p_1)$ term is intractable, we instead evaluate $\KL$ on the 1-dimensional residue distances (that is, we report an finite-sample approximation to $\KL( T\#\mu \| T\#p_1)$ where $T$ is the projection onto the 1 dimensional space; for details, see \Cref{app:v6apo_details}).
We see that, as expected, the generations are pulled towards the observed data as $\lambda$ increases, and the induced distribution of residue distances accurately matches the `ground-truth' proxy $\mu^*$ (we give similar visualisations for gradient-steering in \Cref{app:v6apo_details}).
As before, in \Cref{fig:v6apo_objective} we see that the objective is minimised at the correct $\lambda$ value when using mean-field steering, demonstrating that our principled approach remains effective even in this challenging, high-dimensional setting.

In \Cref{app:Adk}, we also consider another Boltz-2 example in which generations of adenylate kinase are steered to match a two-dimensional distribution over bond angles obtained from molecular dynamics simulations.

\subsubsection{Crystallographic Electron Densities}

\begin{figure*}[!t]
    \centering

    \begin{minipage}[t]{0.47\textwidth}
        \centering
        \includegraphics[width=0.8\linewidth]
        {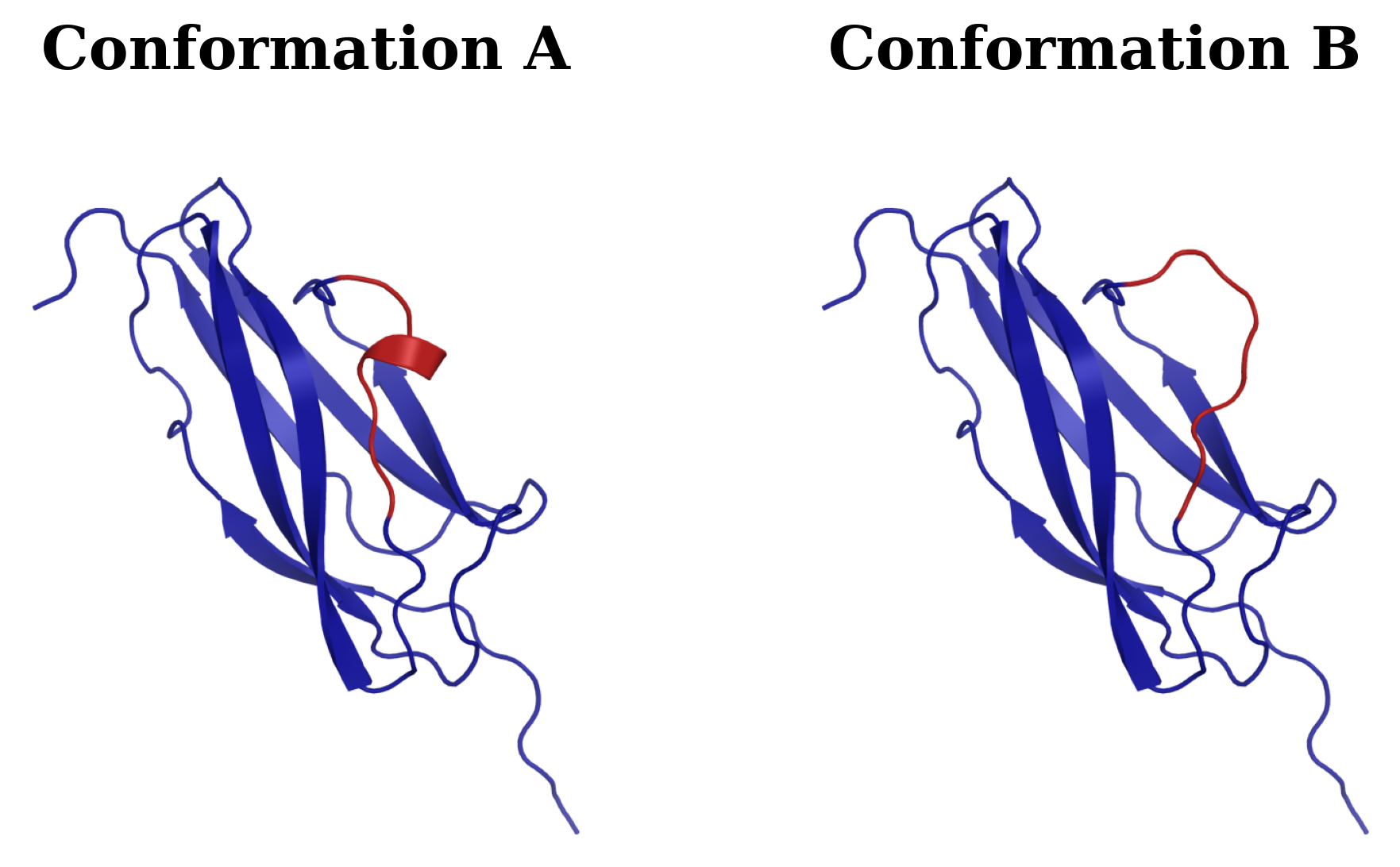}
        \captionof{figure}{Two conformations of 4OLE, with the region of interest (residues 423-431) highlighted in red.}
        \label{fig:4OLE_conformations}
    \end{minipage}
    \hfill
    \begin{minipage}[t]{0.47\textwidth}
        \centering

        \refstepcounter{figure}
        \setcounter{subfigure}{0}
        \vspace{-9em}
        \begin{subfigure}[t]{0.45\linewidth}
            \centering
            \includegraphics[width=0.95\linewidth]
            {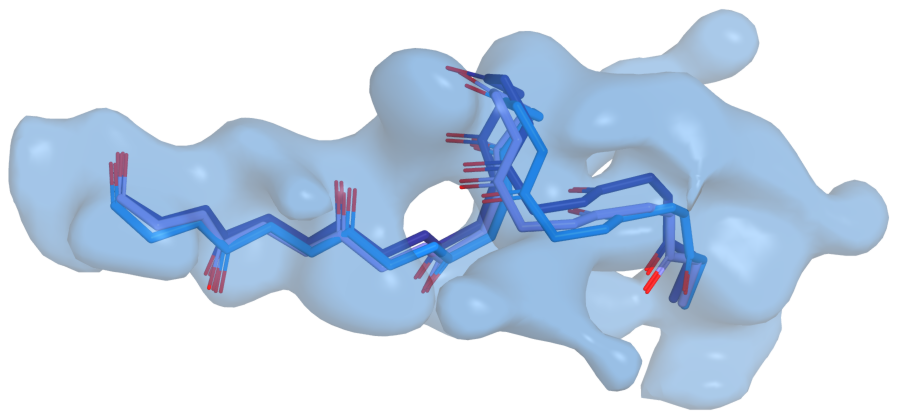}
            \caption{Base distribution $p_1$.}
            \label{fig:base_electron_density}
        \end{subfigure}
        \hfill
        \begin{subfigure}[t]{0.45\linewidth}
            \centering
            \includegraphics[width=0.95\linewidth]
            {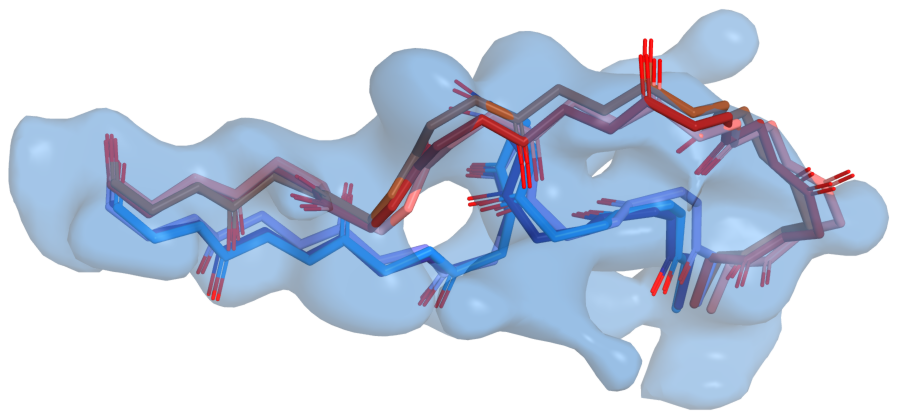}
            \caption{Steered distribution $\mu_1$.}
            \label{fig:steered_electron_density}
        \end{subfigure}

        \addtocounter{figure}{-1}
        \captionof{figure}{
            X-ray crystallography experiment and figure inspired by \citet{maddipatla2025inverse}, showing the sampled ROI and the target electron density. The base diffusion model samples nearly entirely conformation A. Distributional steering can cover both modes, giving a better fit to the observed density.
        }
        \label{fig:electron_densities}
    \end{minipage}

    \vspace{1.5em}
    \begin{minipage}{\textwidth}
        \centering

        \refstepcounter{figure}
        \setcounter{subfigure}{0}

        \begin{subfigure}[t]{0.3\textwidth}
            \centering
            \includegraphics[width=0.8\linewidth]
            {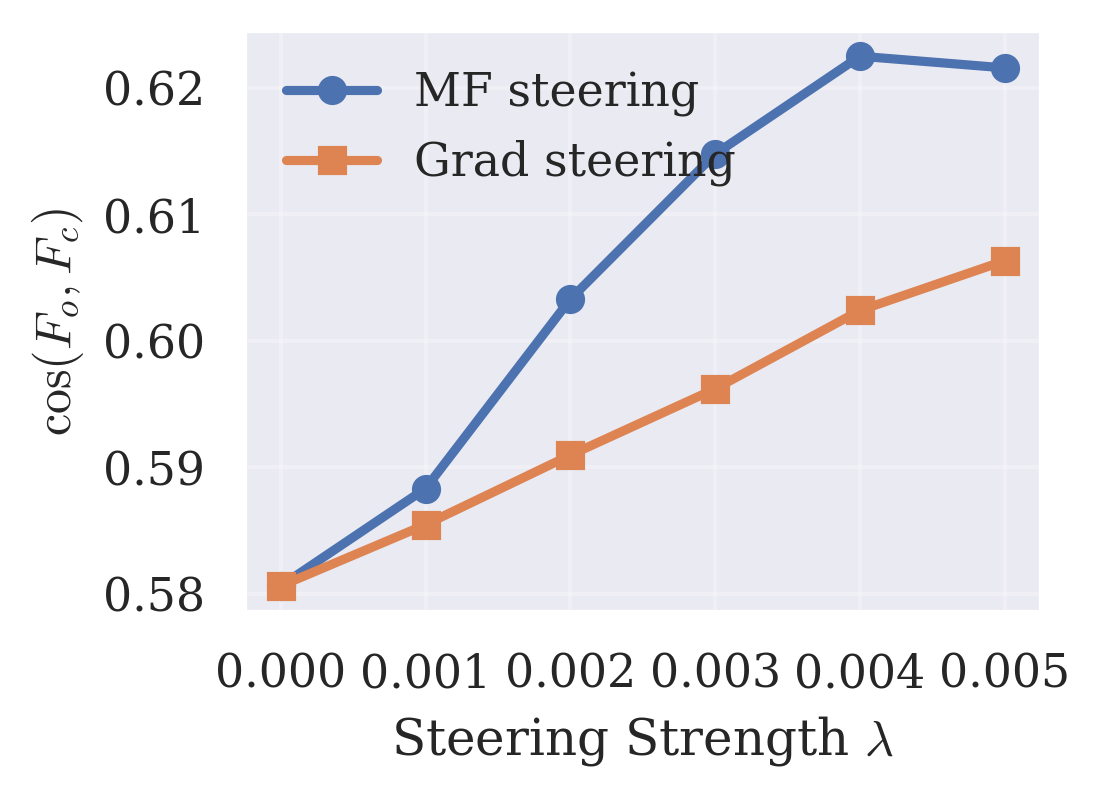}
            \caption{Cosine alignment between generated and target electron densities, as steering strength increases.}
            \label{fig:cosine_vs_strength_4OLE}
        \end{subfigure}
        \hfill
        \begin{subfigure}[t]{0.3\textwidth}
            \centering
            \includegraphics[width=0.8\linewidth]
            {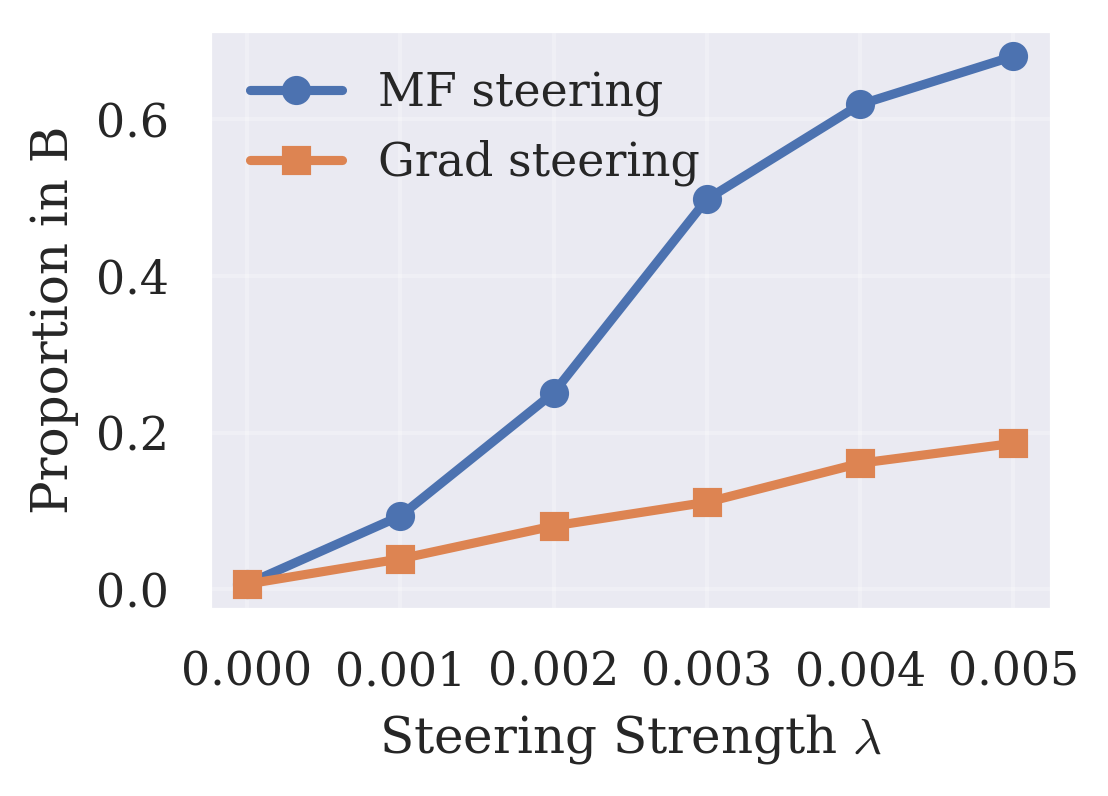}
            \caption{Proportion of generations closer to conformation $B$, as steering strength increases.}
            \label{fig:propB_vs_strength_4OLE}
        \end{subfigure}
        \hfill
        \begin{subfigure}[t]{0.3\textwidth}
            \centering
            \includegraphics[width=0.8\linewidth]
            {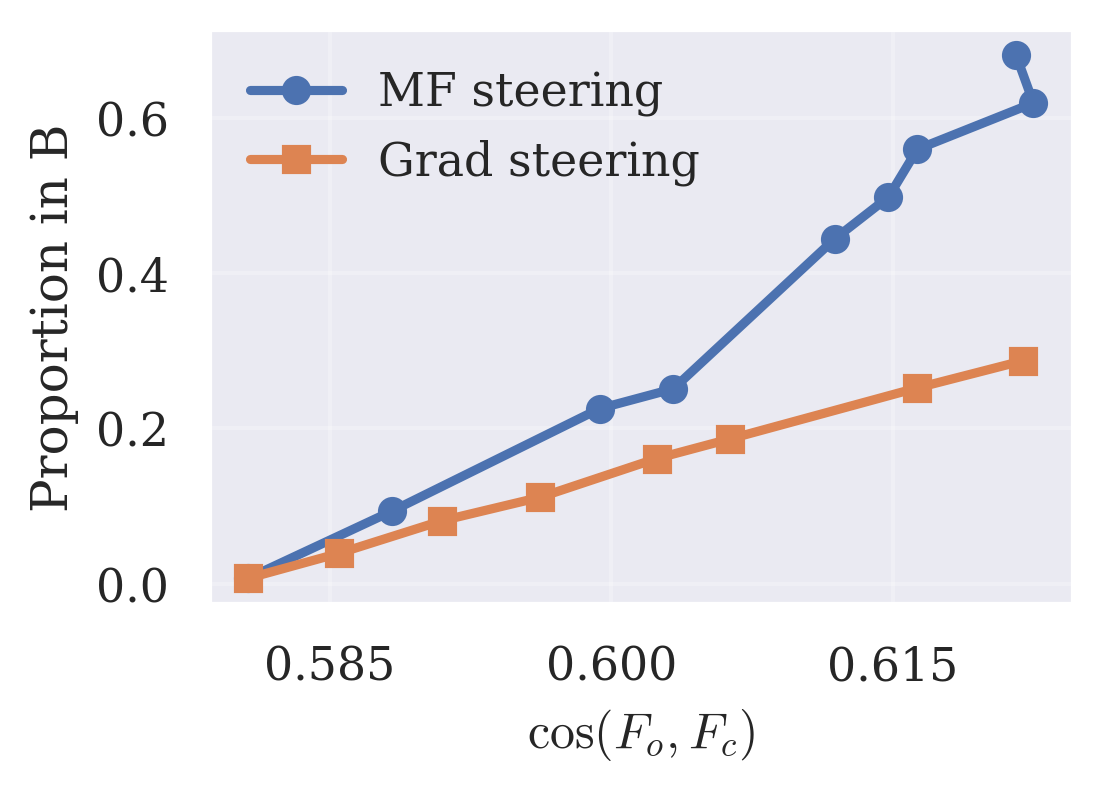}
            \caption{For the same cosine alignment, mean-field steering placed more samples closer to conformation B.}
            \label{fig:propB_vs_cosine_4OLE}
        \end{subfigure}

        \addtocounter{figure}{-1}
        \captionof{figure}{
            Comparing gradient-only and mean-field steering, when steering 4OLE generations according to the observed electron density.
        }
        \label{fig:4OLE_comparisons}
    \end{minipage}

\end{figure*}

To conclude, we consider a more challenging distributional-reward steering example, in which we steer the generations of a diffusion model to fit an electron density observed using X-ray crystallography. This technique provides a spatial and temporal average of the density over molecules in a crystal lattice, therefore constituting a population-level observation requiring a distribution-level reward. Here, we consider an example from \citet{maddipatla2025inverse}, who steer the Protenix model \citep{protenixv0} so that the generated ensemble's density is guided towards the empirical density. Here, we investigate how this steering mechanism is altered by including our additional mean-field correction.

We consider the PDB structure 4OLE, which is known to display heterogeneous conformations within the same crystal structure around a small region of interest (ROI) (\Cref{fig:4OLE_conformations}). The base model generates almost entirely from the helical conformation A, which does not explain the bimodal electron density in this region. By steering according to a log-likelihood reward function, \citet{maddipatla2025inverse} show that both modes can be recovered, giving an improved fit to the empirical electron density (\Cref{fig:electron_densities}).

Inspired by this experiment, in \Cref{fig:4OLE_comparisons} we compare the behaviour of standard gradient-only steering (similar to the steering used in \citet{maddipatla2025inverse}) with the additional inclusion of our mean-field reweighting scheme. For full details of the experimental setup, see Appendix \ref{app:electron_densities}. In the absence of a known ground-truth, we instead focus on discussing the qualitative differences between the results. We first observe that for the same steering strength, mean-field steering appears to result in a stronger reward tilt in this setting, giving a larger cosine alignment and more samples closer to conformation B. We also observe that mean-field steering generally placed more samples in the alternate conformation for the same cosine alignment.
\section{Discussion}
\label{sec:discussion}

\paragraph{Limitations}
The primary limitation of our approach is its reliance on resampling. While similar schemes are used in the pointwise-reward steering setting too, this can be a limitation in certain modalities (particularly images) as the final generations may include versions of similar images. For settings such as conformation generation, this is typically less of an issue.
We anticipate that this can be mitigated using methods to reduce weight variance \citep{ren2026driftlite}. We also require a differentiable reward function and an additional computational cost for the $\dot \Psi_t$ solver.
Finally, we remark that exact minimisation of a distributional objective is often not strictly necessary in practice, and gradient-steering alone may suffice in such cases; our framework can be seen as providing a principled foundation underlying such approaches.

\paragraph{Conclusion and Future Directions}
In this work, we have studied distributional reward steering under a mean-field framework, deriving a Feynman-Kac reweighting scheme that exactly targets the prescribed reward-tilted measure $\mu^*$. Our framework provides a principled counterpart to existing batch-level gradient-steering methods, and recovers existing FK-weighted steering for pointwise rewards as a special case. We empirically verified that our proposed framework targets the true tilted distribution in low-dimensional settings where the target is tractable, and illustrated the extent to which this behaviour transfers to higher-dimensional protein conformation tasks.
We hope that future work can build upon the framework developed here, drawing on domain expertise to formulate more expressive distributional steering objectives inspired by practical applications.

\begin{ack}
SH is supported by the EPSRC CDT in Modern Statistics and Statistical Machine Learning [EP/S023151/1].
The authors acknowledge the use of resources provided by the Isambard-AI National AI Research Resource (AIRR) \citep{mcintoshsmith2024isambardai}. Isambard-AI is operated by the University of Bristol and is funded by the UK Government's Department for Business, Innovation, Science and Trade (DBIST) via UK Research and Innovation; and the Science and Technology Facilities Council [ST/AIRR/I-A-I/1023].
The authors have no competing interests to declare.
\end{ack}

\bibliography{bibliography}
\bibliographystyle{icml2026}


\newpage
\appendix

\crefalias{section}{appendix}
\crefalias{subsection}{appendix}
\crefalias{subsubsection}{appendix}

\section*{Appendix}
\Crefname{section}{Appendix}{Appendices}

The Appendix is structured in the following way. In \Cref{app:background}, we discuss related work and required background information. In \Cref{app:implementation_details}, we provide additional details regarding implementation of the mean-field steering procedure. In \Cref{app:experimental_details}, we provide details of the experimental settings, along with additional experimental results. In \Cref{app:proofs}, we provide the proofs of the results stated in the main body. Finally, \Cref{app:licenses} includes the licenses for assets used.

\section{Background}
\label{app:background}

\subsection{Related Work}

\paragraph{Pointwise-Reward Gradient-Steering Methods}

Since the development of diffusion and flow-based generative models, there has been an extensive literature concerning incorporating additional steering terms into the inference dynamics to encourage desirable behaviour. Important examples include classifier-guidance \citep{Dhariwal21_diffusion_beats_GANs}, diffusion posterior sampling \citep{chung2023DPS}, and universal guidance \citep{bansal2024universal} to name a few. While effective, simply incorporating gradient terms comes with few theoretical guarantees on the generated distribution, so a more recent line of works has incorporated Feynman-Kac reweighting procedures to correct the dynamics to trace out a prescribed path of measures \citep{wu2023twisted, thornton2025controlled, singhal2025fks, skreta2025fkc, he2026_RNE, sabour2025FMTT}. Our presentation largely follows that in \citet{sabour2025FMTT}, though they focus on using a flow-map as their $\hat x_1$ prediction.
Recent works \citet{potaptchik2026MFMs, holderrieth2026diamond} also consider inference-time steering according to pointwise rewards, but sidestep the issues of requiring reweighting schemes by using a stochastic flow map to estimate the value function under a stochastic optimal control framework.

\paragraph{Batch-Reward Gradient-Steering Methods}
Recently there have been a number of works that have used steering terms that incorporate particle interactions in order to encourage desirable distributional properties in the generated samples. Similarly to the pointwise setting,  these approaches are effective but largely heuristic in nature, and come with a weak understanding of the sampled distribution. In fact, many of these methods use gradient-steered dynamics; our framework can thus be interpreted as providing a theoretical grounding for such methods through our derived reweighting scheme.

Many of these works focus on steering to encourage diversity in the generated samples. For a sequence of kernels $k_t$,  the \textit{Particle Guidance} method of \citet{corso2024particleguidance} modifies the inference dynamics to be
\begin{equation}\label{eq:particle_guidance}
    \rmd X_t^i = \big[ b_t(X_t^i) + \eps_t s_t(X_t^i) - \eps_t \nabla_{X_t^i} \sum_{j=1}^N k_t(X_t^i, X_t^j)  \big] \rmd t + \sqrt{2 \eps_t} \rmd B_t^i.
\end{equation}
which corresponds to gradient-steering for a batch-reward function $\cR_t^N(x^{1:N}) = - \tfrac{1}{2N}\sum_{i,j=1}^N k_t(x^i, x^j)$.
As commented in \citet{corso2024particleguidance}, this does not in general sample from $\mu_1^*$ as the dynamics do not correspond to a reversed diffusion process. To correct for this, \citet{corso2024particleguidance} also introduce a \textit{learned} potential steering approach, but this requires additional training and thus falls under the fine-tuning paradigm (see below) rather than the inference-time steering paradigm that we consider in this work.
Recent work by \citet{lam2026metadiffusion} from the computational biology literature  similarly incorporates the gradient of biasing potentials into the inference dynamics, including a diversity-encouraging potential $E_i = \tfrac{k}{N-1} \sum_{i\neq j} \exp\big(-\tfrac{(\CV_i - \CV_j)^2}{2 \sigma^2} \big)$ defined on chosen collective variables.
\citet{kirchhof2025shielded} include a \textit{sparse-repellency} term in the inference dynamics, that steers the generations out of shielded areas. While similar in spirit to particle guidance (in that the base dynamics are modified to include a repulsive interaction term), the control terms cannot be directly interpreted as gradients of a potential. Their steering terms are `sparse' in the sense that the base dynamics are only modified where necessary to avoid shielded areas. Finally, we also note the concurrent works \citet{azangulov2026variancetilted, vinograd2026_eddy} that have also targeted principled techniques for distributional steering, under the particular setting of encouraging diversity.  While our proposed framework does provide a theoretically-principled approach for such diversity steering, we remark that our specific reweighting procedure is somewhat ill-suited to the diversity setting due to its reliance on resampling, as whenever a particle is duplicated then this will generally be bad for diversity. We anticipate that combining with approaches to mitigate weight changes, such as \citet{ren2026driftlite}, can help to address this, and we believe this is a valuable direction for future work. 

Batch-level gradient steering has also been considered beyond the diversity setting. Such works include \citet{maddipatla2025inverse, Maddipatla2026_nature}, who consider steering the Protenix diffusion model \citep{protenixv0} according to experimental observations that provide distributional information. These works formulate the task as an inverse problem in which one aims to sample from the ensemble's posterior distribution $p( \mathcal{X} | a,y) \propto p(y | \mathcal{X}, a) \cdot p(\mathcal{X} | a)$; here, $\mathcal{X} = \{ X^1, ..., X^N \}$ is the sampled ensemble, $p(y | \mathcal{X}, a)$ denotes a chosen likelihood function, and the prior $p(\mathcal{X} | a)$ is given by the base diffusion model (conditioned on amino acid sequence $a$). This corresponds to choosing a log-likelihood reward function $\mathcal{R}(\mu) = \log p(y | \mu, a)$ in our framework. \citet{sani2026_mmdsteering} consider steering according to an MMD-based reward $\mathcal{R}(\hat \mu) = - \textrm{MMD}^2(\mu, \nu)$ for a reference dataset $\nu$, similar to several of our experimental settings. DiffBED \citep{saravanan2026diffbed} steers according to the expected information gain $\mathcal{R}(\hat \mu) = \mathrm{EIG}(\hat \mu)$, for applications in Bayesian experimental design.

The aim of this present work is to show how gradient-based steering for distributional rewards (like in the examples discussed above) can be placed on a theoretically grounded footing using a Feynman-Kac reweighting procedure, analogous to how FK-approaches have placed pointwise gradient steering on a principled foundation. We summarise this positioning of our work within the surrounding pointwise-reward and distributional-reward literature in \Cref{tab:position_in_literature}.

\begin{table}[h]
\centering
\small
\setlength{\tabcolsep}{4pt}
\renewcommand{\arraystretch}{1.25}
\renewcommand{\tabularxcolumn}[1]{m{#1}}
\begin{tabularx}{\columnwidth}{
    >{\centering\arraybackslash}m{0.25\columnwidth}
    >{\centering\arraybackslash}X
    >{\centering\arraybackslash}X}
\toprule
& \textbf{Heuristic Gradient Steering} & \textbf{Principled Targeting of $\boldsymbol{\mu^*}$} \\
\midrule
\shortstack{\textbf{Pointwise Reward}\\$r(X_1)$}
& $\nabla r(X_t)$ steering \par\smallskip
  {\scriptsize (e.g.\ \citet{Dhariwal21_diffusion_beats_GANs,chung2023DPS,bansal2024universal})}
& Reweighting approaches \par\smallskip
  {\scriptsize (e.g.\ \citet{wu2023twisted,singhal2025fks,skreta2025fkc,he2026_RNE,sabour2025FMTT})} \\
\shortstack{\textbf{Distributional Reward}\\$\mathcal{R}(\mu)$}
& $\nabla_{X_i}\mathcal{R}(\hat{\mu})$ steering \par\smallskip
  {\scriptsize (e.g.\ \citet{corso2024particleguidance,maddipatla2025inverse,sani2026_mmdsteering,lam2026metadiffusion,saravanan2026diffbed})}
& \textbf{Mean-Field Steering (Ours)}  \\
\bottomrule
\end{tabularx}
\vspace{1em}
\caption{Positioning of our work within the pointwise-reward and distributional-reward steering literatures.}
\label{tab:position_in_literature}
\end{table}

\paragraph{Fine-tuning Flow-based Generative Models}

In this work, we consider the inference-time steering literature for controlling diffusion models. A separate but related literature concerns \textit{fine-tuning} diffusion models. Similarly to inference-time steering, such methods aim to improve the quality of the generations according to a chosen reward function. However, unlike inference-time steering methods, these approaches involve additional training of the generative model, and must be re-trained for different reward functions.

For the pointwise-reward case, such methods are now widespread; prominent examples of such methods include DRaFT \citep{clark2024draft}, ReFL \citep{xu2023imagereward}, and DDPO \citep{black2024training}, to name a few. Adjoint Matching \citep{domingo-enrich2025adjoint} provides a scalable method for finetuning diffusion models from the perspective of stochastic optimal control (SOC), and has subsequently seen extensions to sampling, discrete settings, and wider SOC problems \citep{havens2025adjoint, liu2025adjoint, so2026discrete, howard2026control}. A recent line of works \citet{santi2026flowcontrol, wang2026efficient} have considered fine-tuning flow-based models according to measure-valued rewards, thus considering a similar setting to the one studied in this work but under the fine-tuning paradigm rather than inference-time steering.

\subsection{Functional derivatives}
\label{app:fun derivatives}
In this section, we briefly recall the notion of a functional derivative with respect to a probability measure. For a complete and rigorous reference, we refer to \citet{Santambrogio2015} for an accessible introduction.

Let \(\cF : \mathcal{P}(\mathbb{R}^d) \to \mathbb{R}\) be a functional on probability measures. We say that a measurable function
\[
\frac{\delta \cF}{\delta \mu}(x,\mu) : \mathbb{R}^d \times \cP(\mathbb{R}^d) \rightarrow \mathbb{R}
\]
is a \emph{functional derivative} (or \emph{first variation}) of \(\cF\) at \(\mu\) if, for measures \(\nu\) close to \(\mu\),
\begin{equation}
\cF(\nu)-\cF(\mu)
=
\int_{\mathbb{R}^d}
\frac{\delta \cF}{\delta \mu}(x,\mu)\,(\nu-\mu)(dx)
+
o(\nu-\mu).
\label{eq:first-variation-def}
\end{equation}
Informally, \(\frac{\delta \cF}{\delta \mu}(x,\mu)\) measures the first-order change in \(\cF\) when one adds a small amount of mass near the point \(x\).
Since \(\nu-\mu\) has total mass zero, the functional derivative is only defined up to an additive constant depending on \(\mu\). In most formulas this ambiguity is harmless.
A convenient equivalent characterisation is obtained by considering  perturbations of the form
\[
\mu_\varepsilon = (1-\varepsilon)\mu + \varepsilon \nu,
\]
for $\varepsilon > 0$ small enough.
Then \(\frac{\delta \cF}{\delta \mu}\) is characterised by
\begin{equation}
\left.\frac{d}{d\varepsilon} \cF(\mu_\varepsilon)\right|_{\varepsilon=0}
=
\int_{\mathbb{R}^d}
\frac{\delta \cF}{\delta \mu}(x,\mu)\,(\nu-\mu)(dx).
\label{eq:first-variation-curve}
\end{equation}

\paragraph{Examples}
In \Cref{tab:functional-derivative-examples} we record the functional derivatives of common measure-defined functionals $\cF$ that could be used under our framework (typically one would use $\cR = -\cF$ as a reward).

\begin{table}[t]
  \caption{Examples of first variations of functionals on probability measures.
  Functional derivatives on probability spaces are defined up to additive constants.}
  \label{tab:functional-derivative-examples}
  \centering
  \small
  \setlength{\tabcolsep}{3pt}
  \renewcommand{\arraystretch}{1.15}
  \begin{tabularx}{\textwidth}{@{}p{0.24\textwidth}YY@{}}
    \toprule
    \textbf{Name}
    &
    \textbf{Functional \(\cF(\mu)\)}
    &
    \textbf{Functional derivative \(\tfrac{\delta \cF}{\delta\mu}\)}
    \\
    \midrule

    Linear reward
    &
    \(\displaystyle \int f(x)\,\mu(dx)\)
    &
    \(\displaystyle f(x)\)
    \\[0.25em]

    Non-linear moment reward
    &
    \(\displaystyle F\!\left(\int \varphi(x)\,\mu(dx)\right)\)
    &
    \(\displaystyle
    F'\!\left(\int \varphi(y)\,\mu(dy)\right)\varphi(x)
    \)
    \\[0.25em]

    Quadratic moment reward
    &
    \(\displaystyle
    \left(\int \varphi(x)\,\mu(dx)\right)^2
    \)
    &
    \(\displaystyle
    2\left(\int \varphi(y)\,\mu(dy)\right)\varphi(x)
    \)
    \\[0.25em]

    Squared Maximum Mean Discrepancy ($\MMD^2$)
    &
    \(\displaystyle \iint k(x,y)\,(\mu-\nu)(dx)(\mu-\nu)(dy)\)
    &
    \(\displaystyle 2\int k(x,y)\,(\mu-\nu)(dy)\)
    \\[0.25em]

    Entropy
    &
    \(\displaystyle \int \rho_\mu(x)\log \rho_\mu(x)\,dx\)
    &
    \(\displaystyle \log \rho_\mu(x)+1\)
    \\[0.25em]

    Quadratic Wasserstein cost to fixed \(\nu\)
    &
    \(\displaystyle \frac12 W_2^2(\mu,\nu)\)
    &
    Kantorovich potential \(\phi(x)\) from \(\mu\) to \(\nu\) for cost
    \(\frac12\|x-y\|^2\)
    \\[0.25em]

    Entropic transport cost to fixed \(\nu\)
    &
    \(\displaystyle \mathrm{OT}_\varepsilon(\mu,\nu)\)
    &
    Schrödinger potential \(\varphi^\varepsilon(x)\)
    \\

    \bottomrule
  \end{tabularx}
\end{table}

\section{Implementation Details}
\label{app:implementation_details}

In this section, we provide additional implementation details for the mean-field steering procedure. Recall from Section \ref{sec:method} that we make a finite-particle approximation to the idealised weighted McKean-Vlasov dynamics \eqref{eq:McKean}. For $N$ particles, letting $\hat \mu_t = \sum_{i=1}^N w_t^i \delta_{X_t^i}$ for weights $w_t^i \propto e^{A_t^i}$, these dynamics are given by
\begin{subequations}
\label{eq:finite_particle_dynamics}
\begin{align}
\rmd X_t^i &= \big[ b_t(X_t^i) + \eps_t s_t(X_t^i) + \eps_t \nabla_x \Psi_t( X_t^i; \hat \mu_t) \big] \rmd t + \sqrt{2 \eps_t} \, \rmd B^i_t, \quad &X_0^i &\sim p_0 \\
\rmd A_t^i &= \big[ b_t(X_t^i) \cdot \nabla_x \Psi_t( X_t^i; \hat \mu_t) + \dot{\Psi}_t(X_t^i) \big] \rmd t, \quad &A_0^i &= 0.
\end{align}
\end{subequations}

\begin{algorithm*}[tb]
\captionsetup{labelfont=bf}
\caption{Mean-Field Inference-Time Steering for Diffusion Models (Multi-Batch)}
\label{alg:interacting_particle_steering_multibatch}

\begin{minipage}{\linewidth}
\setlist[itemize]{nosep, topsep=2pt, partopsep=0pt, parsep=0pt, itemsep=1pt}
\setlist[itemize,1]{label={}, leftmargin=2em, labelsep=0pt}
\setlist[itemize,2]{label=\textbullet, leftmargin=2.5em, labelsep=0.7em}
\setlist[itemize,3]{label=$\circ$, leftmargin=2.5em, labelsep=0.7em}
\setlength{\abovedisplayskip}{4pt}
\setlength{\belowdisplayskip}{4pt}

\noindent\textbf{Input:}
Number of batches $M$, number of particles per batch $N$, distributional reward function $\cR$ and corresponding first variation $\Psi_t$, time increment $\delta t$ or alternative time schedule.

\textbf{Initialisation:} Initialise an empty set of completed samples and log-weights $\mathcal H_0 = \emptyset$.

\noindent\textbf{For} $m \in \{1,\dots,M\}$:
\begin{itemize}
    \item \textbf{Initialise batch:} Sample $X_0^{m,1:N} \sim p_0^{\otimes N}$, initialise log-weights $A_0^{m,i}=0$.

    \item \textbf{For} $t \in \{0, \delta t, 2\delta t, \dots, 1-\delta t\}$:
    \begin{itemize}
        \item \textbf{Construct combined terminal measure:}
        \begin{itemize}
            \item Let $\mathcal H_{m-1} = \{(\bar X_1^\ell,\bar A_1^\ell)\}_{\ell=1}^{(m-1)N}$ denote the retained samples and log-weights from previous batches.
            \item Define $\hat \mu_1^{m,N} = \sum_{\ell=1}^{(m-1)N} \bar w_1^\ell \delta_{\bar X_1^\ell} + \sum_{j=1}^N w_t^{m,j}\delta_{\hat x_1(X_t^{m,j},t)}$, where $\bar w_1^\ell, w_t^{m,j} \propto e^{\bar A_1^\ell}, e^{A_t^{m,j}}$, with normalisation over both previously sampled and denoised current particles.
        \end{itemize}

        \item \textbf{Positional updates:}
        \begin{itemize}
            \item Perform discretisation step for positional updates,
            \[
                X_{t+\delta t}^{m,i} = X_t^{m,i} + \big[ b_t(X_t^{m,i}) + \eps_t s_t(X_t^{m,i}) + \eps_t \nabla_x \Psi_t(X_t^{m,i}; \hat \mu_t^{m,N}) \big]\delta t + \sqrt{2\eps_t}\,\delta B_t^{m,i},
            \]
            where, by slight abuse of notation, using $\hat \mu_t^{m,N}$ denotes using the combined terminal measure approximation induced by $\hat \mu_1^{m,N}$ when evaluating the reward terms.
        \end{itemize}

        \item \textbf{Log-weight updates:}
        \begin{itemize}
            \item Compute $\frac{\rmd}{\rmd t}\Psi_t(\cdot;\hat \mu_t^{m,N})$ at points $X_t^{m,i}$.
            \item Compute mean-field FK potentials $$g_t^{m,i} = b_t(X_t^{m,i}) \cdot \nabla_x \Psi_t(X_t^{m,i};\hat \mu_t^{m,N}) + \frac{\rmd}{\rmd t}\Psi_t(X_t^{m,i};\hat \mu_t^{m,N}).$$
            \item Perform discretisation step for log-weight updates,
            \(
                A_{t+\delta t}^{m,i} = A_t^{m,i} + g_t^{m,i}\delta t.
            \)
        \end{itemize}
        \item \textbf{Resample:} Optionally resample according to current weights $w_t^{m,i} \propto e^{A_t^{m,i}}$.
    \end{itemize}
    \item \textbf{Update completed set:}
    Update completed set with the completed batch samples and corresponding log-weights \[\mathcal H_m = \mathcal H_{m-1} \cup \{(X_1^{m,i}, A_1^{m,i})\}_{i=1}^N.\]
\end{itemize}
\textbf{Return:} Weighted particles $(X_1^i, w_1^i)$.
\end{minipage}
\end{algorithm*}

We now provide additional details regarding computation of the steering terms.

\paragraph{Gradient-Steering Term $\nabla_x \Psi_t( x; \mu_t)$}
As explained in the main text, the gradient-steering term $\nabla_x \Psi_t( x; \mu_t)$ appearing in the positional updates can be computed directly with autodifferentiation, evaluated on the current empirical measure $\hat{\mu}_t $. When one can analytically compute the first variation (as one can do, for example, for functionals in \Cref{tab:functional-derivative-examples}), then one can use a finite-sample approximation $\Psi_t(X_t^i, \hat{\mu}_t)$ and then differentiate with respect to the first argument. Alternatively, one can directly differentiate $\cR_t(\hat{\mu}_t)$ evaluated on the empirical measure with respect to the positions $X_t^i$ and rescale according to the particle weights, giving $(w_t^i)^{-1} \nabla_{X_t^i} \cR_t(\hat \mu_t)$  as the $i$-th steering term. Note that in the first case, the derivative with respect to the position is passed solely through the position $x$ with the measure $\mu$ kept fixed, while in the second the dependence on $x$ is passed through both—this is an algorithmic choice, as under standard smoothness assumptions these both converge to the desired mean-field steering term $\nabla_x \Psi_t( x; \mu_t)$ in the large-$N$ limit -- this follows from arguments similar to those in the proof of Theorem \ref{thm:weighted_particle_convergence}. In practice we found both choices to behave very similarly, and so use the first approach in our experiments.

\paragraph{Time-Derivative Term $\dot \Psi_t(x)$}
The more challenging term in the dynamics is the time-derivative $\dot \Psi_t(x) = \tfrac{\rmd}{\rmd t} \big( \Psi_t(x, \mu_t) \big)$, in which the position $x$ is kept fixed but the time-dependence arises through both $\Psi_t$ and the measure evolution $\mu_t$ (note that the dependence of $\dot \Psi_t(x)$ on $\mu_t$ is suppressed for notational convenience). In contrast to the standard case, this induces \textit{implicit} behaviour in the log-weight updates—their dynamics involve the change in the measure $\mu_t$, but the measure evolution is itself dependent on the change in the weights. We provide two different methods to overcome this. The first is mathematically more explicit, building a tractable linear system of size $N$ which can be solved cheaply. The second considers a simple fixed-point iterative procedure to approximate the weight updates. We outline both below.

For the implicit expression in \eqref{eq:implicit_eq_for_dotPsi}, all terms are generally tractable (in the proof of \Cref{prop:time derivative}, we use integration-by-parts to remove troublesome divergence terms on the score function). In many of the most standard cases (such as MMD, and diversity kernel rewards, for example) the second variation has an analytical expression. In other cases, one can take derivatives with respect to the particle weights for the finite sample approximations (which is cheap, as these are scalar). Note that the integral in \eqref{eq:implicit_eq_for_dotPsi} is linear in $\dot \Psi_t(y)$, so when evaluated using the current empirical measure this gives a linear system for the terms $\dot \Psi_t(X_t^i)$ with tractable coefficients, which can be solved cheaply. The resulting procedure is outlined in \Cref{alg:Psi_solver_implicit}.

As a simpler alternative, one can instead use a small fixed-point iterative procedure based on finite differences. For example, one can first approximate using $\dot \Psi_t(x) \approx \tfrac{\Psi_{t+\delta t}(x, \tmu_{t+\delta t}(w_t)) -  \Psi_t(x, \mu_t)}{\delta t}$ where $\tmu_{t+\delta t}(w_t) = \sum_j w_t^j \delta_{X_{t+\delta t}^j}$, then use this to approximate the weights $w_{t+\delta t}^j$. These new weights can be used as a new approximation to the next measure $\mu_{t+\delta t}$, and repeating the iterations a few times quickly gives accurate weights updates. In most settings this iterative procedure adds little computational overhead, as the expensive parts of the reward computations are usually independent of the weights (indeed, the expensive part is usually evaluating the denoiser estimates). For strong tilt strengths, we sometimes found it necessary to used damped Picard iterations, though for more reasonable tilt-strengths this was not necessary. We outline this fixed-point procedure in \Cref{alg:Psi_solver_fixedpoint}.

We found that the choice of solver for $\dot \Psi_t(x)$ had little effect on the empirical performance of the method, which we verify in \Cref{app:low_dim_details}.
We therefore used the fixed-point iterative solver in our higher-dimensional experiments for simplicity, but note that practitioners might find the first approach useful in certain settings.

\textbf{Details on the finite-difference computation}
As mentioned above, the simplest finite difference implementation would be to use $\dot \Psi_t(x) \approx \tfrac{\Psi_{t+\delta t}(x, \tmu_{t+\delta t}(w_t)) -  \Psi_t(x, \mu_t)}{\delta t}$. While simple, this does come with the downside that it involves evaluating the network on $X_t$ at the time $t+\delta t$. It is in fact possible to avoid any such time-offset by using a combination of autodifferentiation and finite-differences, as we now outline.

We are considering rewards of the form $\cR_t(\mu_t) = \alpha(t) \cR( \hat x_1(t,\cdot) \# \mu)$, which involves projecting to the current predicted endpoint. We can therefore write the first variation of the form $\Psi_t(x, \mu) = \alpha(t) \, \psi\big( \hat x_1(t,x) , \hat x_1(t,\cdot) \# \mu \big)$ for some $\psi$. Using the shorthand notation $\theta_i = \hat x_1(t,X_t^i)$ and $\hat \pi_t = \hat x_1(t,\cdot) \# \mu$, we can write the following approximation
\begin{equation}
    \label{eq:improved_dot_psi}
    \dot{\Psi}_t(X_t^i)
    \approx
    \frac{\alpha(t{+}\delta t)\,\psi\!\big(\theta_i,\hat{\pi}^{(k)}_{t+\delta t}\big)
                 -\alpha(t)\,\psi(\theta_i, \hat \pi_t)}{\delta t}
    + \alpha(t)\,\frac{d}{ds}\,\psi\!\big(\hat{x}_1(s,x), \hat \pi_t\big)\Big|_{s=t}
\end{equation}
In this expression, the first term is a finite-difference approximation that includes the dependence on the time-evolution of $\mu_t$, and the second term can be implemented easily using autodifferentiation. By implementing in this manner, one avoids evaluating the network at any points with a time-offset. We use this implementation in the protein experiments.

\paragraph{Multi-Batch Algorithm}
We outline the inference procedure for a single batch of $N$ particles in \Cref{alg:interacting_particle_steering}, which will approximate the target tilted measure $\mu^*$ as the batch size $N$ goes to infinity.
In practice, the batch size is limited by computational constraints; we therefore also provide a multi-batch procedure in \Cref{alg:interacting_particle_steering_multibatch}.
Specifically, one can simulate multiple batches of $N$ particles, and retain the generated samples along with the obtained weights. For each subsequent batch simulation, we evaluate the steering terms using the combined measure of the previously generated samples (with their respective weights) along with the particles currently being simulated. The procedure for updating the positions and log-weights of the current simulated particles remains unchanged.

\paragraph{Additional Steering Terms}

We remark here that one can also include additional steering terms in the positional updates, if one accounts for these in the log-weight dynamics. In particular, if one includes an additional $u_t(X_t; \mu_t)$ term as follows,
\begin{align}
    \rmd X_t &= \big[ b_t(X_t) + \eps_t s_t(X_t) + \eps_t \nabla_x \Psi_t( X_t; \mu_t) + u_t(X_t; \mu_t) \big] \rmd t + \sqrt{2 \eps_t} \, \rmd B_t,
\end{align}
then the necessary log-weight dynamics become
\begin{align}
    \rmd A_t &= \big[ b_t(X_t) \cdot \nabla_x \Psi_t( X_t; \mu_t) + \dot{\Psi}_t(X_t)
    + \nabla_x \cdot u_t(X_t; \mu_t) \nonumber \\
    &\qquad + u_t(X_t; \mu_t) \cdot (s_t(X_t) + \nabla_x \Psi_t( X_t; \mu_t))
    \big] \rmd t.
\end{align}
In particular, a natural choice might be to add in additional gradient steering by taking $u_t(X_t; \mu_t) = \chi_t \nabla_x \Psi_t( X_t; \mu_t)$ for a coefficient function $\chi_t$ (similar to the presentation in \citet{sabour2025FMTT}). In this case, the dynamics become
\begin{subequations}
\begin{align}
\rmd X_t &= \big[ b_t(X_t) + \eps_t s_t(X_t) + \eps_t \nabla_x \Psi_t( X_t; \mu_t) + \chi_t \nabla_x \Psi_t( X_t; \mu_t) \big] \rmd t + \sqrt{2 \eps_t} \, \rmd B_t, \\
\rmd A_t &= \big[ b_t(X_t) \cdot \nabla_x \Psi_t( X_t; \mu_t) + \dot{\Psi}_t(X_t) \nonumber \\
    &\qquad + \chi_t \big( 
    \lVert \nabla_x \Psi_t( X_t; \mu_t) \rVert^2 + \Delta_x \Psi_t( X_t; \mu_t) + s_t(X_t) \cdot \nabla_x \Psi_t( X_t; \mu_t)
    \big)
    \big] \rmd t.
\end{align}
\end{subequations}
These dynamics are again analogous to those in the corresponding pointwise case (see \citet{sabour2025FMTT}). Note too that there is no additional implicit dependence on the measure evolution; this is all captured by the terms already present in the standard dynamics \eqref{eq:McKean}.

In this work, we used the dynamics \eqref{eq:McKean} in our experiments. We note, however, that a possible limitation of this approach is that it can only redistribute mass among regions populated by existing particles. If the target tilt requires adding mass in an unpopulated area, this would not be possible by reweighting alone. To mitigate this issue, one could include additional steering into the positional updates using the above dynamics, to help ensure that there are particles covering the support of the target measures that can then be reweighted appropriately.

\paragraph{Optional Poisson Equation Solver}
To combat weight degeneracy in Feynman-Kac steering inference schemes, recent work \citet{ren2026driftlite} subsequently approximately solves the Poisson equation ${\nabla \cdot (\rho_t u_t) = - \rho_t g_t}$ for an additional drift term $u_t$. By solving this Poisson equation, the resulting update offloads the changes in log-weights into changes in position. \citet{ren2026driftlite} parameterise $u_t$ with a small set of basis functions, and approximately solve the Poisson equation by solving a resulting linear system. We remark here that this method can also be directly applied in our setting. The exact parameterisation can be chosen by the user; for example, given that we are already considering an interacting particle system, one could consider parameterising the additional drift to be of form
\(
    u_t(x) = \sum_j \alpha_j \nabla_x k(x, X_t^j),
\)
for a suitably chosen kernel $k$. We leave investigating combining our approach with that of \citet{ren2026driftlite} to future work; in particular, we anticipate that reducing the changes in the weights could enable applications for diversity steering, for which a resampling procedure is not well-suited.

\paragraph{Relation to Pointwise Rewards} Recall from Section \ref{sec:method} that our framework recovers the standard pointwise steering approach outlined in Section \ref{sec:pointwise background} when using the reward $\cR(\mu) = \int r(x) \rmd \mu(x)$. Given this, one might ask whether a pointwise reward could simply be used in place of a distributional one. We emphasise that the two settings solve fundamentally different problems. As an illustration, consider what happens as the reward strength $\lambda$ grows towards infinity. In the pointwise case, the tilted distribution becomes increasingly concentrated on the single point that maximises $r$. In the distributional case,  the tilt pushes $\mu$ towards the maximiser of the reward functional; for a divergence reward $D(\mu, \nu)$, this means $\mu$ approaches the entire target measure $\nu$ rather than any individual point. More fundamentally, matching a prescribed target measure $\nu$ via a pointwise reward would require choosing $r(x) = \log(\nu(x)/p_1(x))$, and hence require explicit access to the base density $p_1$ which is typically not available. Distributional rewards therefore enable forms of control that are not naturally expressible in the pointwise setting.

\begin{algorithm}[t]
\caption{Implicit Linear System $\dot \Psi_t$ Solver}
\label{alg:Psi_solver_implicit}
\begin{algorithmic}[1]
\STATE \textbf{Input:} Weighted particles $(X_t^i,w_t^i)_{i=1}^N$, functions $b_t$, $\partial_t\Psi_t$, $\nabla_y\Psi_t$, $\tPhi_t$, $\nabla_y\tPhi_t$.
\STATE \textbf{Output:} Values $\dot\psi_i \approx \dot\Psi_t(X_t^i)$.
\STATE Set $\hat\mu_t = \sum_{i=1}^N w_t^i \delta_{X_t^i}$.
\STATE For each $i,j \in \{1,\dots,N\}$, set $\phi_{ij} = \tPhi_t(X_t^i,X_t^j,\hat\mu_t)$, $g_{ij} = \nabla_y\tPhi_t(X_t^i,X_t^j,\hat\mu_t)$, $b_j = b_t(X_t^j)$, and $h_j = \nabla_y\Psi_t(X_t^j,\hat\mu_t)$.
\STATE For each $i,j \in \{1,\dots,N\}$, set $A_{ij} = \delta_{ij} - w_t^j \phi_{ij}$.
\STATE For each $i \in \{1,\dots,N\}$, set $r_i = \partial_t\Psi_t(X_t^i,\hat\mu_t) + \sum_{j=1}^N w_t^j \big( b_j \cdot g_{ij} + \phi_{ij} b_j \cdot h_j \big)$.
\STATE Solve $A\dot\psi = r$ for $\dot\psi = (\dot\psi_1,\dots,\dot\psi_N)^\top$.
\STATE \textbf{return} $(\dot\psi_i)_{i=1}^N$.
\end{algorithmic}
\end{algorithm}

\begin{algorithm}[t]
\caption{Fixed-Point $\dot \Psi_t$ Solver}
\label{alg:Psi_solver_fixedpoint}
\begin{algorithmic}[1]
\STATE \textbf{Input:} Particles and log-weights $(X_t^i,A_t^i)_{i=1}^N$, next-step particles $(X_{t+\delta t}^i)_{i=1}^N$, $b_t$, $\nabla_x\Psi_t$, step size $\delta t$, iterations $K_{FP}$, optional damping parameter $\alpha \in (0,1]$.
\STATE \textbf{Output:} Values $\dot\psi_i \approx \dot\Psi_t(X_t^i)$ and updated weights $w_{t+\delta t}^i$.
\STATE Set $\hat\mu_t = \sum_{i=1}^N w_t^i \delta_{X_t^i}$, and $a_i^{(0)} = A_t^i$, for $i=1,\dots,N$
\FOR{$k=0,\dots,K_{FP}-1$}
    \STATE Construct candidate weighted next-step measure $\hat\mu_{t+\delta t}^{(k)} = \sum_{j=1}^N w_j^{(k)} \delta_{X_{t+\delta t}^j}$.
    \STATE Compute finite differences $\dot\psi_i^{(k)} = \big(\Psi_{t+\delta t}(X_t^i,\hat\mu_{t+\delta t}^{(k)}) - \Psi_t(X_t^i,\hat\mu_t)\big)/\delta t$ (or using \eqref{eq:improved_dot_psi}).
    \STATE Update candidate log-weights updates $\tilde a_i^{(k+1)} = A_t^i +  \delta t \big[ \dot\psi_i^{(k)} + b_t(X_t^i) \cdot \nabla_x\Psi_t(X_t^i,\hat\mu_t) \big]$.
    \STATE Update next-step weights as  $a_i^{(k+1)} = (1-\alpha)a_i^{(k)} + \alpha \, \tilde a_i^{(k+1)}$.
    \STATE Construct updates next-step weights $w_i^{(k+1)} \propto e^{a_i^{(k+1)}}$.
\ENDFOR
\STATE Set $\dot\psi_i = \dot\psi_i^{(K_{FP}-1)}$ and $w_{t+\delta t}^i = w_i^{(K_{FP})}$, for $i=1,\dots,N$.
\STATE \textbf{Return:} Final iterations $(\dot\psi_i,w_{t+\delta t}^i)_{i=1}^N$.
\end{algorithmic}
\end{algorithm}
\section{Experimental Details and Additional Results}
\label{app:experimental_details}

\subsection{Low-dimensional experiments}
\label{app:low_dim_details}

\paragraph{Experimental setting}

For the base model, we train a simple flow-matching model to sample from a bimodal Gaussian mixture base distribution $p_1$, with means $(-1.0, 1.0)$, both modes with unit variance, and mode weights $(1,3)$ respectively. To construct a distributional reward, we consider tilting the model's output distribution towards a similar Gaussian mixture $\nu$ but now with respective weights $(3,1)$. For the results in the main paper, we use the squared MMD reward
\begin{equation}
    \mu^* = \argmin_\mu \big\{ \KL( \mu \| p_1) + \lambda^* \, \MMD^2( \mu , \nu  )  \big\}.
\end{equation}
We also demonstrate similar results for other reward functions below.
By optimising a Gaussian mixture model, we also obtain a highly accurate approximation to the true tilted solution $\mu^*$ from which we aim to sample; the resulting density is plotted in \Cref{fig:trajectories_3_panel} to give a visual indication of how well $\mu^*$ has been targeted.

For the reward function, we use a Gaussian radial basis kernel with bandwidth 5.0 (we note that results were robust across different kernel widths, provided they were sufficiently wide to give meaningful gradients). The results in \Cref{fig:trajectories_3_panel} are for $\lambda^*=10.0$, and again we note that similar behaviour was observed across $\lambda$ values. We generated 4,000 samples with 100 steps, and for mean-field sampling we resampled using residual resampling every 34 steps for the visualisation in \Cref{fig:trajectories_3_panel}, and every 10 steps in the other examples (we observed very similar performance for different resampling schemes). For the results in \Cref{fig:trajectories_3_panel,fig:objective_curves_main} we used the fixed-point iterations with no damping, and used 3 fixed-point iterations for the visualisation in \Cref{fig:trajectories_3_panel}. To evaluate the KL divergence, we used a leave-one-out estimator on pointclouds, using a Gaussian RBF kernel with bandwidth 0.2.  The low-dimensional experiments were run on a single NVIDIA GeForce RTX 2080Ti GPU.

\subsubsection{Additional Results}

\paragraph{Objective results for different $\lambda$ values}
In \Cref{fig:objective_curves_main}, we showed how the value of the objective optimality gap $J(\mu) - J(\mu^*)$ changes for different steering strengths, for standard gradient steering and mean-field steering. Here, $J(\mu) = \lambda^* \, \MMD^2(\mu, \nu) + \KL(\mu \| p_1)$, and we took $\lambda^*=10.0$. In particular, we saw that the optimum steering strength for mean-field steering coincided with the true $\lambda^*$, and that this was consistent across different $\sigma_t$ schedules, indicating that the mean-field sampling is correctly targeting the tilting measure $\mu^*$. In contrast, the best choice of steering strength for standard gradient steering did not correspond to the $\lambda^*$ in the objective, and thus would need searching over. In \Cref{fig:combined_schedule_objective_app}, we report the same results for different values of $\lambda^*$, and observe a similar pattern.

\begin{figure}[t]
    \centering

    \begin{subfigure}{\textwidth}
        \centering
        \includegraphics[width=0.8\textwidth]{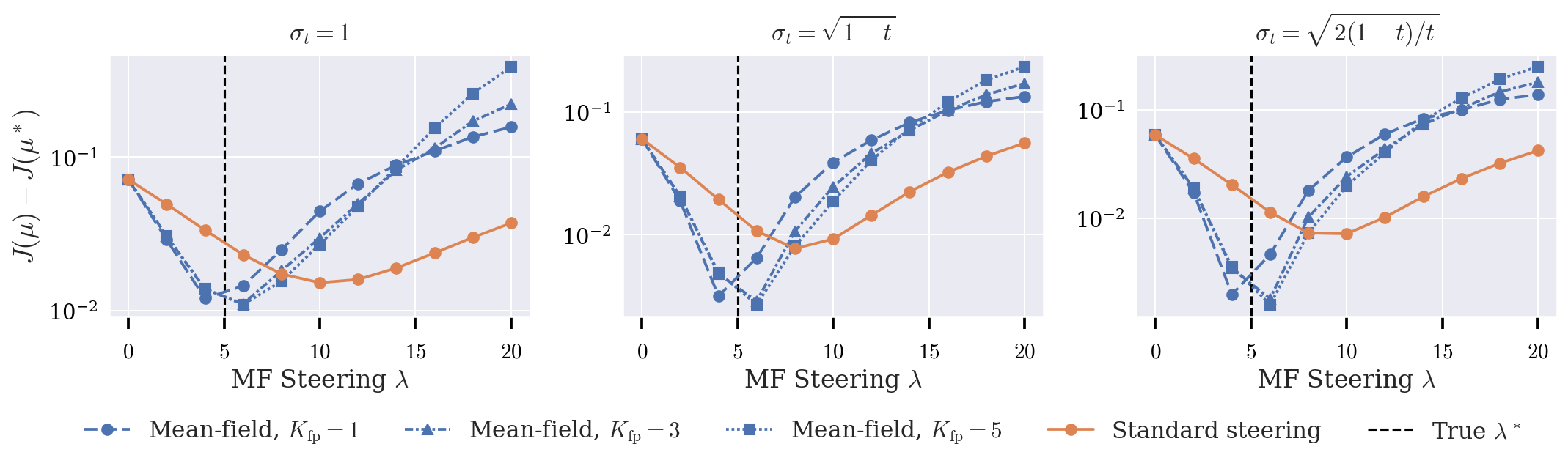}
        \label{fig:plot1}
    \end{subfigure}


    \begin{subfigure}{\textwidth}
        \centering
        \includegraphics[width=0.8\textwidth]{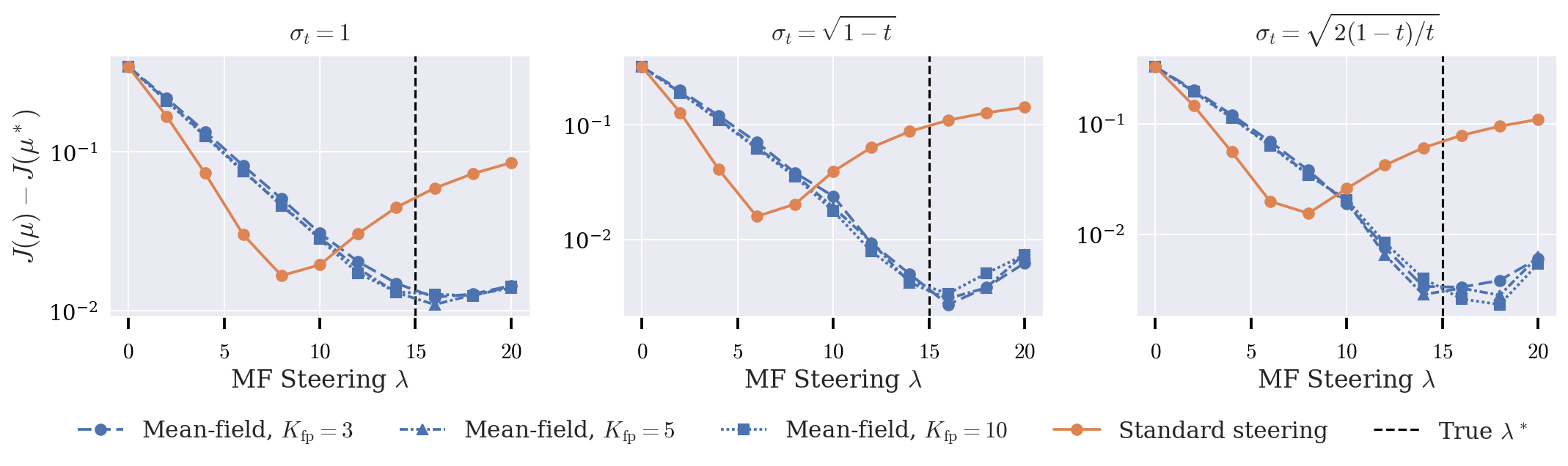}
        \label{fig:plot2}
    \end{subfigure}
    \caption{\textbf{Low-dimensional MMD Example:} Objective optimality gap plots, similar to \Cref{fig:objective_curves_main}, for additional $\lambda^*$ values 5.0 \textit{(top)} and 15.0 \textit{(bottom)}. As in \Cref{fig:objective_curves_main}, the objective gap is minimised when using mean-field steering with the correct value of $\lambda$.}
    \label{fig:combined_schedule_objective_app}
\end{figure}

\paragraph{Comparison of different $\dot \Psi_t$ solvers}
In Section \ref{sec:method}, we provide two possible approaches for computing the term $\dot \Psi_t$, namely a fixed-point procedure (\Cref{alg:Psi_solver_fixedpoint}), and also an implicit linear solver \Cref{alg:Psi_solver_implicit}. In the majority of our experiments, we use the fixed-point solver for simplicity, however here we provide a comparison and discussion regarding the choice of solver.

We generally found the fixed-point solver to converge quickly and accurately. When the tilt strength is not too high, damping in \Cref{alg:Psi_solver_fixedpoint} is not required, however for large tilt strengths we found it helped to ensure the iterations converged. The implicit linear solver worked reliably in all cases, though its implementation is more involved. We visualise results for standard and large tilt strengths $\lambda\in \{10.0, 40.0\}$ in \Cref{fig:toy_mmd_dotpsi_comp}, which reflect the discussion here.

\paragraph{Dependence on $N$}
To assess the dependence on the batch size $N$, we plot the Sinkhorn divergence relative to approximated `ground-truth' $\mu^*$ for different values of $N$ in \Cref{fig:N_dependence}. As expected, the approximation becomes increasingly accurate as $N$ increases, and for large $N$ is very consistent between runs.

\begin{figure}
    \centering
    \includegraphics[width=0.35\linewidth]{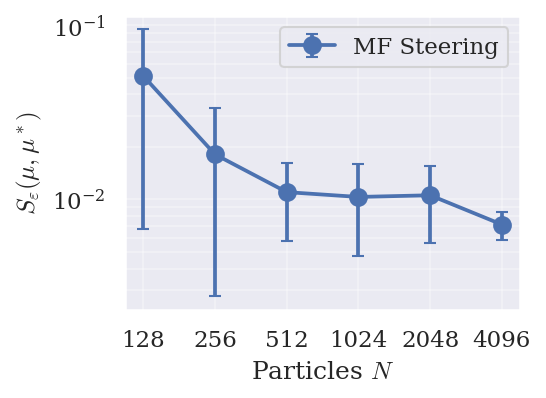}
    \caption{Sinkhorn divergence $\SD_\eps(\mu, \mu^*)$ relative to approximated `ground-truth' $\mu^*$, for the generated mean-field steering samples using different batch sizes $N$ (mean$\pm$std over 8 runs, using entropic regularisation $\eps=0.1$).}
    \label{fig:N_dependence}
\end{figure}

\paragraph{Runtime Analysis}
In \Cref{tab:runtime_comparison}, we also report runtimes for the different approaches in the MMD example. In this low-dimensional setting, the main additional cost comes from the kernel evaluations required to compute $\dot \Psi_t$, which explains the increased inference time relative to gradient steering. However, we can re-use these kernel evaluations across the fixed-point iterations, meaning that subsequent fixed-point iterations do not noticeably increase the sampling time. Generally, we found the implicit linear solver to be slower than the fixed-point approach. We note that in the higher-dimensional cases, the bottleneck increasingly becomes the network evaluation rather than such reward-based computations, meaning the cost of the $\dot \Psi_t$ solver becomes less pronounced.

\begin{table}[t]
\centering
\begin{tabular}{lccccc}
\toprule
Grad & $K_{\mathrm{FP}}=1$ & $K_{\mathrm{FP}}=4$ & $K_{\mathrm{FP}}=10$ & Implicit Linear \\
\midrule
 1.6 & 3.9 & 3.9 & 3.9 & 5.8 \\
\bottomrule
\end{tabular}
\vspace{1em}
\caption{Runtime comparison across steering and solver variants, in the MMD example. Time taken for 100 step simulations (s).}
\label{tab:runtime_comparison}
\end{table}

\begin{figure}[t]
    \centering

    \begin{subfigure}{\textwidth}
        \centering
        \includegraphics[width=0.9\textwidth]{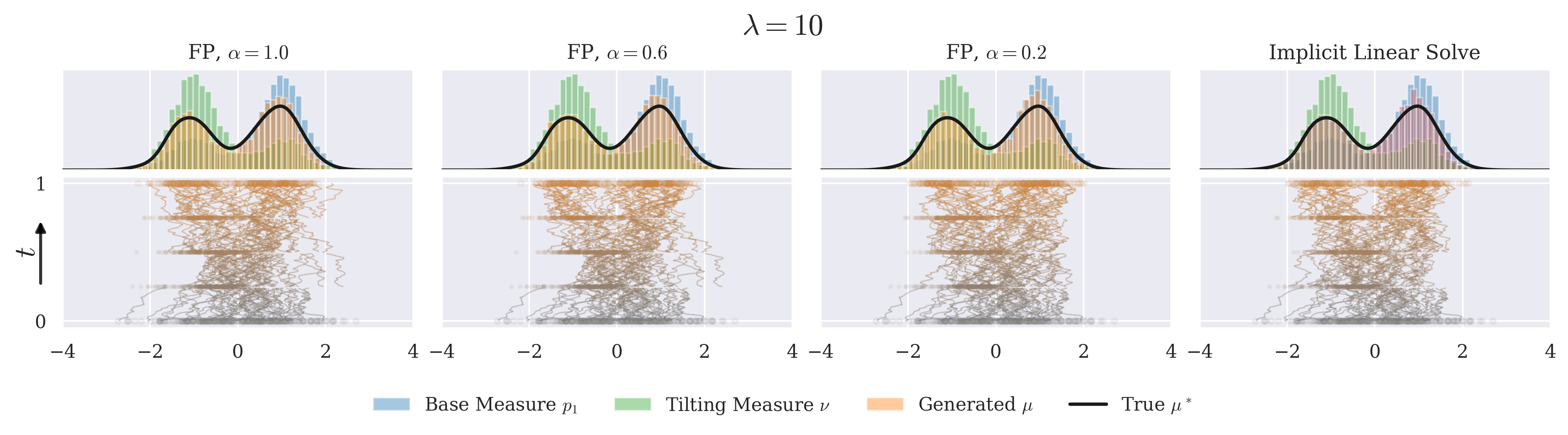}
        \label{fig:plot1}
    \end{subfigure}


    \begin{subfigure}{\textwidth}
        \centering
        \includegraphics[width=0.9\textwidth]{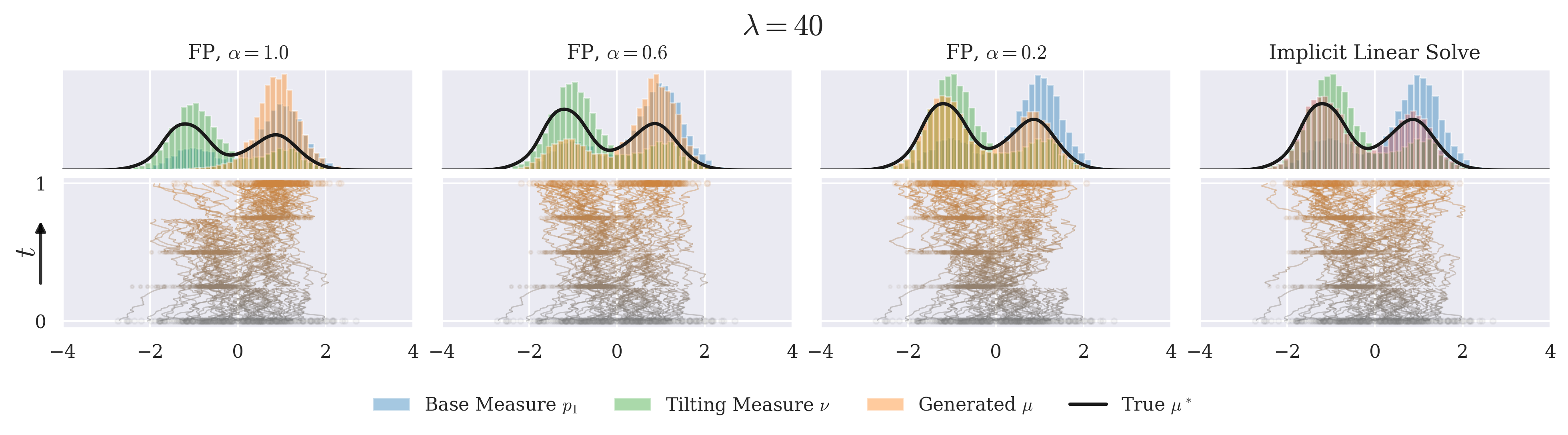}
        \label{fig:plot2}
    \end{subfigure}
    \caption{\textbf{Comparison of $\dot \Psi_t$ solvers, low-dimensional MMD example:} \textit{(Top)} For smaller tilt strength $\lambda=10.0$, the fixed-point procedure converges accurately for all damping levels. \textit{(Bottom)} A stronger tilt strength necessitates using a larger damping factor in order to converge. \textit{(Right)} The implicit linear solver computes the solution accurately in both cases. }
    \label{fig:toy_mmd_dotpsi_comp}
\end{figure}

\subsubsection{Different Reward Functions}

To verify that the same behaviour holds in other settings, we also consider the same experimental setup as above but for some alternative choices of the reward function $\cR(\mu)$.

\paragraph{Sinkhorn Divergence}
The Sinkhorn Divergence $\SD_\eps$ \citep{genevay2018_sinkdiv} is a smooth, debiased approximation to the Wasserstein distance between two measures, regularised by an entropic parameter $\eps$. As it is differentiable, we can steer to minimise the objective
\begin{equation}
    J(\mu) = \lambda \, \SD_\eps(\mu, \nu) + \KL(\mu \| p_1).
\end{equation}
We display obtained trajectories and the corresponding objective optimality gaps for this choice of reward in \Cref{fig:sinkhorn_divergence_results_app}, using $\eps=0.1$ as the regularisation value. Again, $J(\mu)$ is generally minimised by mean-field steering with $\lambda=\lambda^*$, confirming that our procedure is targeting the correct $\mu^*$ in this setting too. This is also confirmed visually in the histogram plots.

\begin{figure}[t]
    \centering

    \begin{subfigure}{\textwidth}
        \centering
        \includegraphics[width=0.9\textwidth]{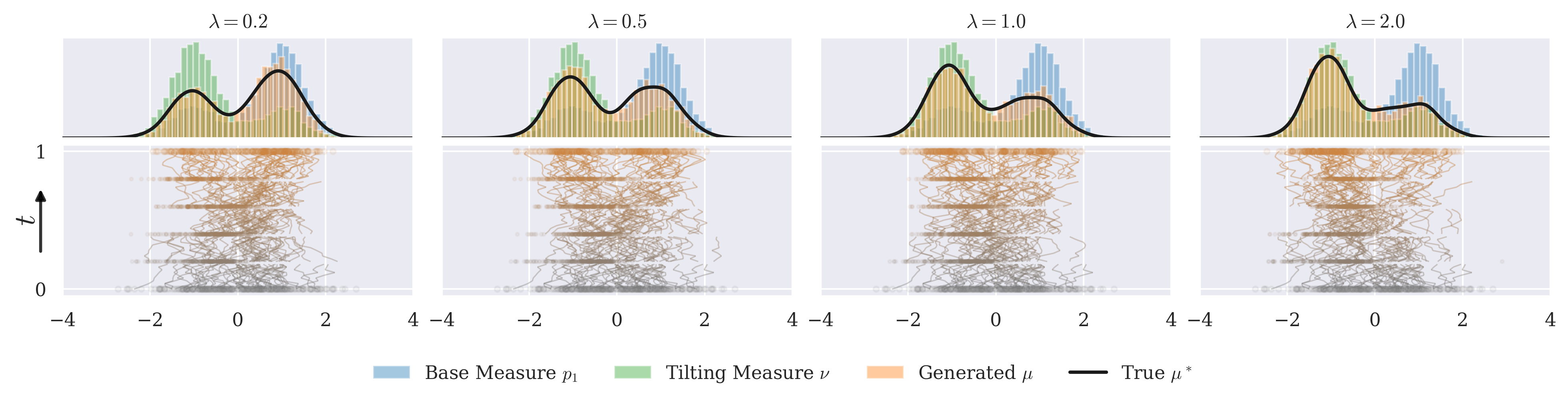}
        \label{fig:plot1}
    \end{subfigure}


    \begin{subfigure}{\textwidth}
        \centering
        \includegraphics[width=0.8\textwidth]{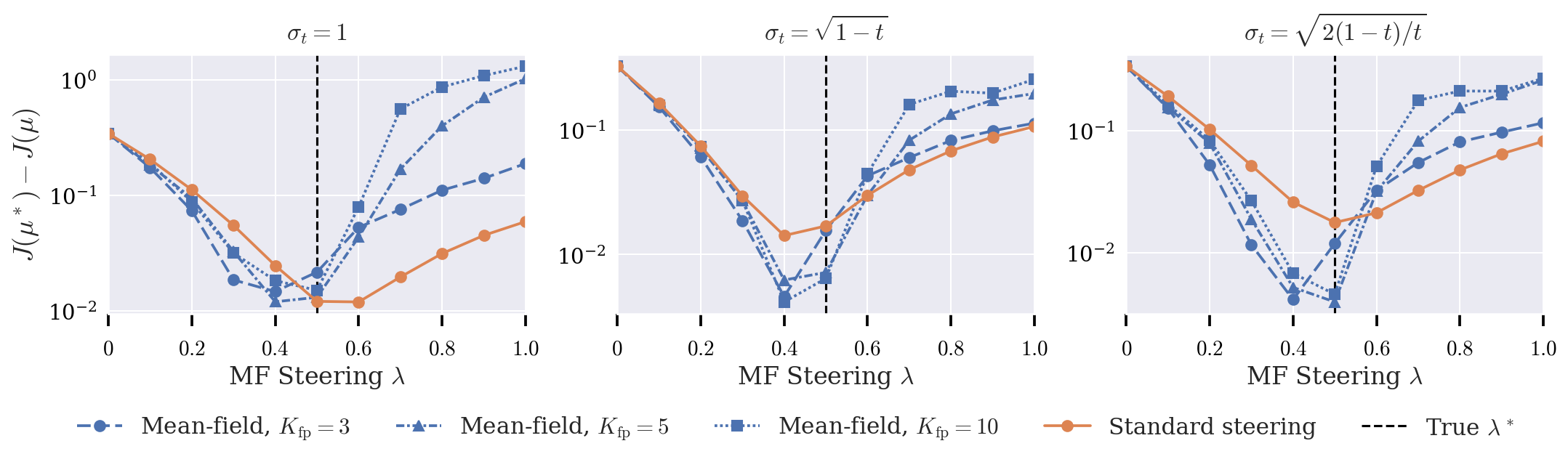}
        \label{fig:plot2}
    \end{subfigure}
    \caption{\textbf{Low-dimensional Sinkhorn Divergence Example:} \textit{(Top)} Comparison of trajectories obtained using mean-field steering, for increasing $\lambda$. The $\dot \Psi_t$ use iterations $\{3,3,10,20\}$ with damping parameters $\{1.0, 1.0, 0.6, 0.2\}$ respectively. \textit{(Bottom)} Comparison of the optimality gap $J(\mu)-J(\mu^*)$ for $\lambda^*=0.5$, for mean-field and gradient-only steering. The gradient-steering coefficient is varied linearly in $[0, 3.0]$.}
    \label{fig:sinkhorn_divergence_results_app}
\end{figure}

\paragraph{Entropy}
We also do similarly using the entropy $\cR(\mu) = - \int \log \mu \, \rmd \mu$. Rather than steering towards a `tilting' measure $\nu$, such a reward encourages the particles to spread out. We plot the same results as above for this new reward in \Cref{fig:entropy_results_app}, and observe similar behaviour in this setting too.

\begin{figure}[t]
    \centering

    \begin{subfigure}{\textwidth}
        \centering
        \includegraphics[width=0.9\textwidth]{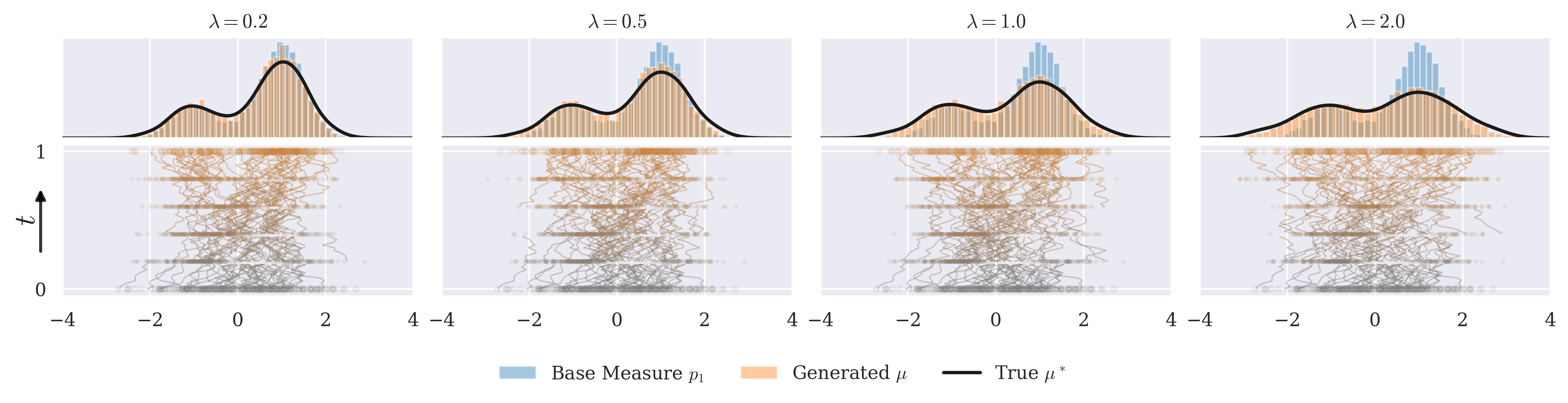}
        \label{fig:plot1}
    \end{subfigure}


    \begin{subfigure}{\textwidth}
        \centering
        \includegraphics[width=0.8\textwidth]{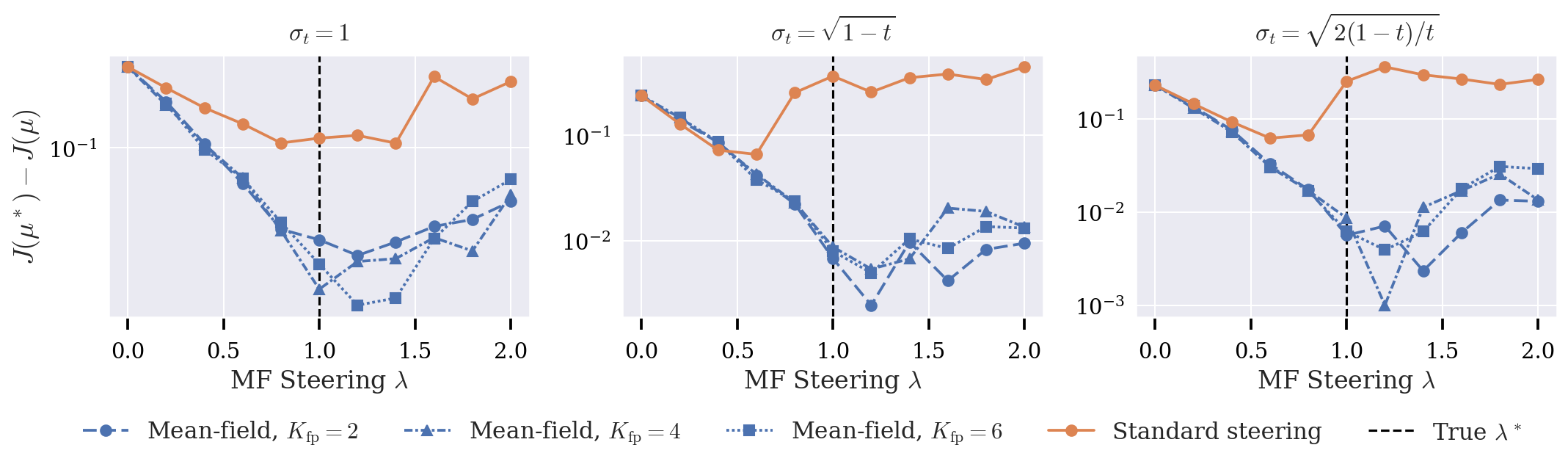}
        \label{fig:plot2}
    \end{subfigure}
    \caption{\textbf{Low-dimensional Entropy Example:} \textit{(Top)} Comparison of trajectories obtained using mean-field steering, for increasing $\lambda$. The $\dot \Psi_t$ use iterations $\{3,3,5,10\}$ with damping parameter $\alpha \in \{1.0, 1.0, 1.0, 0.5\}$ respectively.\textit{(Bottom)} Comparison of the optimality gap $J(\mu)-J(\mu^*)$ for $\lambda^*=1.0$, for mean-field and gradient-only steering.}
    \label{fig:entropy_results_app}
\end{figure}

\subsubsection{Gaussian Ground-Truth Validation in Higher Dimensions}

Above, we validated that our mean-field steering procedure correctly targets the true $\mu^*$ in several 1$d$ settings, in which the ground-truth is not available in closed form but can nevertheless be accurately approximated in a tractable manner. We now provide an additional example to validate our method in higher dimensions. Numerical optimisation to obtain an approximation to the ground-truth $\mu^*$ is no longer tractable, so we consider a setting in which $\mu^*$ is available in closed form. Specifically, we consider a setting where the base diffusion samples a standard Gaussian $p_1 = \mathcal{N}(0, I_d)$, and the reward is defined based on the mean of $\mu$ as $\mathcal{R}(\mu) = - \tfrac{\lambda}{2} \lVert \mathbb{E}_{\mu}[X] - a \rVert_2^2$ for some $a\in \mathbb{R}^d$. In this setting, the ground-truth can be computed analytically and is given in closed form as $\mu^* = \mathcal{N}( \tfrac{\lambda}{1+\lambda} a, I_d)$.

In \Cref{tab:Gaussian_BW_distances} below, we report the Bures-Wasserstein distance to the ground-truth when $\lambda^*=1.0$ and $a = \mathbf{1}$. We run the procedure for different values of $\lambda$, and as in the existing examples we again see that the value is minimised for the true value $\lambda=\lambda^*$, indicating that our approach again targets the correct $\mu^*$ (mean $\pm$ std, over 8 runs). The experiments are run with 100 steps, and 10,000 samples.

\begin{table}[h]
\centering
\small
\begin{tabular}{lccccc}
\toprule
\textbf{Steering $\boldsymbol{\lambda}$:} & $\mathbf{0.6}$ & $\mathbf{0.8}$ & $\boldsymbol{\lambda^*=1.0}$ & $\mathbf{1.2}$ & $\mathbf{1.4}$ \\
\midrule
$\mathbf{d=5}$ & $0.305 \pm 0.009$ & $0.196 \pm 0.025$ & $0.150 \pm 0.013$ & $0.192 \pm 0.017$ & $0.245 \pm 0.014$ \\
$\mathbf{d=10}$ & $0.515 \pm 0.022$ & $0.379 \pm 0.018$ & $0.346 \pm 0.014$ & $0.366 \pm 0.018$ & $0.429 \pm 0.018$ \\
$\mathbf{d=20}$ & $0.830 \pm 0.011$ & $0.664 \pm 0.021$ & $0.610 \pm 0.021$ & $0.625 \pm 0.014$ & $0.675 \pm 0.028$ \\
\bottomrule
\end{tabular}
\vspace{1em}
\caption{Bures-Wasserstein distance to the ground-truth $\mu^*$ corresponding to a true tilting strength $\lambda^*=1.0$ in the Gaussian example, for different steering strengths $\lambda$. The distance is minimised for the correct value $\lambda=\lambda^*$, indicating that $\mu^*$ is being targeted correctly.}
\label{tab:Gaussian_BW_distances}
\end{table}

We also report the dependence on the number of particles $N$ in \Cref{tab:Gaussian_N_dependence}, for the same Gaussian example with $d=10$.

\begin{table}[h]
\centering
\begin{tabular}{cccc}
\toprule
$\mathbf{N=1{,}000}$ & $\mathbf{N=3{,}000}$ & $\mathbf{N=10{,}000}$ & $\mathbf{N=30{,}000}$ \\
\midrule
$0.742 \pm 0.073$ & $0.478 \pm 0.023$ & $0.330 \pm 0.018$ & $0.287 \pm 0.016$ \\
\bottomrule
\end{tabular}
\vspace{1em}
\caption{Bures-Wasserstein distance to the ground-truth $\mu^*$ corresponding to a true tilting strength $\lambda^*=1.0$ in the Gaussian example, for different numbers of particles $N$.}
\label{tab:Gaussian_N_dependence}
\end{table}


\subsection{HIV-1 Protease}
\label{app:v6apo_details}

Recently, there have been many works that have considered steering diffusion models for protein conformation generation according to experimental observations (see \citet{Fadini2026, Liu25_ExEnDiff, levy2025solvinginverseproblemsprotein, maddipatla2025inverse, Maddipatla2026_nature, li2026robustinferencetimesteeringprotein}, to name a few). In Section \ref{sec:HIV_protease_main}, we first considered an example in which we steer generations for HIV-1 protease according to an experimentally observed distribution over residue distances \citep{Liu16_HIV1conformation}.
HIV-1 protease is a symmetric homodimer with 99-residues in each monomer, resulting in generating a 1536 atom system giving a full dimensionality of 4608.
The system exhibits an important `flap region', which opens and closes between different conformational states (\Cref{fig:v6apo_visualisation}). The double electron-electron resonance (DEER) method can be used to infer proportions across different states by attaching spin-labels to selected residues and approximately measuring the distance between them, aggregated over the conformational ensemble. We will consider the v6 variant of the protease as studied in \citet{Liu16_HIV1conformation}; this is known from X-ray crystallography and MD modelling to exhibit four different states: \textit{closed},\textit{ semi-open}, \textit{wide-open}, and \textit{tucked}. The DEER data analysed by \citet{Liu16_HIV1conformation} suggests that these states are occupied by proportions 21\% closed, 61\% semi-open, 8\% wide-open, and 10\% tucked respectively. However, samples from Boltz-2 are strongly biased towards the closed state by comparison, with the states being sampled with proportions 79\% closed, 21\% semi-open with no tucked or wide-open generations (\Cref{fig:v6apo_histograms_mf}, left). Following the analysis from \citet{Liu16_HIV1conformation}, we will aim to tilt towards a Gaussian mixture $\nu$ defined over residue 55 distances, with relative proportions 25\% and 75\% (adjusted for removing the tucked and wide-open states), and optimise the objective
\begin{equation}
    \label{eq:MMD_objective_proj}
    \mu^* = \argmin_\mu \big\{ \KL( \mu \| p_1) + \lambda \, \MMD^2( T \# \mu , \nu  )  \big\},
\end{equation}
where $T$ denotes the map to the residue-55 $C_\alpha-C_\alpha$ distance. For the MMD computation, we used a Gaussian RBF kernel on this residue-distance space, with bandwidth 2.5.

\paragraph{Experimental Setting}

We implement our mean-field steering method in Boltz-2, performing the appropriate modifications for the variance-exploding formulation used by the model. We use 50 steps, and slightly adjust the time schedule to spend longer in the middle of the generative procedure (according to a sigmoid-shaped CDF), as this is where the bifurcation into the distinct states appears to occur. We set the step-scale parameter of Boltz-2 to 1.5, which is a common value to control the temperature from which the distribution is sampled. In all Boltz-2 experiments, we use a sample size of 64 and generate 1024 samples, using the multi-batch procedure in \Cref{alg:interacting_particle_steering_multibatch}. We use residual resampling at every step (though we also obtained similar results with less resampling). We observed occasional spikes in the derivative through the denoiser, so for stability we clipped the maximum change in the log-weights to 1.0 at each step, and clipped the steering term $\nabla \Psi_t(X_t^i, \hat \mu_t)$ so that its norm did not exceed 10\% of the norm of the standard drift $b_t + \eps_t s_t$. As the cost of the Picard iterations is negligible (see \Cref{tab:compute_cost_profiling}), we used $K_{FP}=50$ fixed-point iterations with a damping factor of $\alpha=0.5$ to be confident of convergence (we remark that this is likely very conservative). The Boltz-2 experiments were run on a single NVIDIA GH200 GPU.

Note that when evaluating the objective $J(\mu)$, it is intractable to evaluate the $\KL( \mu \| p_1)$ term in the full ambient space, so instead we compute a proxy on the 1$d$ projection onto the residue-55 distances. As the data exhibits two prominent modes with different scales, we use a bandwidth of 0.2 for the left mode and 0.4 for the right mode when computing the KL divergence proxy; other than this, we follow the same procedure as in the low-dimensional experiments.

\subsubsection{Additional Results}

In \Cref{fig:v6apo_histogram_comp_app}, we provide similar histogram plots to the ones in \Cref{fig:v6apo_histograms_mf} for the standard gradient-steering approach, which corresponds to the mean-field steering but without the corrective weights, with the mean-field steering results plotted again for comparison. We see that the standard gradient-steering, while steering towards the targets, results in a generated distribution that is slightly different in shape to $\mu^*$, with a stronger peak. In contrast, mean-field steering exhibits an accurate fit to $\mu^*$.

In \Cref{fig:v6apo_objectives_app}, we also plot the objective curves for values $\lambda^*\in \{1,3\}$, in addition to $\lambda^*=2$ also reported in \Cref{fig:v6apo_objective}. For each, we again see that the best choice of $\lambda$ indeed coincides with the true value $\lambda^*$, indicating that our method is correctly targeting $\mu^*$. In contrast, using strength $\lambda=\lambda^*$ is not the best choice of steering strength when using gradient-only steering.

\begin{figure}[!t]
    \centering

    \begin{subfigure}{\textwidth}
        \centering
        \includegraphics[width=\textwidth]{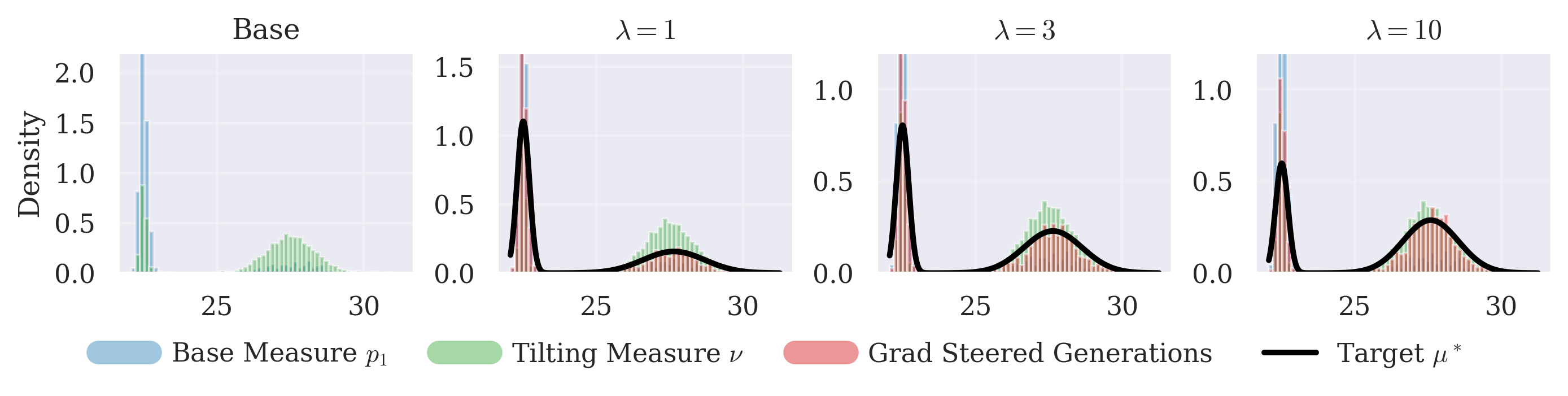}
        \label{fig:plot1}
    \end{subfigure}


    \begin{subfigure}{\textwidth}
        \centering
        \includegraphics[width=\textwidth]{figures/neurips/v6apo_autodiff_histograms_base_target_mf_gmm_lambdas_1_3_10.png}
        \label{fig:plot2}
    \end{subfigure}
    \caption{\textbf{HIV-1 Protease Histogram Comparison:} Comparison of the histograms generated for using gradient-only steering \textit{(top)}, and the same steering but now using the mean-field reweighting \textit{(bottom)}.
    }
    \label{fig:v6apo_histogram_comp_app}

    \vspace{1em}

    \begin{subfigure}{0.32\textwidth}
        \centering
        \includegraphics[width=0.9\linewidth]{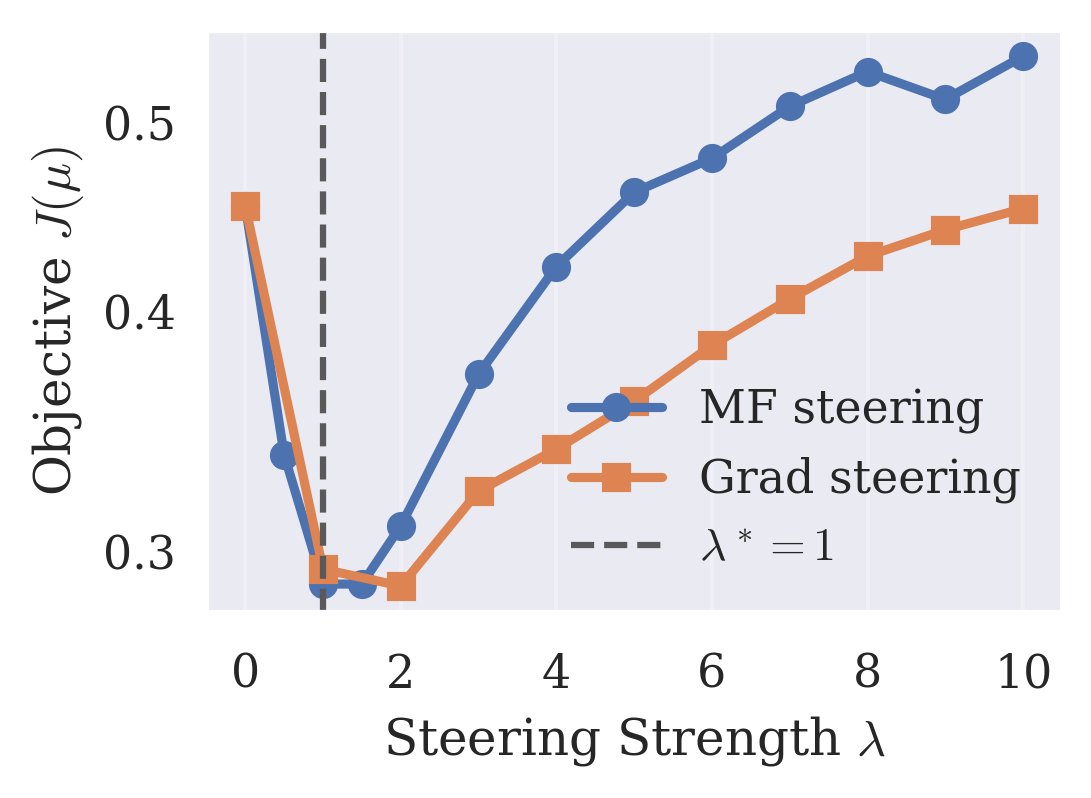}
        \caption{$\lambda^*=1.0$}
        \label{fig:first}
    \end{subfigure}
    \hfill
    \begin{subfigure}{0.32\textwidth}
        \centering
        \includegraphics[width=0.9\linewidth]{figures/neurips/v6apo_autodiff_objectives_mf_grad_true_lambda_2p0.png}
        \caption{$\lambda^*=2.0$}
        \label{fig:second}
    \end{subfigure}
    \hfill
    \begin{subfigure}{0.32\textwidth}
        \centering
        \includegraphics[width=0.9\linewidth]{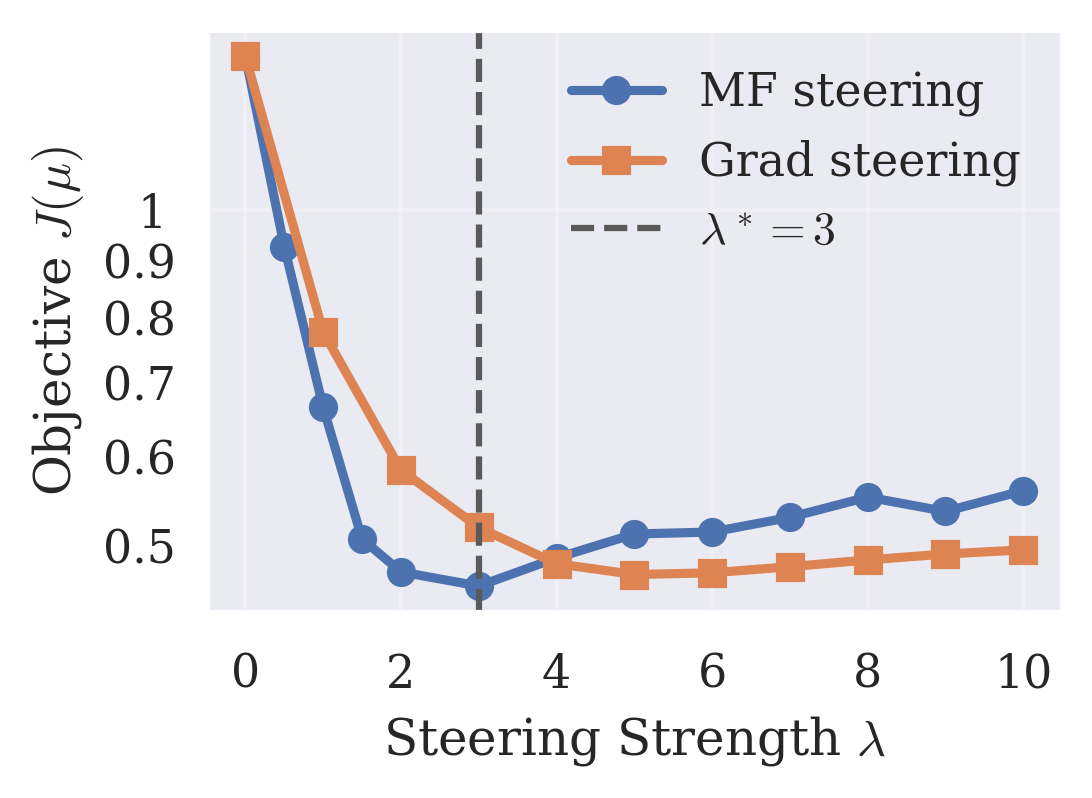}
        \caption{$\lambda^*=3.0$}
        \label{fig:third}
    \end{subfigure}
    \caption{\textbf{HIV-1 Protease Objective:} Values of $J(\mu) = \lambda^* \, \MMD^2(\mu, \nu) + \KL(\mu \| p_1)$, for the different steering mechanisms with varying steering strengths.}
    \label{fig:v6apo_objectives_app}
\end{figure}

\subsection{Adenylate Kinase}
\label{app:Adk}

Another route towards observing a protein's conformational landscape is through molecular dynamics (MD) simulations. Atomistic MD is accurate but very expensive; this has motivated cheaper coarse-grained simulations to infer information about the higher-level structure. Such simulations could be used to infer approximate distributions over certain collective variables, such as important angles or residue distances. Given such information, one may then wish to `tilt' the full-atom generations of a diffusion model towards such approximate distributional properties of the higher-level structure.

As an illustrative example of this idea, we consider another example using Boltz-2 in which we sample the conformational ensemble of adenylate kinase (AdK). This consists of 1664 atoms giving a dimensionality of 4992. AdK is known to transition between open and closed states through motions that are well-characterised through two collective variables: the LID-CORE angle and the NMP-CORE angle (\Cref{fig:adk_visualisation}). From MD simulations, one can observe the joint distribution over these two angles. In our experiment, we consider tilting according to data from \citet{BECKSTEIN2009160}, which is displayed as the underlying heatmap in \Cref{fig:adk_mf_heatmaps}.

We generate the tilting data by sampling 4,000 points according to the densities provided by \citet{BECKSTEIN2009160}. We again optimise for an MMD objective as in \eqref{eq:MMD_objective_proj}, in which the MMD is now evaluated on the 2-dimensional angle space using an isotropic kernel bandwidth of 10.0. We stop the reweighting procedure at time 0.75 so that the duplicated particles become sufficiently dispersed in the remaining generation time. For the plot in \Cref{fig:top_left_proportion}, we classify a generation as in the `upper-left' region if the NMP and LID angles are in the ranges $(40 \degree, 55\degree)$ and $(112\degree, 150\degree)$ respectively.
We remark that we observed negligible change in the proportion of steric clashes in the generations for the steering strengths we considered. Across both steering mechanisms, the mean clash rate per 1000 heavy atoms varied between 5.9 and 6.3 across the different strengths. It is unfortunately difficult to reliably evaluate a 2$d$ proxy for the $\KL$ divergence in this example, as $p_1$ does not place sufficient mass in the upper-left region. In \Cref{fig:adk_mmd_plot_app} we plot just the squared MMD value, which similarly to \Cref{fig:top_left_proportion,fig:adk_mf_heatmaps} demonstrate that mean-field steering produces a stronger shift towards $\nu$.

\begin{figure*}[!t]
    \centering

    \begin{minipage}{0.3\textwidth}
        \centering
        \includegraphics[width=\textwidth]{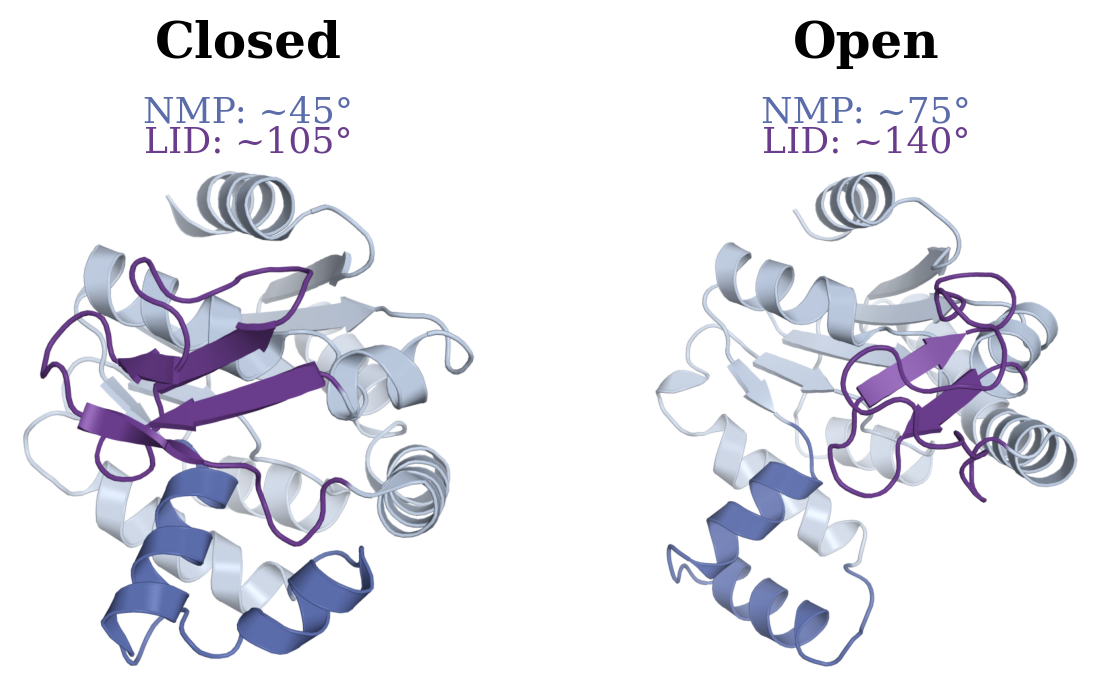}
        \caption{Open and closed conformations of Adenylate Kinase.}
        \label{fig:adk_visualisation}
    \end{minipage}
    \hfill
    \begin{minipage}{0.3\textwidth}
        \centering
        \includegraphics[width=0.9\textwidth]{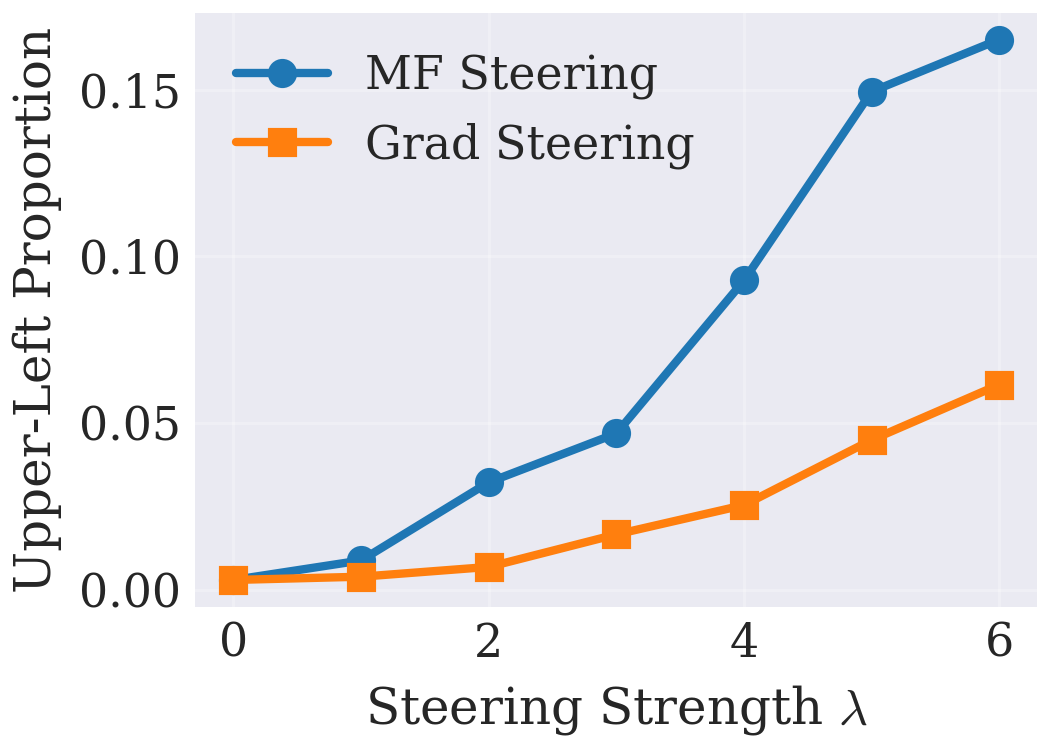}
        \caption{The proportion of generations in the `upper-left' region, for different $\lambda$.}
        \label{fig:top_left_proportion}
    \end{minipage}
    \hfill
    \begin{minipage}{0.3\textwidth}
        \centering
        \includegraphics[width=0.9\textwidth]{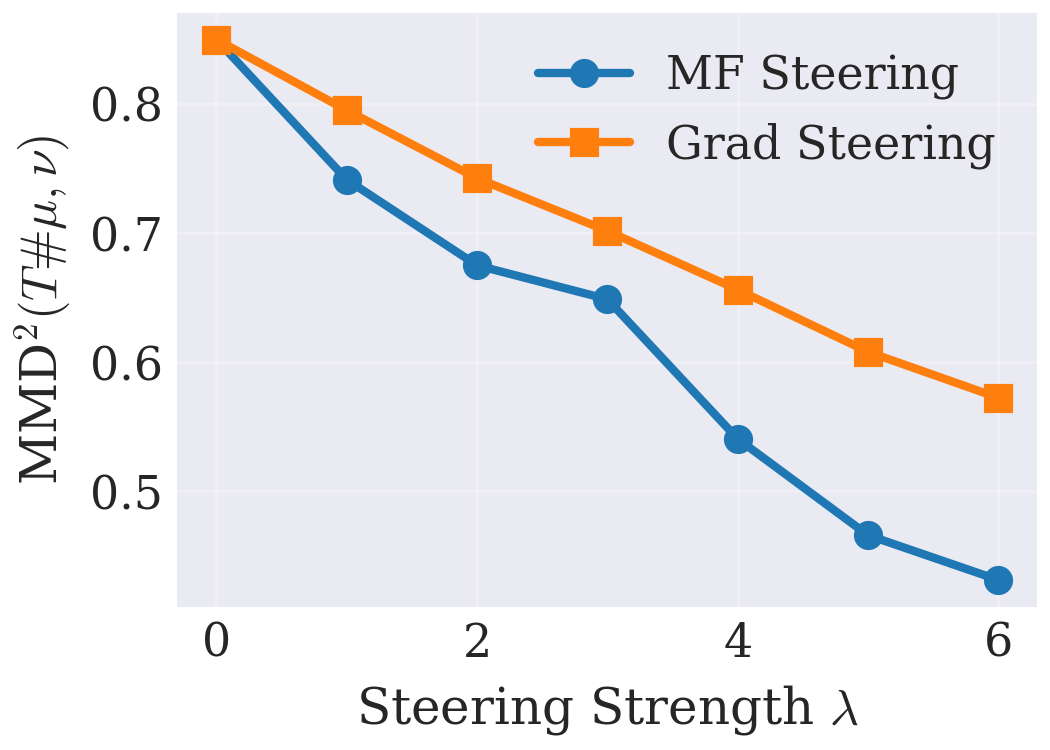}
        \caption{$\MMD^2(T\# \mu, \nu)$ for increasing steering strength $\lambda$.}
        \label{fig:adk_mmd_plot_app}
    \end{minipage}

    \vspace{1.5em}

    \begin{subfigure}{0.49\textwidth}
        \centering
        \includegraphics[width=\linewidth]{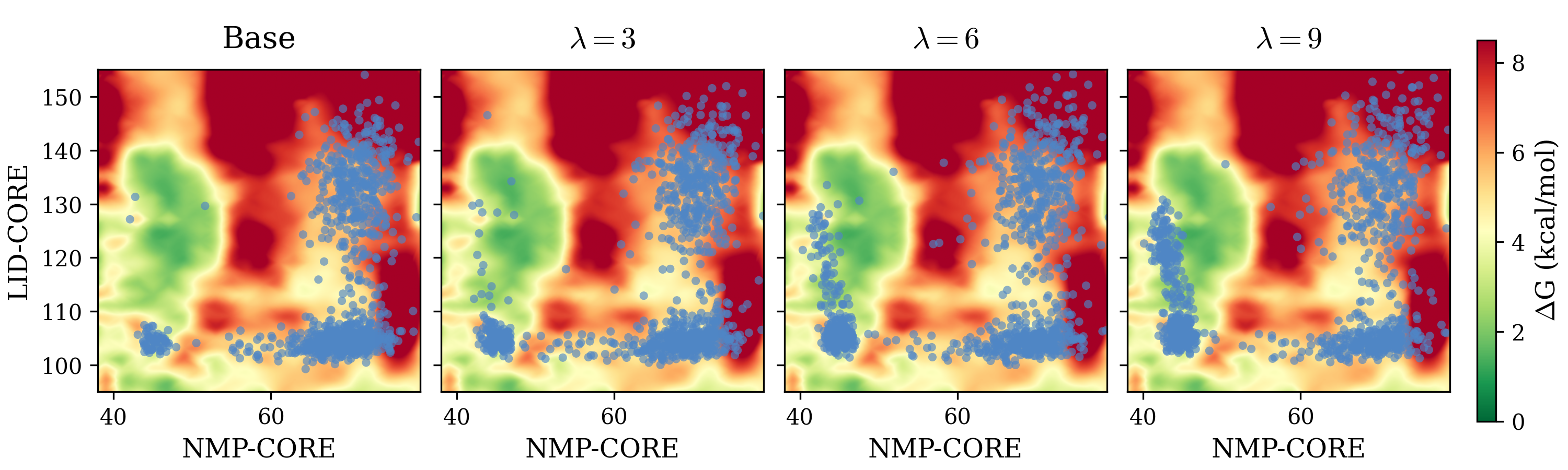}
        \caption{Gradient-only steering.}
        \label{fig:adk_gradient_heatmaps}
    \end{subfigure}
    \hfill
    \begin{subfigure}{0.49\textwidth}
        \centering
        \includegraphics[width=\linewidth]{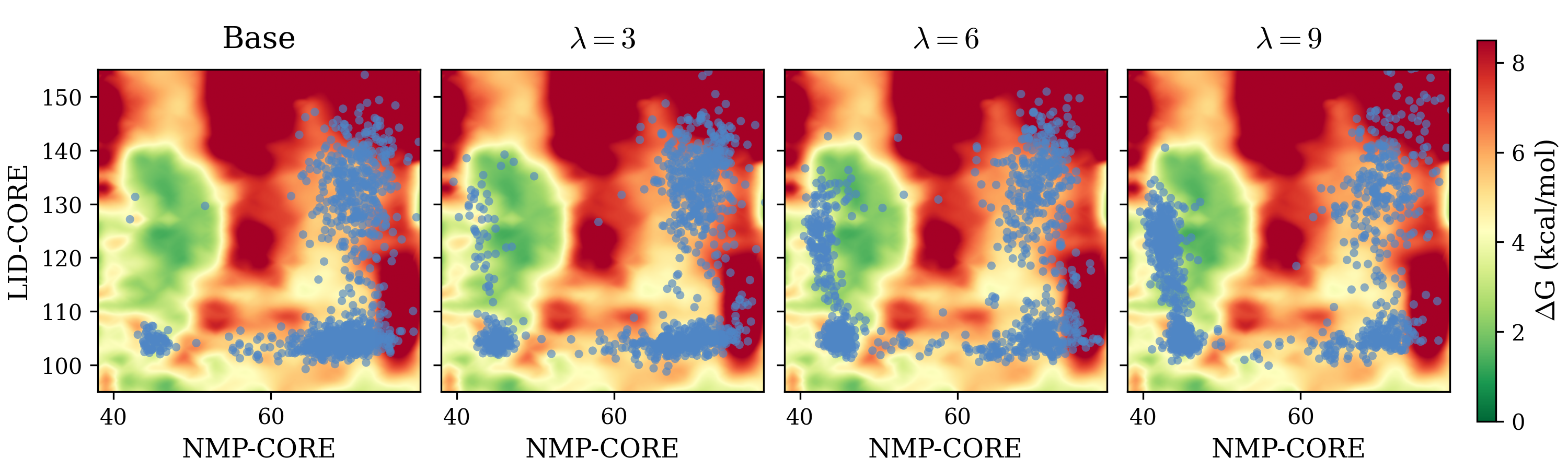}
        \caption{Mean-field steering.}
        \label{fig:adk_mf_heatmaps_subfig}
    \end{subfigure}
    \caption{
    Scatter plots for the LID and NMP angles for the generated data, with overlaid empirical heatmap from \citet{BECKSTEIN2009160}.
    }
    \label{fig:adk_mf_heatmaps}

\end{figure*}

\subsection{Analysis of Computational Cost}
\label{app:comp_cost}
We now examine the additional computational cost of our approach compared to standard gradient-only distributional steering.
The cost of the interactions between the particles in the \textit{positional} updates is the same as in standard gradient-only steering; the additional overhead incurred by our approach is limited to the solving of the 1-dimensional log-weight dynamics. We note that as network size increases, the cost is dominated by the network evaluations. Generally, our approach will require an additional $\sim$50\% of runtime and no additional peak memory usage compared to gradient-only steering; we explain this by breaking down the cost of each operation below.

Our approach requires 3 different types of network evaluation: the evaluation at $X_t$ (a forward pass), the evaluation of $\nabla_x \Psi_t$ and $\partial_t \Psi_t$ (a backward pass), and the drift evaluation at the next step $X_{t+\delta t}$ for the finite-difference computation of $\dot \Psi_t$ (a forward pass). Gradient-only steering already requires the first two passes; our approach incurs only the final forward pass at the next step as an additional cost, leading to an approximately 50\% increase in runtime per step. The peak memory requirements are unchanged because the final additional forward pass can be computed without gradient tracking so it incurs comparatively little memory overhead. The cost of the Picard iterations is negligible compared to the cost of the network evaluations.

These observations are confirmed by the profiling of the per-step computational cost (reported in \Cref{tab:compute_cost_profiling}) in the HIV-1 Protease example, which reports the proportion of per-step runtime and the peak memory usage of each of these operations (similar results are observed in the other experiments too). As expected, increasing the number of Picard iterations adds negligible computational cost. In \Cref{tab:runtime_comparison_protein}, we also report sampling runtimes for both steering mechanisms in the two protein examples, which are consistent with the runtime discussion above. As expected, the proportional overhead of the $\dot \Psi_t$ solver is lower compared to the timings in \Cref{tab:runtime_comparison}, as the cost of evaluating the network is now much higher. Again, the results confirm that the additional cost of further fixed point iterations is negligible.

\begin{table}[!t]
\centering
\resizebox{\textwidth}{!}{%
\begin{tabular}{lcccc}
\toprule
\textbf{Sampler operation} & \textbf{1 Picard iteration} & \textbf{5 iterations} & \textbf{10 iterations} & \textbf{Peak memory usage} \\
\midrule
Network forward at $X_t$
& 32.41\% & 32.38\% & 32.34\% & 38.4 GB \\
Backward for $\nabla_x \Psi$ and $\partial \Psi_t$
& 35.56\% & 35.33\% & 35.29\% & 37.47 GB \\
Network forward at $X_{t+\delta t}$
& 32.19\% & 32.16\% & 32.12\% & 4.42 GB \\
$\Psi$ evaluation and Picard iterations
& 0.036\% & 0.128\% & 0.244\% & 2.01 GB \\
\bottomrule
\end{tabular}%
}
\vspace{1em}
\caption{Breakdown of computational costs for the different operations performed in each sampling step, in the HIV-1 protease experiment, measured on an NVIDIA L40S GPU. Note that the first two operations are required for gradient-only steering too.}
\label{tab:compute_cost_profiling}

\vspace{1em}

\centering
\begin{tabular}{lccccc}
\toprule
& Grad & $K_{\mathrm{FP}}=1$ & $K_{\mathrm{FP}}=4$ & $K_{\mathrm{FP}}=10$\\
\midrule
HIV-1 & 44s & 69s & 69s & 69s \\
AdK & 48s & 75s & 75s & 75s \\
\bottomrule
\end{tabular}
\vspace{1em}
\caption{Runtime comparison across steering and solver variants, in the protein examples. Time taken for 50 step simulations, with a batch size of 64, measured on an NVIDIA L40S GPU.}
\label{tab:runtime_comparison_protein}
\end{table}

\subsection{Tilting Towards X-ray Crystallographic Electron Densities}
\label{app:electron_densities}

\paragraph{X-Ray Crystallography} Another approach for probing the conformational behaviour of proteins is X-ray crystallography. This process involves crystallising a purified sample of the protein, and analysing the resulting X-ray diffraction pattern to produce an experimentally-observed 3-dimensional electron density. This electron density is an average over all molecules in the crystal, so when the molecules exhibit heterogeneous conformations in the structure then the density cannot be explained by a single conformation prediction alone. Aligning a generative model towards such experimentally-observed information therefore again requires using a distributional reward function.

Motivated by the fact that many experimental techniques give ensemble-level information, recent works \citet{maddipatla2025inverse, Maddipatla2026_nature} have considered steering Protenix \citep{protenixv0} (an open-source implementation of AlphaFold3 \citep{abramson2024_AF3}) using the gradient of distribution-level reward functions to guide the generation.
For an amino acid sequence $a$, and experimental observation $y$ dependent on the ensemble $\cX = \{ X^1, ..., X^n \}$, they aim to sample an ensemble from the posterior $p(\cX | a,y)$, which by Bayes' rule satisfies $p(\cX | a,y) \propto p(y | \cX, a) \cdot p(\cX | a)$. Here, $p(y | \cX, a)$ denotes a chosen likelihood function, and the prior $p(\cX | a)$ is given by the base diffusion model.
This corresponds to a finite-sample version of our framework, using the log-likelihood reward function $\cR(\mu) = \log p(y | \mu, a)$.
Steering according to X-ray crystallographic electron densities is an important experimental setting considered in these works.

As with other examples of gradient-only steering, steering with the gradient alone will not sample from the exact posterior distribution $\mu^*$. We therefore here compare standard gradient-steering, akin to the approach in \citet{maddipatla2025inverse, Maddipatla2026_nature}, along with our mean-field steering procedure, to investigate how the inclusion of the FK weights affects the behaviour.

\paragraph{Reward Function}
We largely follow the experimental settings of \citet{maddipatla2025inverse, Maddipatla2026_nature}, with a few modifications to make suitable for our setting. The reward function is chosen as
\begin{equation}
    \cR(\hat \mu) = - \lambda \left \| F_o - \tfrac{1}{N} \sum_{i=1}^N F_c(X^i) \right \|_2^2,
\end{equation}
where $F_o$ denotes the experimentally-observed density, $F_c(X^i)$ denotes a constructed density for the generated structure $X^i$, and $\lambda$ is a strength parameter. We remark that \citet{maddipatla2025inverse, Maddipatla2026_nature} used an $L_1$ distance rather than $L_2$ considered here; we chose to use $L_2$ here as we found it gave more consistent behaviour between experiments. Following \citet{maddipatla2025inverse}, we steer only the atoms that are in the region of interest (ROI) that exhibits the variations in structure. To do so, the electron density $F_c(X^i)$ is constructed using the current generation's ROI atoms combined with non-ROI atoms from the known reference structure in the PDB dataset. For 4OLE, the ROI consists of the residues 423-431 (highlighted in red in \Cref{fig:4OLE_conformations}).

As the densities will not necessarily be on the same scale, we centre and normalise both densities (we stop the gradient with respect to atom positions through the standardisation constants). 
Similar to \citet{maddipatla2025inverse}, the density $F_c(X^i)$ for the generated samples is constructed as a mixture of Gaussian kernels,
\begin{equation}
    F_c(\xi)\;=\;\sum_{i=1}^{m}\;\sum_{j=1}^{6}\;
a_{ij}
\left(\frac{4\pi}{b_{ij}+B_i}\right)^{\!3/2}
\exp\!\left(-\frac{4\pi^{2}}{b_{ij}+B_i}\,
\big\|\mathbf{x}_i-\xi\big\|_2^{2}\right).
\end{equation}
Here, $\xi \in \sR^3$ is the point where the density is being evaluated, and $\mathbf{x}_i \in \sR^3$ is the 3d position of the $i$-th atom, and $a_{ij}, b_{ij}$ are per-atom constants. The $B_i$ factors additionally determine the width of the kernels; the work \citet{maddipatla2025inverse} sets this to $B=4/N$, but we choose to use the $B$ factors from the PDB data as we simulate larger ensembles.
As we are interested in the behaviour of the atoms in the ROI, we construct this density using the predicted ROI atoms, along with non-ROI atom positions from the known reference structure. Therefore the steering acts only on the ROI atoms. Note that before evaluating the reward, the generations are aligned to the reference structure using the non-ROI atoms using the Kabsch algorithm \citep{Kabsch76}. We use the $2mF_{\mathrm{o}}-DF_{\mathrm{c}}$ map of 4OLE as the target density, whose values are given on the FFT grid at approximately 0.84\AA\ spacing. We evaluate the reward using the 13,247 points on the grid within 5\AA\ of the ROI atoms.

The implementation in \citet{maddipatla2025inverse} also includes an additional substructure-conditioning term, which steers the non-ROI atoms towards their known positions in the reference structure. We omit this steering in our implementation, as it would mean that the base generative process would not follow the trajectory of a trained flow model which is required for us to have access to the score of the marginals. In practice, we found the base diffusion process already had good alignment with the structure outside of the ROI, so the generations were reasonable without requiring this additional steering term.

\paragraph{Sampling}
We sample 16 batches of 64 samples, using our multi-batch sampling procedure. We note that the sampled distribution became more consistent as the number of batches increased, which we anticipate may be due to using more terminal samples in the empirical approximation of $\hat \mu_1$ which gives a better approximation (as the terminal samples are more realistic than the Tweedie estimates). We used 100 sampling steps, using residual sampling at each reweighting step. Similarly to in the Boltz-2 examples, we used a middle-heavy schedule to spend more time in the period in which the modes were determined. We used $K_{FP}=200$ Picard iterations with damping factor $\alpha=0.1$ to be confident of convergence (this is again likely very conservative, but adds negligible cost). We stopped applying the reweighting procedure after time 0.75 to allow the particles to become sufficiently dispersed in the final generated ensemble (since in the Protenix sampling procedure, the noise becomes increasingly small towards the terminal time).  The experiments were run on a single NVIDIA GH200 GPU.

\paragraph{Results}
We report results in \Cref{fig:4OLE_comparisons}. In \Cref{fig:cosine_vs_strength_4OLE}, we report how the cosine alignment between the target density $F_o$ and the constructed density $\sum_i F_c(X^i)$ of the generated ensemble changes as the steering strength increases, which reflects changes in the reward $\cR(\mu)$. In \Cref{fig:propB_vs_strength_4OLE} we do likewise for the proportion of the generations closer to conformation B. In \Cref{fig:propB_vs_cosine_4OLE}, we also plot this same proportion against the cosine-alignment obtained using different steering strengths; the displayed results use strengths $\lambda \in \{ 0.0, 0.001, 0.002, 0.003, 0.004, 0.005, 0.0075, 0.01\}$ for gradient-only steering, and $\lambda \in \{ 0.0, 0.001, 0.0015, 0.002, 0.0025, 0.003, 0.0035, 0.004, 0.005\}$ for mean-field steering. The visualisation of the experimental setting in \Cref{fig:electron_densities} is inspired by Figure 3 in \citet{maddipatla2025inverse}; \Cref{fig:base_electron_density} displays three samples from the base diffusion model, and \Cref{fig:steered_electron_density} shows samples obtained using mean-field steering with strength $\lambda=0.003$. For clear visualisation, we display three samples close to conformation A and the three samples closest to conformation B.

In the absence of a known ground-truth, we instead focus on discussing the qualitative differences between the results. We first note that for the same steering strength, mean-field steering appears to result in a stronger reward tilt, suggesting that gradient-only steering understeers relative to the true inverse-problem target $\mu^*$ in this instance. We also observe that mean-field steering placed more samples in the alternate conformation (which appeared more difficult for gradient-only steering). We note that the accuracy of these results is likely limited by using a simple KDE approximation for the generated electron density; as noted by \citet{maddipatla2025inverse}, we anticipate these results could be improved by using a more accurate electron density model.
\section{Proofs}
\label{app:proofs}

In this section, we provide the proofs of the theoretical results in the paper.

\subsection{Proof of Proposition \ref{prop:mean field tilt}}
\label{app:proof mean field tilt}

\

We establish the result under the following (mild) regularity assumption on $\cR$:

\begin{assumption}
The reward functional $\mathcal{R}$ admits a functional derivative, is upper semicontinuous with respect to weak convergence and bounded from above.
\end{assumption}

\begin{proof}[Proof of Proposition \ref{prop:mean field tilt}]

Denote the objective by
\[
\mathcal J(\mu)
:=
\mathcal R(\mu)-\KL(\mu\|p_1).
\]
1) We first show that \(\mathcal J\) admits a maximiser. Since \(\mathcal R\) is bounded from above (say $\sup_{\mu \in \cP(\mathbb{R}^d)} \cR(\mu) \le C_{\cR}$) and \(\KL(\mu\|p_1)\geq 0\), we have
$
\sup_{\mu\in\mathcal P(\mathbb R^d)} \mathcal J(\mu)
<  \infty$.

Let \((\mu_n)_{n\geq 1}\) be a maximising sequence. Then, for \(n\) large enough, we have that
$
\mathcal J(\mu_n)\geq \mathcal R(p_1)-1,
$ because $\KL(p_1\|p_1) = 0$.
It follows that
\[
\KL(\mu_n\|p_1)
=
\mathcal R(\mu_n)-\mathcal J(\mu_n)
\leq
C_{\mathcal R}-\mathcal R(p_1)+1,
\]
and thus the sequence \((\KL(\mu_n\|p_1))_{n \ge 1}\) is bounded. The sublevel sets of the relative entropy are weakly compact, so, after passing to a subsequence, we may assume that
\(
\mu_n \Rightarrow \mu^*,
\)
weakly for some candidate maximiser \(\mu^*\in\mathcal P(\mathbb R^d)\).

Since \(\mathcal R\) is upper semicontinuous with respect to weak convergence, and since \(\KL(\cdot\|p_1)\) is lower semicontinuous, we have
\[
\mathcal J(\mu^*)
=
\mathcal R(\mu^*)-\KL(\mu^*\|p_1)
\geq
\limsup_{n\to\infty}
\left(
\mathcal R(\mu_n)-\KL(\mu_n\|p_1)
\right).
\]
Hence, \(\mu^*\) is indeed a maximiser. This maximiser is also unique by concavity of $\cR$ and strict convexity of $\KL(\cdot \| p_1)$. 

2) We now derive the first-order optimality condition \eqref{eq:distributional_tilt_target}.
For an arbitrary perturbation \(\nu\in\mathcal P(\mathbb R^d)\), set
\[
\mu_\varepsilon=(1-\varepsilon)\mu^*+\varepsilon \nu.
\]
By optimality of \(\mu^*\), the derivative of $
\varepsilon \mapsto \mathcal J(\mu_\varepsilon)
$
at \(\varepsilon=0\) vanishes. By definition, we have
\[
\left.\frac{\rmd}{\rmd\varepsilon}\mathcal R(\mu_\varepsilon)\right|_{\varepsilon=0}
=
\int_{\mathbb R^d}
\Psi(x,\mu^*)\,(\nu-\mu^*)(\rmd x).
\]
We also compute the functional derivative of the $\KL$-regulariser:
\[
\left.\frac{\rmd}{\rmd\varepsilon}\KL(\mu_\varepsilon\|p_1)\right|_{\varepsilon=0}
=
\int_{\mathbb R^d}
\left(
\log\frac{\rmd \mu^*}{\rmd p_1}(x)+1
\right)(\nu-\mu^*)(\rmd x).
\]
As a consequence,
\[
\int_{\mathbb R^d}
\left[
\Psi(x,\mu^*)
-
\log\frac{\rmd \mu^*}{\rmd p_1}(x)
-
1
\right]
(\nu-\mu^*)(\rmd x)
=
0.
\]
Since $\nu$ is arbitrary, and \(\nu-\mu^*\) has total mass zero, we see that
\[
\log\frac{\rmd \mu^*}{\rmd p_1}(x)
=
\Psi(x,\mu^*)-C,
\]
for some constant $C \in \mathbb{R}$.
Equivalently,
\[
\mu^*(\rmd x)
=
e^{-C}e^{\Psi(x,\mu^*)}p_1(x) \, \rmd x,
\]
leading to \eqref{eq:distributional_tilt_target} after enforcing normalisation.
\end{proof}

\subsection{Proof of Proposition \ref{prop:target evolution}}
\label{app:proof target evolution}

In order to construct a particle system that realises the interpolating measures defined in \eqref{eq:targets}, we need the following intermediate result about $\mu_t^*$.
\begin{proposition}
\label{prop:target evolution}
The targets \eqref{eq:targets} satisfy the evolution equation
\begin{equation}
\label{eq:target evolution}
\partial_t \mu^*_t(x) = \mu^*_t(x) \left( h_t(x) - \int_{\mathbb{R}^d} h_t(y) \mu^*_t(y) \, dy \right),
\end{equation}
where we define the \emph{log-growth rate} $h_t(x) := \frac{\rmd}{\rmd t} \left( \Psi_t(x,\mu^*_t) \right)- \nabla \cdot b_t(x) - \nabla \log p_t(x) \cdot b_t(x)$, which is measure-dependent.
\end{proposition}
\begin{remark}
As described by \eqref{eq:target evolution}, the target $\mu_t^*$ is instantaneously reweighted by $h_t$ (relative to its mean). This type of equation appears in several related settings, including SMC and Feynman-Kac type flows \citep{DelMoral2004}, tempering and annealed importance sampling \citep{Neal2001}, and evolutionary replicator dynamics \citep{HofbauerSigmund1998}. In our setting, much like $\Psi_t(x;\mu_t^*)$ in \eqref{eq:targets}, $h_t$ itself depends on the targets $\mu_t^*$.  
\end{remark}

\begin{proof}[Proof of Proposition \ref{prop:target evolution}]
We take logarithms on both sides of \eqref{eq:targets}, leading to
\[
\log \mu_t^*(x)
=
\Psi_t(x,\mu_t^*)
+
\log p_t(x)
-
\log Z_t.
\]
Differentiating in time yields
\[
\frac{\partial_t \mu_t^*(x)}{\mu_t^*(x)}
=
\frac{\rmd}{\rmd t}\Psi_t(x,\mu_t^*)
+
\partial_t \log p_t(x)
-
\partial_t \log Z_t.
\]
The time marginals associated to the base dynamics \eqref{eq:base model} satisfy the continuity equation
\[
\partial_t p_t+\nabla\cdot(p_t b_t)=0,
\]
and thus we obtain
\[
\partial_t \log p_t(x)
=
\frac{\partial_t p_t(x)}{p_t(x)}
=
-\frac{\nabla\cdot(p_t b_t)(x)}{p_t(x)}
=
-\nabla\cdot b_t(x)-\nabla\log p_t(x)\cdot b_t(x).
\]
Therefore, by the definition of \(h_t\),
\begin{equation}
\label{eq:mu proof halfway}
\frac{\partial_t \mu_t^*(x)}{\mu_t^*(x)}
=
h_t(x)-\partial_t\log Z_t.
\end{equation}

It remains to identify \(\partial_t\log Z_t\). 
Differentiating under the integral sign gives
\[
\partial_t Z_t
=
\int_{\mathbb R^d}
e^{\Psi_t(y,\mu_t^*)}p_t(y)
\left[
\frac{\rmd}{\rmd t}\Psi_t(y,\mu_t^*)
+
\partial_t\log p_t(y)
\right]\rmd y =
\int_{\mathbb R^d}
e^{\Psi_t(y,\mu_t^*)}p_t(y)h_t(y)\,\rmd y,
\]
using the definition of \(h_t\) in the second equality. Therefore, we have
\[
\partial_t\log Z_t
=
\frac{\partial_t Z_t}{Z_t}
=
\int_{\mathbb R^d} h_t(y)\,\mu_t^*(y)\,\rmd y,
\]
and combining with \eqref{eq:mu proof halfway} gives \eqref{eq:target evolution}.
\end{proof}

\subsection{Proof of Theorem \ref{th:mf_weighted_dynamics}}

We establish the result under the following assumption:
\begin{assumption}
\label{ass:well posed}
The McKean-Vlasov dynamics \eqref{eq:McKean} are well posed: for the initial condition $X_0 \sim p_0$, the solution exists and is unique.
\end{assumption}

This assumption holds true under mild growth and regularity conditions on the base drift $b_t$ and the reward $\mathcal{R}$ (see, for example, \citet[Section 3]{del2007branching} or \citet[Section 4.2]{carmona_delarue_2018_mfg1}). We highlight the following subtlety: to convert \eqref{eq:McKean} into standard form covered by those references, it is necessary to \emph{define} $\dot{\Psi}$ via the linear equation \eqref{eq:implicit_eq_for_dotPsi} (and verify appropriate assumptions on the associated solution operator), because time derivatives with respect to the measure $\mu_t$ are non-standard in this context.

\begin{proof}[Proof of Theorem \ref{th:mf_weighted_dynamics}]
For convenience, let us introduce the short-hands
\begin{equation}
\label{eq:drift coefficients}
B_t(x,\mu)
:=
b_t(x)+\eps_t s_t(x)+\eps_t\nabla_x\Psi_t(x;\mu),
\qquad
g_t(x,\mu)
:=
b_t(x)\cdot\nabla_x\Psi_t(x;\mu)+\dot\Psi_t(x),
\end{equation}
and define the unnormalised weighted law
\[
\nu_t(\phi):=\mathbb E[\phi(X_t)e^{A_t}].
\]
By It\^o's formula, \(\nu_t\) satisfies the Feynman--Kac equation \citep[Section 3.1]{del2007branching}
\[
\partial_t\nu_t
=
(L_t^{\mu_t})^\dagger\nu_t
+
g_t(\cdot,\mu_t)\nu_t,
\]
where
\begin{equation}
\label{eq:generator}
L_t^\mu \phi
=
B_t(\cdot,\mu)\cdot\nabla\phi+\eps_t\Delta\phi
\end{equation}
is the $\mu$-dependent generator of \eqref{eq:X McKean}, and  $(L_t^{\mu_t})^\dagger$ is its adjoint in $L^2(\mathbb{R}^d)$.
Since \(\mu_t=\nu_t/\nu_t(1)\), the normalised law satisfies
\begin{equation}
\label{eq:normalised-fk-short}
\partial_t\mu_t
=
(L_t^{\mu_t})^\dagger\mu_t
+
\left(
g_t(\cdot,\mu_t)-\mu_t(g_t(\cdot,\mu_t))
\right)\mu_t.
\end{equation}

We now check that that path of target measures \(\mu_t^*\) satisfies the same equation. Since
\(
\mu_t^*(x)=\frac1{Z_t}e^{\Psi_t(x;\mu_t^*)}p_t(x)
\)
and \(s_t=\nabla\log p_t\), we can write the score of the target measures as
\[
s_t+\nabla_x\Psi_t(\cdot;\mu_t^*)
=
\nabla\log\mu_t^*.
\]
Consequently, looking at the expression for \((L_t^{\mu_t^*})^*\mu_t^*\),  the \(\eps_t\)-terms cancel,
\begin{align*}
    (L_t^{\mu_t^*})^*\mu_t^* &= -\nabla\cdot\bigl((b_t + \eps_t \nabla\log\mu_t^* )\, \mu_t^*\bigr) + \eps_t\Delta\mu_t^* \\
    &= - \nabla \cdot (b_t\mu_t^*) - \eps_t \nabla \cdot (\mu_t^* \nabla \log\mu_t^*)  + \eps_t\Delta\mu_t^* \\
    &= -\nabla\cdot(b_t\mu_t^*).
\end{align*}
Expanding this divergence, we have
\begin{align*}
    -\nabla\cdot(b_t\mu_t^*) &= - (\nabla \cdot b_t) \, \mu_t^* - b_t \cdot \nabla \mu_t^* \\
    &= - (\nabla \cdot b_t) \, \mu_t^* - b_t \cdot \bigl (\mu_t^* (s_t + \nabla \Psi_t) )\bigr)  \\
    &= \bigl( - \nabla \cdot b_t - b_t \cdot s_t - b_t \nabla \Psi_t \bigr) \, \mu_t^*
\end{align*}
Noting that $g_t(x,\mu_t^*) \mu_t^* = \bigl( b_t \cdot \nabla \Psi_t + \dot \Psi_t\bigr) \mu_t^*$, we therefore have
\begin{align*}
    (L_t^{\mu_t^*})^*\mu_t^* + g_t(\cdot,\mu_t^*)\mu_t^*
    &= 
    \bigl( \dot \Psi_t - \nabla \cdot b_t - b_t \cdot s_t \bigr) \, \mu_t^* \\
    &= h_t\,\mu_t^*,
\end{align*}
where \(h_t\) is the function from Proposition~\ref{prop:target evolution}. Integrating this expression, and noting that the first term integrates to 0, we see that $\mu_t^*(g_t(\cdot,\mu_t^*)) = \mu_t^*(h_t)$. After subtracting this, we have
\[
(L_t^{\mu_t^*})^*\mu_t^*
+
\left(
g_t(\cdot,\mu_t^*)-\mu_t^*(g_t(\cdot,\mu_t^*))
\right)\mu_t^*
=
\mu_t^*
\left(
h_t-\mu_t^*(h_t)
\right).
\]
By Proposition~\ref{prop:target evolution}, the right-hand side equals
\(
\partial_t\mu_t^*.
\)
Thus \(\mu_t^*\) solves the same normalised Feynman--Kac equation
\eqref{eq:normalised-fk-short} as \(\mu_t\). Since the initial conditions agree and the equation is unique by assumption, we conclude that
\[
\mu_t=\mu_t^*,
\qquad t\in[0,T].
\]

\end{proof}

\subsection{Proof of Proposition \ref{prop:time derivative}}

\begin{proof}
We use the definition $\dot \Psi_t(x) = \tfrac{\rmd}{\rmd t} \big( \Psi_t(x, \mu_t) \big)$ and the chain rule to compute
\begin{align*}
\dot{\Psi}_t(x) &= \partial_t \Psi_t(x,\mu_t) + \int_{\mathbb{R}^d} \Phi(x,y,\mu_t) \, \partial_t \mu_t(\rmd y) \\
&= \partial_t \Psi_t(x,\mu_t) + \int_{\mathbb{R}^d} \Phi(x,y,\mu_t) \, \left( h_t(y) - \int_{\mathbb{R}^d} h_t(z) \mu_t(z) \, \rmd z \right) \mu_t(\rmd y) \\
&= \partial_t \Psi_t(x,\mu_t) + \int_{\mathbb{R}^d} \tPhi(x,y,\mu_t) \, h_t(y) \, \mu_t(\rmd y) \\
        &= \partial_t \Psi_t(x,\mu_t) + \int_{\mathbb{R}^d} \tPhi(x,y,\mu_t) \, \big[ \dot{\Psi}_t(y) - \nabla \cdot b_t(y) - \nabla \log p_t(y) \cdot b_t(y) \big] \, \mu_t(\rmd y),
\label{eq:time_derivative_eq1}
\end{align*}
where we have used Proposition \ref{prop:target evolution} in the second line.
From a computational standpoint, the problematic term here is $\int_{\mathbb{R}^d} \tPhi(x,y,\mu_t) \nabla \cdot b_t(y) \mu_t(\rmd y)$, as this divergence is required to be computed in the data space which could be large. Fortunately, we can use integration by parts to obtain a more tractable expression. For any test function $\phi(y)$, we have
    \begin{align*}
        \int \phi(x) \nabla & \cdot b_t(x) \, \mu_t(\rmd x) = 
        - \int b_t(x) \cdot \nabla \big( \phi(x) \mu_t(x) \big) \, \rmd x \\
        &= - \int b_t(x) \cdot \big[ \mu_t(x) \nabla \phi(x) + \phi(x) \nabla \mu_t(x) \big ]\,  \rmd x \\
        & = - \int b_t(x) \cdot \mu_t(x) \nabla \phi(x) \, \rmd x - \int \phi(x) b_t(x) \cdot \nabla \log \mu_t(x) \, \mu_t(\rmd x) \\
        & = - \int b_t(x) \cdot \mu_t(x) \nabla \phi(x) \rmd x - \int \phi(x) \, b_t(x) \cdot \nabla \Psi_t(x, \mu_t) \, \mu_t(\rmd x) 
        \\
        & \qquad - \int \phi(x) \, b_t(x) \cdot \nabla \log p_t(x) \, \mu_t(\rmd x),
    \end{align*}
where we used the log-derivative trick $\nabla \mu_t = (\nabla \log \mu_t) \mu_t$, and in the final line used the fact that $\nabla \log \mu_t = \nabla \Psi_t(\cdot, \mu_t) + \nabla \log p_t$ (which follows from \eqref{eq:targets}). Taking $\phi(y) = \tPhi_t(x,y, \mu_t)$ and noting that the terms involving the score $\nabla \log p_t$ cancel, we obtain \eqref{eq:implicit_eq_for_dotPsi}.
\end{proof}

\subsection{Proof of Theorem \ref{thm:weighted_particle_convergence_main}}

Consider the dynamics
\begin{subequations}
\label{eq:finite_particle_dynamics}
\begin{align}
\rmd X_t^{i,N}
&=
\left[
b_t(X_t^{i,N})
+\eps_t s_t(X_t^{i,N})
+\eps_t\nabla_x\Psi_t(X_t^{i,N};\hat\mu_t^N)
\right]\rmd t
+\sqrt{2\eps_t}\,\rmd B_t^i,
\qquad
X_0^{i,N}\sim p_0,
\\
\label{eq:particle weights}
\rmd A_t^{i,N}
&=
\left[
b_t(X_t^{i,N})\cdot
\nabla_x\Psi_t(X_t^{i,N};\hat\mu_t^N)
+
\dot\Psi_t^{(N)}(X_t^{i,N})
\right]\rmd t,
\qquad
A_0^{i,N}=0.
\end{align}
\end{subequations}
where the weighted empirical measure is
\begin{equation}
\label{eq:weighted empirical}
\hat\mu_t^N
:=
\sum_{i=1}^Nw_t^{i,N}\delta_{X_t^{i,N}},
\end{equation}
with normalised weights
\begin{equation}
\label{eq:weights}
w_t^{i,N}
:=
\frac{\exp(A_t^{i,N})}
{\sum_{j=1}^N\exp(A_t^{j,N})}.
\end{equation}
In \eqref{eq:particle weights}, $\dot{\Psi}_t^{(N)}$ denotes a computable approximation of $\dot{\Psi}_t=\frac{\rmd}{\rmd t}\Psi_t(\cdot;\mu_t)$, in the following sense: For an arbitrary probability measure $\mu$, define $\dot{\Psi}_t^{\mathrm{exact}}(\cdot,\mu)$ as the solution to \eqref{eq:implicit_eq_for_dotPsi} where all occurrences of $\mu_t^*$ are replaced by $\mu$. In other words, $\dot{\Psi}_t^{\mathrm{exact}}(\cdot,\mu)$ is defined to be the solution to  
\begin{align}
\label{eq:PsiN}
\dot \Psi^{\mathrm{exact}}_t(x,\mu) &  =   \partial_t \Psi_t(x, \mu)
\\
& + \int_{\mathbb{R}^d} \Big(\tPhi_t(x,y,\mu) \dot \Psi^{\mathrm{exact}}_t(y,\mu) \notag 
+ b_t(y) \cdot \nabla_y \tPhi_t(x,y,\mu) \notag \\
& \qquad\qquad + \tPhi_t(x,y,\mu) \, b_t(y) \cdot \nabla_y \Psi_t(y,\mu) \Big) \; \rmd \mu(y).
\end{align}
In particular, for the weighted empirical measure $
\hat{\mu}_t^N
=
\sum_{j=1}^N
w_t^{j,N}\delta_{X_t^{j,N}}$,
this equation reads
\begin{align}
\label{eq:empirical_dotPsi}
\dot{\Psi}_t^{\mathrm{exact}}(x,\hat{\mu}_t^N)
={}&
\partial_t\Psi_t(x,\hat{\mu}_t^N)
\notag
+
\sum_{j=1}^N
w_t^{j,N}
\Big[
\tPhi_t(x,X_t^{j,N},\hat{\mu}_t^N)
\dot{\Psi}_t^{\mathrm{exact}}
(X_t^{j,N},\hat{\mu}_t^N)
\notag\\
&\qquad\qquad
+
b_t(X_t^{j,N})\cdot
\nabla_y\tPhi_t
(x,X_t^{j,N},\hat{\mu}_t^N)
\notag\\
&\qquad\qquad
+
\tPhi_t(x,X_t^{j,N},\hat{\mu}_t^N)
\,b_t(X_t^{j,N})\cdot
\nabla_y\Psi_t
(X_t^{j,N},\hat{\mu}_t^N)
\Big];
\end{align}
evaluating this at $x = X_t^{i,N}$ leads to an $N \times N$ linear system that forms the basis of Algorithm \ref{alg:Psi_solver_implicit}.

The implemented quantity $\dot{\Psi}_t^{(N)}$ used in \eqref{eq:particle weights} may instead only be approximately equal to $\dot{\Psi}_t^{\mathrm{exact}}(\cdot,\hat{\mu}_t^N)$, for example coming from a finite number of Picard iterations as in Algorithm \ref{alg:Psi_solver_fixedpoint}. The solver-consistency
condition in the third item of Assumption~\ref{ass:particle_convergence} quantifies the
resulting numerical error.

\begin{assumption}[Regularity, well-posedness and solver consistency]
\label{ass:particle_convergence}
Assume the following.
\begin{enumerate}
\item The McKean-Vlasov dynamics \eqref{eq:McKean} are well posed, in the sense that the corresponding Feynman-Kac equation \eqref{eq:normalised-fk-short} has a unique solution in $C([0,T],\mathcal P(\mathbb R^d))$.
\item The drift coefficients of \eqref{eq:McKean} (denoted by \(B_t\) and \(g_t\) in \eqref{eq:drift coefficients}) are bounded and Lipschitz in
\((x,\mu)\), uniformly for \(t\in[0,T]\). More precisely, for some constant
\(L>0\),
\begin{align}
|B_t(x,\mu)-B_t(y,\nu)|
&\leq
L\bigl(|x-y|+d_{\mathrm{BL}}(\mu,\nu)\bigr),
\\
|g_t(x,\mu)-g_t(y,\nu)|
&\leq
L\bigl(|x-y|+d_{\mathrm{BL}}(\mu,\nu)\bigr),
\end{align}
and
\begin{equation}
\label{eq:boundedness}
|B_t(x,\mu)|+|g_t(x,\mu)|\leq L.
\end{equation} 
We also assume that $\sup_{t\in[0,T]}\eps_t<\infty$.
\item The approximation $\dot{\Psi}_t^{(N)}$ is consistent in the sense that
\begin{equation}
\label{eq:solver_error}
\delta_N
:=
\mathbb{E}
\left[
\int_0^T
\sum_{i=1}^N w_t^{i}
\left|
 \dot{\Psi}_t^{(N)}(X_t^{i})
-
\dot{\Psi}^{\mathrm{exact}}_t(X_t^{i},\hat{\mu}_t^N)
\right|
\,\rmd t
\right]
\longrightarrow 0
\qquad
\text{as }N\to\infty,
\end{equation}
where $\dot{\Psi}^{\mathrm{exact}}_t(\cdot,\hat{\mu}_t^N)$ denotes the exact solution for $\dot{\Psi}$ to equation \eqref{eq:implicit_eq_for_dotPsi} in which every occurrence of $\mu_t^*$ is  replaced by $\hat{\mu}_t^N$, and $t \in [0,T]$ is fixed.
\end{enumerate}
\end{assumption}
Under this assumption, the weighted empirical measure in \eqref{eq:weighted empirical} converges to the target $\mu_t^*$:
\begin{restatable}[\textbf{Convergence of the Weighted Particle System}]
{theorem}{convergenceresult}
\label{thm:weighted_particle_convergence}
Suppose Assumption~\ref{ass:particle_convergence} holds. 
Then
\begin{equation}
\label{eq:test_function_qualitative_convergence}
    \mathbb E\left[
        \sup_{t\in[0,T]}
        \left|
            \hat\mu_t^N(\phi)-\mu_t^*(\phi)
        \right|
    \right]
    \longrightarrow0  \qquad
\text{as }N\to\infty,
\end{equation}
 for every $\phi\in C_b^2(\mathbb R^d)$,
\end{restatable}
With this in place, a similar result holds if the particle system \eqref{eq:finite_particle_dynamics} is rejuvenated by resampling:
\begin{corollary}[Convergence with residual resampling]
\label{cor:particle_convergence_resampling}
Suppose the assumptions of
Theorem~\ref{thm:weighted_particle_convergence} hold, and suppose that
residual resampling is performed at a fixed finite collection of times
\[
0<\tau_1<\cdots<\tau_K\leq T,
\]
where the maximum number of resampling steps $K$ is uniformly bounded in $N$. After each resampling
event, the particles are assigned equal weights \(\tfrac{1}{N}\), i.e. $A^j_{\tau_i} = 0$ for $j=1,\ldots, N$ and $i = 1,\ldots, K$. Then, for
every \(\phi\in C_b^2(\mathbb R^d)\),
\[
\mathbb E\left[
    \sup_{t\in[0,T]}
    \left|
        \hat\mu_t^N(\phi)-\mu_t^*(\phi)
    \right|
\right]
\longrightarrow0
\qquad
\text{as }N\to\infty.
\]
\end{corollary}

Before proceeding to the proof, we need the following auxiliary result:
\begin{lemma}
\label{lem:S}
Defining the squared sum of weights
\(
    S_t^N
    :=\sum_{i=1}^N(w_t^{i,N})^2,
\)
we have that
\begin{equation}
\label{eq:weight bound}
    \mathbb E\left[
        \sup_{t\in[0,T]}S_t^N
    \right]
    \leq
    e^{4LT}\left(\frac1N+4\delta_N\right)
    \longrightarrow0, \qquad \text{as} \,\, N \rightarrow \infty.
\end{equation}
\end{lemma}

\begin{proof}
For each particle, define the shorthand
\[
g_t^{i,N}
:=
g_t\left(X_t^{i,N},\widehat{\mu}_t^N\right)
\]
and the solver error
\[
e_t^{i,N}
:=
\dot{\Psi}_t^{(N)}\left(X_t^{i,N}\right)
-
\dot{\Psi}_t^{\mathrm{exact}}
\left(
    X_t^{i,N},
    \widehat{\mu}_t^N
\right).
\]
We also define the weighted averages
\[
\overline{g}_t^N
:=
\sum_{j=1}^N
w_t^{j,N}g_t^{j,N},
\qquad
\overline{e}_t^N
:=
\sum_{j=1}^N
w_t^{j,N}e_t^{j,N}.
\]
Differentiating \eqref{eq:weights}, and using \eqref{eq:particle weights},
gives
\begin{equation*}
\frac{\mathrm{d}}{\mathrm{d}t}w_t^{i,N}
=
w_t^{i,N}
\left(
    g_t^{i,N}-\overline{g}_t^N
    +
    e_t^{i,N}-\overline{e}_t^N
\right).
\end{equation*}
Consequently,
\begin{equation}
\label{eq:S dynamics}
\frac{\mathrm{d}}{\mathrm{d}t}S_t^N
={}
2\sum_{i=1}^N
\left(w_t^{i,N}\right)^2
\left(
    g_t^{i,N}-\overline{g}_t^N
\right)
+
2\sum_{i=1}^N
\left(w_t^{i,N}\right)^2
\left(
    e_t^{i,N}-\overline{e}_t^N
\right).
\end{equation}
We proceed by estimating the right-hand side of \eqref{eq:S dynamics}:
By the boundedness assumption \eqref{eq:boundedness}, we have $
\left|g_t^{i,N}\right|\leq L$ and $
\left|\overline{g}_t^N\right|
\leq
\sum_{j=1}^N
w_t^{j,N}\left|g_t^{j,N}\right|
\leq L$, and hence
\[
2\sum_{i=1}^N
\left(w_t^{i,N}\right)^2
\left(
    g_t^{i,N}-\overline{g}_t^N
\right)
\leq
4L S_t^N.
\]

By the definition \eqref{eq:weights}, the weights $w_t^{i,N}$ are nonnegative and sum to one, and so $
\left(w_t^{i,N}\right)^2
\leq
w_t^{i,N}$ and $
S_t^N\leq 1$.
Therefore,
\begin{align*}
2\sum_{i=1}^N
\left(w_t^{i,N}\right)^2
\left(
    e_t^{i,N}-\overline{e}_t^N
\right)
\leq
2\sum_{i=1}^N
\left(w_t^{i,N}\right)^2
\left|e_t^{i,N}\right|
+
2S_t^N
\left|\overline{e}_t^N\right|
\leq
4\sum_{i=1}^N
w_t^{i,N}\left|e_t^{i,N}\right|.
\end{align*}
Combining the preceding estimates with \eqref{eq:S dynamics} gives
\[
\frac{\mathrm{d}}{\mathrm{d}t}S_t^N
\leq
4L S_t^N
+
4\sum_{i=1}^N
w_t^{i,N}\left|e_t^{i,N}\right|.
\]

Since the weights are initialised uniformly, \(w_0^{i,N}=1/N\), we have $
S_0^N
=
\sum_{i=1}^N\frac{1}{N^2}
=
\frac{1}{N}$.
Gr\"onwall's inequality therefore yields
\[
\sup_{t\in[0,T]}S_t^N
\leq
e^{4LT}
\left[
    \frac{1}{N}
    +
    4\int_0^T
    \sum_{i=1}^N
    w_t^{i,N}\left|e_t^{i,N}\right|
    \,\mathrm{d}t
\right].
\]
Taking expectations then implies the claim.
\end{proof}

\begin{proof}[Proof of Theorem \ref{thm:weighted_particle_convergence}]

We fix $\phi\in C_b^2(\mathbb R^d)$ and derive the evolution equation for  the empirical measure $\hat\mu_t^N(\phi) = \sum_{i=1}^N w_t^{i,N} \phi(X_t^{i,N})$. It\^o's formula gives 
\begin{equation}
\label{eq:empirical_weak}
    \hat\mu_t^N(\phi)
    =
    \hat\mu_0^N(\phi)
    +\int_0^t\mathcal F_s^\phi(\hat\mu_s^N)\,\rmd s
    +E_t^{N,\phi}
    +M_t^{N,\phi},
\end{equation}
where 
\begin{align}
\label{eq:F_fk}
\mathcal F_t^\phi(\mu)
:={}
\mu\left(
B_t(\cdot,\mu)\cdot\nabla\phi
    +\eps_t\Delta\phi
\right)
+
\mu(g_t(\cdot,\mu)\phi)
-
\mu(g_t(\cdot,\mu))\mu(\phi)
\end{align}
refers to the integrated right-hand side of \eqref{eq:normalised-fk-short}, the solver error is encoded in 
\begin{equation}
\label{eq:solver_residual}
    E_t^{N,\phi}
    :=
    \int_0^tR_s^{N,\phi}\,\rmd s,
    \qquad
    R_t^{N,\phi}
    :=
    \sum_{i=1}^Nw_t^{i,N}e_t^{i,N}
    \left(
        \phi(X_t^{i,N})-\hat\mu_t^N(\phi)
    \right),
\end{equation}
and
\begin{equation}
\label{eq:martingale}
    M_t^{N,\phi}
    :=
    \int_0^t
    \sqrt{2\eps_s}
    \sum_{i=1}^Nw_s^{i,N}
    \nabla\phi(X_s^{i,N})\cdot\rmd B_s^i
\end{equation}
is the martingale part.
Notice that $
    |R_t^{N,\phi}|
    \leq
    2\|\phi\|_\infty\sum_{i=1}^N
    w_t^{i,N}\left|e_t^{i,N}\right|$; the assumption in \eqref{eq:solver_error} therefore implies
\begin{equation}
\label{eq:solver_residual_bound}
    \mathbb E\left[
        \sup_{t\in[0,T]}|E_t^{N,\phi}|
    \right]
    \leq
    2\|\phi\|_\infty\delta_N
    \longrightarrow0.
\end{equation}
The quadratic variation of the martingale part can be estimated as
\[
    \langle M^{N,\phi}\rangle_T
    \leq
    2\|\eps\|_\infty T
    \|\nabla\phi\|_\infty^2
    \sup_{t\in[0,T]}S_t^N.
\]
It{\^o}'s isometry, Doob's $L^2$ maximal inequality, and Lemma \ref{lem:S} therefore imply
\begin{equation}
\label{eq:martingale_bound}
    \mathbb E\left[
        \sup_{t\in[0,T]}|M_t^{N,\phi}|^2
    \right]
    \leq
    C_T\|\nabla\phi\|_\infty^2
    \left(\frac1N+\delta_N\right)
    \longrightarrow0,
\end{equation}
where the constant $C_T$ is independent of $N$. Formally, we thus conclude that \eqref{eq:empirical_weak} converges term-by-term to a weak form of the Feynman-Kac equation \eqref{eq:normalised-fk-short} which governs the target law $\mu_t^*$. To make this argument precise, we argue as follows:

\textbf{Tightness.} We show that the family $\{\operatorname{Law}(\hat\mu^N):N\geq1\}$ is tight in $C([0,T],\mathcal P(\mathbb R^d))$. To this end, fix $\phi\in C_c^\infty(\mathbb R^d)$.  By the boundedness  assumption \eqref{eq:boundedness} and $\sup_{t \in [0,T]}\varepsilon_t < \infty$, we have that $
    \sup_{t,\mu}|\mathcal F_t^\phi(\mu)|<\infty$.
Therefore, the drift paths
$t\mapsto\int_0^t\mathcal F_s^\phi(\hat\mu_s^N)\,\rmd s$ have a
common deterministic Lipschitz constant, while
\eqref{eq:solver_residual_bound} and
\eqref{eq:martingale_bound} show that the remaining error
$E_t^{N,\phi}
    +M_t^{N,\phi}$  in \eqref{eq:empirical_weak} vanishes uniformly in probability.  It now follows from
Arzela--Ascoli that the family of real-valued processes
$(\hat\mu^N(\phi))_{N\geq1}$ is tight in $C([0,T],\mathbb R)$. To obtain tightness of $\{\operatorname{Law}(\hat\mu^N):N\geq1\}$, we verify the conditions of Jakubowski's tightness criterion
\citep[Theorem~3.1(i)]{jakubowski1986}: indeed,
we may choose a countable convergence-determining family
$\mathcal G\subset C_c^\infty(\mathbb R^d)$; compact containment follows from Lemma~\ref{lem:compact_containment} below.
Finally, since every
$\hat\mu^N$ has continuous paths, and the Skorokhod topology restricted
to continuous paths is the uniform topology
\citep[Proposition~1.6(i)]{jakubowski1986}, the laws \{$\operatorname{Law}(\hat\mu^N):N\geq1\}$ are indeed tight in 
$C([0,T],\mathcal P(\mathbb R^d))$.

\textbf{Subsequential limits.} We conclude by showing that every subsequence of $\hat{\mu}^N$ has a further subsequence whose law converges weakly to the Dirac measure at $\mu^*$ in $\mathcal{P}(C([0,T],\mathcal P(\mathbb R^d)))$: Let $(N_j) \subset \mathbb{N}$ be an
arbitrary subsequence.  By tightness, it has a further subsequence,
not relabelled, such that
\[
    \hat\mu^{N_j}
    \Longrightarrow
    \overline\mu
    \qquad\text{in }
    C([0,T],\mathcal P(\mathbb R^d)).
\]
For every $\phi\in C_c^\infty(\mathbb R^d)$, Assumption~\ref{ass:particle_convergence} implies the Lipschitz property
\begin{equation}
\label{eq:F_continuity}
    |\mathcal F_t^\phi(\mu)-\mathcal F_t^\phi(\nu)|
    \leq
    C_\phi d_{\mathrm{BL}}(\mu,\nu),
\end{equation}
uniformly in $t$. Together with the facts that  
$E_t^{N,\phi}+M_t^{N,\phi}$ converges to zero in probability and  that the initial empirical measures converge to $p_0$ by the law of large numbers,
the continuous mapping theorem and Slutsky's lemma  imply
that, almost surely,
\[
    \overline\mu_t(\phi)
    =
    p_0(\phi)
    +\int_0^t\mathcal F_s^\phi(\overline\mu_s)\,\rmd s
\]
for every rational $t$ and every $\phi$ in a countable dense subset of
$C_c^\infty(\mathbb R^d)$.  Continuity in time and density extend this
identity to every $t\in[0,T]$ and every
$\phi\in C_c^\infty(\mathbb R^d)$. As a consequence, $\overline\mu$ is a weak
solution of the Feynman--Kac equation \eqref{eq:normalised-fk-short}.
Uniqueness in Assumption~\ref{ass:particle_convergence} gives
$\overline\mu=\mu^*$ almost surely.

Every subsequence therefore has a further subsequence converging in
law to the same deterministic path $\mu^*$.  Hence the full sequence
converges in probability to \(\mu^*\) in $
C\bigl([0,T],\mathcal P(\mathbb R^d)\bigr)$.
Since every \(\phi\in C_b^2(\mathbb R^d)\) is bounded and Lipschitz,
it follows that $
\sup_{t\in[0,T]}
\left|
\hat\mu_t^N(\phi)-\mu_t^*(\phi)
\right|
\longrightarrow 0$ in probability.
Moreover,
\[
\sup_{t\in[0,T]}
\left|
\hat\mu_t^N(\phi)-\mu_t^*(\phi)
\right|
\leq
2\|\phi\|_\infty
\]
shows that the sequence is  uniformly integrable, and thus the claim follows.
\end{proof}
The following technical lemma is used in the tightness step of the preceding proof. 
\begin{lemma}[Compact containment]
\label{lem:compact_containment}
For every \(\eta>0\), there exists a weakly compact set $
\mathsf K_\eta
\subset
\mathcal P(\mathbb R^d)$ such that
\[
\liminf_{N\to\infty}
\mathbb P\left(
    \hat\mu_t^N\in\mathsf K_\eta
    \text{ for every }t\in[0,T]
\right)
\geq
1-\eta.
\]
\end{lemma}

\begin{proof}
For any radius \(r>0\), we denote the open ball in $\mathbb{R}^d$ by $
B_r:=\left\{x\in\mathbb R^d:|x|<r\right\}$.
Choose a radial, nondecreasing function
\(\chi\in C^\infty(\mathbb R^d;[0,1])\) such that $
\chi=0$ on $B_1$ and 
$\chi=1$ on $B_2^c$, 
and set
\[
\chi_R(x):=\chi(x/R), \qquad R > 0.
\]
Notice the following properties,
\begin{equation}
\label{eq:chi bound}
\mathbf 1_{B_{2R}^c}
\leq
\chi_R
\leq
\mathbf 1_{B_R^c},
\qquad
\|\nabla\chi_R\|_\infty\leq\frac{C}{R},
\qquad
\|\Delta\chi_R\|_\infty\leq\frac{C}{R^2},
\end{equation}
for some constant $C<\infty$, as well as 
\begin{equation}
\left|
\operatorname{Cov}_{\mu}
\left(
g_t(\cdot,\mu),\chi_R
\right)
\right|
\leq
2L\mu(\chi_R),
\end{equation}
which follows directly from $0 \le \xi_R \le 1$ and \eqref{eq:boundedness}.

We now apply
\eqref{eq:empirical_weak} with \(\phi=\chi_R\), take the supremum over time and expectations. Using the estimates \eqref{eq:solver_residual_bound},  
\eqref{eq:martingale_bound}, \eqref{eq:chi bound}, as well as Cauchy-Schwarz, we arrive at
\begin{align*}
\mathbb E\left[
    \sup_{r\in[0,t]}\hat\mu_r^N(\chi_R)
\right]
\leq{}&
p_0(\chi_R)
+
C_T\left(
    \frac1R+\frac1{R^2}
\right)
+
2\delta_N
\\
&+
\frac{C_T}{R}
\left(
    \frac1N+\delta_N
\right)^{1/2}
+
2L\int_0^t
\mathbb E\left[
    \sup_{u\in[0,s]}\hat\mu_u^N(\chi_R)
\right]
\,\mathrm ds,
\end{align*}
with a constant $C_T < \infty$ that may depend on $T$, but not on $N$ or $R$.
Gr\"onwall's inequality then yields
\begin{equation}
\label{eq:cutoff_tail_bound}
\mathbb E\left[
    \sup_{t\in[0,T]}\hat\mu_t^N(\chi_R)
\right]
\leq
C_T\left[
    p_0(\chi_R)
    +
    \frac1R
    +
    \delta_N
    +
    \frac1R
    \left(
        \frac1N+\delta_N
    \right)^{1/2}
\right],
\end{equation}
with a different constant $C_T < \infty$, denoted by the same symbol.
Since \(p_0(\chi_R)\to0\) as \(R\to\infty\) and
\(\delta_N\to0\) as $N \rightarrow \infty$, it follows that
\[
\lim_{R\to\infty}
\limsup_{N\to\infty}
\mathbb E\left[
    \sup_{t\in[0,T]}\hat\mu_t^N(\chi_R)
\right]
=
0.
\]
This uniform control of the probability mass at spatial infinity yields the claim by a relatively standard countable-tail
construction; see, for example,
\citet[Lemma~6.1]{giesecke2013default}.
\end{proof}

\begin{proof}[Proof of Corollary]
Between resampling times, the proof of
Theorem~\ref{thm:weighted_particle_convergence} is unchanged.
Moreover, resampling resets all weights to \(\tfrac{1}{N}\), and hence
\[
\sum_{i=1}^N
\left(w_{\tau_k}^{i,N}\right)^2
=
\frac1N
\leq
\sum_{i=1}^N
\left(w_{\tau_k-}^{i,N}\right)^2,
\]
where $w_{\tau_k-}^{i,N}$ are the weights just before, and $w_{\tau_k}^{i,N}$ are the weights just after resampling.
As a consequence, the conclusion of Lemma~\ref{lem:S} remains
valid.

For a bounded test function \(\phi\), define the jump induced by the
\(k\)-th resampling step by
\[
\Delta_k^{N,\phi}
:=
\hat\mu_{\tau_k}^N(\phi)
-
\hat\mu_{\tau_k-}^N(\phi).
\]
In \eqref{eq:empirical_weak}, the effect of resampling is the appearance of an additional jump term 
\begin{equation}
\label{eq:jump term}
J_t^{N,\phi}
:=
\sum_{k:\,\tau_k\leq t}
\Delta_k^{N,\phi}.
\end{equation}

Residual resampling is conditionally unbiased and its conditional
variance is bounded by that of multinomial resampling; see
\citet[Sections~2.2 and~3.2]{douc2005comparison}.
Consequently,
\[
\mathbb E\left[
    \Delta_k^{N,\phi}
    \,\middle|\,
    \mathcal F_{\tau_k^N-}
\right]
=0,
\qquad
\mathbb E\left[
    \left|\Delta_k^{N,\phi}\right|^2
    \,\middle|\,
    \mathcal F_{\tau_k^N-}
\right]
\leq
\frac{\|\phi\|_\infty^2}{N}.
\]
Therefore, \eqref{eq:jump term}
is a martingale, and we have the estimate
\[
\mathbb E\left[
    \sup_{t\in[0,T]}
    \left|J_t^{N,\phi}\right|^2
\right]
\leq
\frac{C K\|\phi\|_\infty^2}{N}
\longrightarrow0,
\]
where the limit statement rests on the fact that $K$ is uniformly bounded. The tightness and subsequential-identification arguments
proceed exactly as in the proof of
Theorem~\ref{thm:weighted_particle_convergence}, yielding the result.
\end{proof}

\section{Licenses}
\label{app:licenses}

The following assets were used in this work.
\begin{itemize}
    \item Boltz-2 \citep{passaro2025boltz2} \\
    MIT License \\ \url{https://github.com/jwohlwend/boltz}
    \item Protenix \citep{protenixv0} \\
    Apache-2.0 License \\ \url{https://github.com/bytedance/Protenix}
    \item Guided Protein Structure Prediction \citep{maddipatla2025inverse,Maddipatla2026_nature} \\
    \url{https://github.com/sai-advaith/guided_alphafold}
    \item Adenylate Kinase Potential of Mean Force \citep{BECKSTEIN2009160} \\
    Creative Commons Attribution-ShareAlike 4.0 International License \\ \url{https://becksteinlab.physics.asu.edu/research/52/adk-apo-pmf}
    \item 4OLE PDB entry \citep{4OLE_structure} \\
    CC0 1.0 Universal (CC0 1.0) Public Domain Dedication \\
    \url{https://www.rcsb.org/structure/4OLE}
\end{itemize}

\newpage
\section*{NeurIPS Paper Checklist}

\begin{enumerate}

\item {\bf Claims}
    \item[] Question: Do the main claims made in the abstract and introduction accurately reflect the paper's contributions and scope?
    \item[] Answer: \answerYes{} 
    \item[] Justification: In our paper, we provide a mean-field framework for distributional control of diffusion models, and empirically evaluate its behaviour. Our proposed framework and methodology is outlined in Section \ref{sec:mean_field_framework}, with experimental validation in Section \ref{sec:experiments}.

\item {\bf Limitations}
    \item[] Question: Does the paper discuss the limitations of the work performed by the authors?
    \item[] Answer: \answerYes{} 
    \item[] Justification: We include a dedicated limitations section in Section \ref{sec:discussion}.

\item {\bf Theory assumptions and proofs}
    \item[] Question: For each theoretical result, does the paper provide the full set of assumptions and a complete (and correct) proof?
    \item[] Answer: \answerYes{} 
    \item[] Justification: Proofs of our theoretical results are included in \Cref{app:proofs}.

    \item {\bf Experimental result reproducibility}
    \item[] Question: Does the paper fully disclose all the information needed to reproduce the main experimental results of the paper to the extent that it affects the main claims and/or conclusions of the paper (regardless of whether the code and data are provided or not)?
    \item[] Answer: \answerYes{} 
    \item[] Justification: Detailed descriptions of method implementation details are given in \Cref{app:implementation_details}, and detailed descriptions of the experimental settings are provided in \Cref{app:experimental_details}.

\item {\bf Open access to data and code}
    \item[] Question: Does the paper provide open access to the data and code, with sufficient instructions to faithfully reproduce the main experimental results, as described in supplemental material?
    \item[] Answer: \answerYes{} 
    \item[] Justification: We include code for the low-dimensional experiments in Section \ref{sec:low_dim} in the supplementary material. For our protein experiments, we adapt this codebase to the Boltz-2 codebase.

\item {\bf Experimental setting/details}
    \item[] Question: Does the paper specify all the training and test details (e.g., data splits, hyperparameters, how they were chosen, type of optimizer) necessary to understand the results?
    \item[] Answer: \answerYes{} 
    \item[] Justification: We include all experimental details in \Cref{app:experimental_details}.

\item {\bf Experiment statistical significance}
    \item[] Question: Does the paper report error bars suitably and correctly defined or other appropriate information about the statistical significance of the experiments?
    \item[] Answer: \answerNo{} 
    \item[] Justification: We report error bars in some plots (such as \Cref{fig:N_dependence}), but it is not computationally feasible to include in all our plots. We also note in the Appendix that most of the metrics are very consistent across runs due to the large batch size.

\item {\bf Experiments compute resources}
    \item[] Question: For each experiment, does the paper provide sufficient information on the computer resources (type of compute workers, memory, time of execution) needed to reproduce the experiments?
    \item[] Answer: \answerYes{} 
    \item[] Justification: We include these details in \Cref{app:experimental_details}.
    
\item {\bf Code of ethics}
    \item[] Question: Does the research conducted in the paper conform, in every respect, with the NeurIPS Code of Ethics \url{https://neurips.cc/public/EthicsGuidelines}?
    \item[] Answer: \answerYes{} 
    \item[] Justification: We have reviewed the Code of Ethics and confirm that our work conforms to these guidelines. We see no potential harmful consequences of our work, and include extensive reproducibility details in Appendix \ref{app:experimental_details}.

\item {\bf Broader impacts}
    \item[] Question: Does the paper discuss both potential positive societal impacts and negative societal impacts of the work performed?
    \item[] Answer: \answerNA{} 
    \item[] Justification: Our work is mostly theory and methodology-focused. We do not see any immediate societal impacts, though certain applications may share similar societal consequences as in flow-based generative models upon which our approach is based.
    
\item {\bf Safeguards}
    \item[] Question: Does the paper describe safeguards that have been put in place for responsible release of data or models that have a high risk for misuse (e.g., pre-trained language models, image generators, or scraped datasets)?
    \item[] Answer: \answerNA{} 
    \item[] Justification: Our work poses no such risks.

\item {\bf Licenses for existing assets}
    \item[] Question: Are the creators or original owners of assets (e.g., code, data, models), used in the paper, properly credited and are the license and terms of use explicitly mentioned and properly respected?
    \item[] Answer: \answerYes{} 
    \item[] Justification: We cite the works that we use, and include the licenses of code and datasets in Appendix \ref{app:licenses}.

\item {\bf New assets}
    \item[] Question: Are new assets introduced in the paper well documented and is the documentation provided alongside the assets?
    \item[] Answer: \answerNA{} 
    \item[] Justification: The paper does not release new assets.

\item {\bf Crowdsourcing and research with human subjects}
    \item[] Question: For crowdsourcing experiments and research with human subjects, does the paper include the full text of instructions given to participants and screenshots, if applicable, as well as details about compensation (if any)? 
    \item[] Answer: \answerNA{} 
    \item[] Justification: The paper does not involve crowdsourcing nor research with human subjects.

\item {\bf Institutional review board (IRB) approvals or equivalent for research with human subjects}
    \item[] Question: Does the paper describe potential risks incurred by study participants, whether such risks were disclosed to the subjects, and whether Institutional Review Board (IRB) approvals (or an equivalent approval/review based on the requirements of your country or institution) were obtained?
    \item[] Answer: \answerNA{} 
    \item[] Justification: The paper does not involve crowdsourcing nor research with human subjects.

\item {\bf Declaration of LLM usage}
    \item[] Question: Does the paper describe the usage of LLMs if it is an important, original, or non-standard component of the core methods in this research? Note that if the LLM is used only for writing, editing, or formatting purposes and does \emph{not} impact the core methodology, scientific rigor, or originality of the research, declaration is not required.
    \item[] Answer: \answerNA{} 
    \item[] Justification: The core method development in this research does not involve LLMs as any important, original, or non-standard components.

\end{enumerate}

\end{document}